\documentclass[twoside]{article}

\usepackage{microtype}
\usepackage{graphicx}
\usepackage{subfigure}
\usepackage{booktabs}
\usepackage{amssymb}
\usepackage{amsmath}
\usepackage{amsfonts,amsthm}
\usepackage{stmaryrd}
\usepackage{ragged2e}
\usepackage{graphicx}
\usepackage{empheq}
\usepackage{color}
\usepackage{dsfont}
\usepackage{hyperref}
\usepackage{bbold}
\usepackage{hyperref}
\usepackage[ruled,vlined]{algorithm2e}
\usepackage[normalem]{ulem}
\usepackage{enumitem}
\usepackage{thm-restate}
\usepackage{cleveref}
\usepackage{comment}
\usepackage[dvipsnames]{xcolor}
\usepackage{comment}

\definecolor{darkgreen}{rgb}{0,0.6,0}
\newtheorem{theorem}{Theorem}

\newtheorem{definition}{Definition}

\newtheorem{proposition}[theorem]{Proposition}

\newtheorem{exam}{Example}
\newenvironment{example}{\begin{exam} \rm}{\end{exam}}
\usepackage{xspace}
\newcommand{\ken}{Kendall's-$\tau$\xspace}

\newcommand{\n}{[\![n]\!]}

\newcommand{\SuggestEdit}[3][red]{\textcolor{#1}{\sout{#2}}\textcolor{#1}{#3}}
\let\svthefootnote\thefootnote
\newcommand\colorfootnote[2][black]{\def\thefootnote{\color{#1}\svthefootnote}%
  \footnote{\color{#1}#2}\def\thefootnote{\color{black}\svthefootnote}}
\newcommand\RevComment[3][red]{\protect\colorfootnote[#1]{{\textbf{[#2: #3]}}}}
\newcommand\newrevisor[3]{
  \colorlet{#1}{#3}
  \expandafter\newcommand\csname#1\endcsname[2]{\SuggestEdit[#1]{##1}{##2}}
  \expandafter\newcommand\csname#2\endcsname[1]{\RevComment[#1]{#2}{##1}}
}

\newrevisor{pavlo}{PAVLO}{yellow}
\newrevisor{ste}{STE}{green}
\newrevisor{ekhine}{EKHINE}{red}
\newrevisor{morgane}{MORGANE}{purple}

\usepackage[accepted]{aistats2022}
\begin{document}

%

%

\twocolumn[

\aistatstitle{Statistical Depth Functions for Ranking Distributions: \\Definitions, Statistical Learning and Applications}

\aistatsauthor{ Morgane Goibert \And Stéphan Clémençon \And  Ekhine Irurozki \And Pavlo Mozharovskyi }

\aistatsaddress{ Télécom Paris \\ Criteo \And  Télécom Paris \And Télécom Paris \And  Télécom Paris } ]

\addtocontents{toc}{\protect\setcounter{tocdepth}{0}}

\begin{abstract}
The concept of \textit{median/consensus} has been widely investigated in order to provide a statistical summary of ranking data, \textit{i.e.} realizations of a random permutation $\Sigma$ of a finite set, $\{1,\; \ldots,\; n\}$ with $n\geq 1$ say. As it sheds light onto only one aspect of $\Sigma$'s distribution $P$, it may neglect other informative features. It is the purpose of this paper to define analogues of quantiles, ranks and statistical procedures based on such quantities for the analysis of ranking data by means of a metric-based notion of \textit{depth function} on the symmetric group. Overcoming the absence of vector space structure on $\mathfrak{S}_n$, the latter defines a center-outward ordering of the permutations in the support of $P$ and extends the classic metric-based formulation of \textit{consensus ranking} (\textit{medians} corresponding then to the \textit{deepest} permutations).  The axiomatic properties that \textit{ranking depths} should ideally possess are listed, while computational and generalization issues are studied at length. Beyond the theoretical analysis carried out, the relevance of the novel concepts and methods introduced for a wide variety of statistical tasks are also supported by numerous numerical experiments.

\end{abstract}
\section{Introduction}

The statistical analysis of ranking data as recently received much attention (\textit{e.g.} \cite{AY14} and references therein), fed by the increasing number of modern applications involving \textit{preferences} data (search engines, recommender systems, etc.).
Such data usually consist of $N\geq 1$ permutations $\sigma_1,\; \ldots,\; \sigma_N$ on an ensemble of $n\geq 1$ items, indexed by $i\in\{1,\; \ldots,\; n\}$. The major scientific challenge arises from \textit{the absence of any vector space structure} on the set of all permutations, the symmetric group $\mathfrak{S}_n$. Given the impossibility of 'averaging' the $\sigma_j$'s in a straightforward manner, the issue of summarizing a ranking dataset by a single permutation, referred to as \textit{Consensus Ranking} or \textit{Ranking Aggregation}, has concentrated much interest (seminal works of \cite{de1781memoire, Condorcet} in social choice theory, \cite{Patel13} in bioinformatics, \cite{DSM16} in meta-search engines, \cite{DL05} in competition ranking, etc.).
Two approaches to Consensus Ranking have been studied. The first one, initiated by Condorcet in the 18th century, is based on probabilistic modelling.
The second one is a metric-based: equipped with a (pseudo-) distance on $\mathfrak{S}_n$, a barycentric permutation, referred to as a \textit{ranking median}, is found.
However, central measures such as medians shed light on only one aspect of a multivariate distribution and ignore other interesting characteristics. Thus, the informative nature of ranking medians about the distribution $P$ of a random permutation $\Sigma$, \textit{i.e.} a r.v. takings its values in $\mathfrak{S}_n$, is limited and must be complemented by additional quantities, providing information analogous to that illuminated by quantiles for a univariate distribution. \\

This article is devoted to defining such quantities for ranking data. We extend the statistical depth concept, originally introduced so as to define quantiles for probability distributions on $\mathbb{R}^d$ with $d\geq 2$ (see \textit{e.g.} \cite{Mosler13}), to ranking distributions. Some basics in statistical depth theory are briefly recalled in section \ref{sec:background}, while section \ref{sec:depth_for_rankings} introduced an extension of the notion of depth function tailored to ranking data. Desirable properties for ranking depths are listed therein, and shown to hold under mild conditions, \textit{e.g.} stochastic transitivity.
Based on a pseudo-metric on $\mathfrak{S}_n$, the depth of a ranking $\sigma$ relative to $P$ measures its expected closeness to the random permutation $\Sigma$. Hence, ranking medians correspond to the deepest rankings. In section 4, statistical guarantees are provided for the ranking depth and its by-products, in the form of non-asymptotic bounds for the deviations between the ranking depth function and its statistical counterpart in particular. A trimming algorithm, based on the ranking depth concept, to recover automatically a stochastically transitive version of the empirical ranking distribution is also proposed therein.
Beyond the theoretical/algorithmic concepts introduced and analyzed here, the relevance of the notion of ranking depth is motivated by a wide variety of statistical applications, illustrated by several numerical experiments in section \ref{sec:exp}.

The main contributions of the paper are summarized below:
\begin{itemize}[noitemsep, itemsep=2pt, topsep=0pt, leftmargin=4mm]
    \item Statistical depth and related axiomatic properties are extended to ranking data, in order to emulate quantiles/ranks for r.v.'s valued in $\mathfrak{S}_n$.
    \item A finite-sample analysis ensures the usability of the notion of ranking depth introduced.
    \item An algorithm of great simplicity that uses ranking depth to build stochastically transitive empirical ranking distributions (based on which, crucial statistical tasks such as consensus ranking are straightforward) is proposed.
    \item The ranking depth and the related quantile regions in $\mathfrak{S}_n$ it defines can be used for the statistical analysis of rankings: 1) fast and robust recovery of medians in consensus ranking, 2) informative graphical representations of ranking data, 3) anomaly/novelty detection, 4) homogeneity testing.
\end{itemize}

\section{Background and Preliminaries}
\label{sec:background}

We start with recalling some basics in statistical depth theory, together with key notions of the statistical analysis of ranking data involved in the subsequent analysis.
Throughout the paper the indicator function of any event $\mathcal{E}$ is denoted by $\mathbb{I}\{ \mathcal{E} \}$, the Dirac mass at any point $a$ by $\delta_a$, the floor function by $u\in \mathbb{R}\mapsto \lfloor u \rfloor$, the convolution product of two real valued functions $f$ and $g$ defined on the real line, when well-defined, by $f*g$, the cardinality of any finite set $E$ by $\# E$ and the set of permutations of $\{1,\; \ldots,\; n  \}$ by $\mathfrak{S}_n$ for $n\geq 1$.

\subsection{Depth Functions for Multivariate Data}
\label{subsec:depth}

In absence of any 'natural order' on $\mathbb{R}^d$ with $d\geq2$, the concept of \textit{statistical depth} permits to define a center-outward ordering of points in the support of a probability distribution $P$ on $\mathbb{R}^d$, so as to extend the notions of order and (signed) rank statistics to multivariate data, see \textit{e.g.} \cite{Mosler13}.
A depth function $D_P:\mathbb{R}^d\rightarrow \mathbb{R}_+$ relative to $P$ should ideally assign the highest values $D_P(x)$  to points $x\in \mathbb{R}^d$ near the "center" of the distribution.
Originally introduced in the seminal contribution \cite{Tukey75}, the \textit{half-space depth} of $x$ in $\mathbb{R}^d$ relative to $P$ is the minimum of the mass $P(H)$ taken over all closed half-spaces $H\subset \mathbb{R}^d$ such that $x\in H$. Many alternatives have been proposed since then, see \textit{e.g.} \cite{Liu, LiuS93, koshevoy1997, Chaudhuri, OJA1983, Vardi1423, chernozhukov2017, ZuoSerfling00}. To compare the merits and drawbacks of different notions of depth function, an axiomatic nomenclature has been introduced in \cite{ZuoSerfling00}, listing four properties that statistical depths should ideally satisfied, see \cite{Dyckerhoff04, Mosler13} for a different formulation of a statistically equivalent set of properties.

\begin{itemize}
\item[$(i)$] {\sc (Affine invariance)} Denoting by $P_X$ the distribution of any r.v. $X$ taking its values in $\mathbb{R}^d$, it holds: $D_{P_{AX+b}}(Ax+b)=D_P(x)$ for all $x\in \mathbb{R}^d$, any r.v. $X$ valued in $\mathbb{R}^d$, any $d\times d$ nonsingular matrix $A$ with real entries and any vector $b$ in $\mathbb{R}^d$.
\item[$(ii)$] {\sc (Maximality at center)} For any probability distribution $P$ on $\mathbb{R}^d$ that possesses a symmetry center $x_P$ (for different notions of center), the depth function $D_P$ takes its maximum value at it, \textit{i.e.} $D_P(x_P)=\sup_{x\in \mathbb{R}^d}D_P(x)$.
\item[$(iii)$] {\sc (Monotonicity relative to deepest point)} For any probability distribution $P$ on $\mathbb{R}^d$ with deepest point $x_P$, the depth at any point $x$ in $\mathbb{R}^d$ decreases as one moves away from $x_P$ along any ray passing through it, \textit{i.e.}  $D_P(x)\leq D_P(x_P+\alpha(x-x_P))$ for any $\alpha$ in $[0,1]$.
\item[$(iv)$] {\sc (Vanishing at infinity)} For any probability distribution $P$ on $\mathbb{R}^d$, the depth function $D_P$ vanishes at infinity,  \textit{i.e.} $D_P(x)\rightarrow 0$  as $\vert\vert x\vert\vert$ tends to infinity.
\end{itemize}

As the distribution $P$ of interest is generally unknown in practice, its analysis relies on the observation of $N\geq 1$ independent realizations $X_1,\; \ldots,\; X_N$ of $P$. A statistical version of $D_P(x)$ can be built by replacing $P$ with its empirical counterpart $\widehat{P}_N=(1/N)\sum_{i=1}^N\delta_{X_i}$, yielding the \textit{empirical depth function} $D_{\widehat{P}_N}(x)$. Its consistency and asymptotic normality have been studied for various notions of depth, refer to \textit{e.g.} \cite{DonohoG92, ZuoS00b}, and concentration results for empirical depth and contours have been recently proved in the half-space depth case, see \cite{BurrF17, Brunel}.

\subsection{Consensus Ranking}
\label{subsec:consensus_ranking}

Given a certain metric $d(.,\; .)$ on $\mathfrak{S}_n$ and a r.v. $\Sigma$ defined on a probability space $(\Omega,\; \mathcal{F},\; \mathbb{P})$ and drawn from an unknown probability distribution $P$ on $\mathfrak{S}_n$ (\textit{i.e.} $P(\sigma)=\mathbb{P}\{ \Sigma=\sigma \}$ for any $\sigma\in \mathfrak{S}_n$), the metric approach to consensus ranking consists in finding a ranking $\sigma^*\in \mathfrak{S}_n$ whose expected distance to $\Sigma$ is minimum, \textit{i.e.} such that
\begin{equation}\label{eq:consensus}
L_P(\sigma^*)=\min_{\sigma\in \mathfrak{S}_n}L_P(\sigma),
\end{equation}
where $L_P(\sigma)=\mathbb{E}_{P}[d(\Sigma,\sigma)  ]$ is referred to as the \textit{ranking risk} of any median candidate $\sigma$ in $\mathfrak{S}_n$ w.r.t. $d$ and $\Sigma$. The ranking median $\sigma^*$ (not necessarily unique) is viewed as an informative summary of $P$ and $L_P(\sigma^*)$ as a dispersion measure.
The choice of the (pseudo) distance $d(.,\; .)$ is crucial, regarding the theoretical properties of the corresponding medians and the computational feasibility, see section \ref{sec:depth_for_rankings}. Various  distances have been considered in the literature, see \textit{e.g.} \cite{DH98}, the most popular choices being listed below: $\forall (\sigma,\sigma')\in \mathfrak{S}_n^2$,

\begin{eqnarray*}
d_{\tau}(\sigma,\sigma')&=&\sum_{i<j}\mathbb{I}\left\{\left(\sigma(i)-\sigma(j)) (\sigma'(i)-\sigma'(j)\right)<0\right\},\\
d_2(\sigma,\sigma')&=&\left(\sum_{i=1}^n(\sigma(i)-\sigma'(i)
  )^2\right)^{1/2},\\
d_1(\sigma,\sigma')&=&\sum_{i=1}^n\left\vert\sigma(i)-\sigma'(i)
\right\vert,\\
d_H(\sigma,\sigma')&=&\sum_{i=1}^n\mathbb{I}\left\{ \sigma(i)\ne \sigma'(i)\right\},
\end{eqnarray*}
known respectively as the Kendall $\tau$, the Spearman $\rho$, the Spearman footrule and the Hamming distances.
The literature has essentially focused on solving a statistical version of the minimization problem \eqref{eq:consensus}, see \textit{e.g.} \cite{Hudry08}, \cite{DG77} or \cite{bartholdi1989computational}. Assuming that $N\geq 1$ independent copies $\Sigma_1,\; \ldots,\; \Sigma_N$ of the generic r.v. $\Sigma$ are observed, a natural empirical estimate of $L_P(\sigma)$ is
$\widehat{L}_N(\sigma)=(1/N)\sum_{s=1}^Nd(\Sigma_s,\sigma)=L_{\widehat{P}_N}(\sigma)$,
where $\widehat{P}_N=(1/N)\sum_{i=1}^N\delta_{\Sigma_i}$ is the empirical measure.
The set $\mathfrak{S}_n$ being of finite cardinality, an empirical ranking risk minimizer always exists, just like a solution to \eqref{eq:consensus}, not necessarily unique however. Generalization guarantees and fast rate conditions for empirical consensus ranking have been investigated in \cite{CKS17}.




\section{Depth Functions for Ranking Data}
\label{sec:depth_for_rankings}

In order to define relevant extensions of the concept of statistical depth to ranking data, we define axiomatic properties that candidate functions on $\mathfrak{S}_n$ should satisfy. We next show that the metric-based ranking depths we propose to analyze ranking distributions satisfy these properties under mild conditions.

\subsection{Ranking Depths - Axioms}
\label{subsec:desirable_prop}

Just like in the multivariate setup (see subsection \ref{subsec:depth}), a list of key properties the ranking depth function $D_P$ should ideally satisfy can be made. These properties are essential to emulate the information provided by quantiles (resp. quantile regions) of univariate distributions (resp. multivariate distributions) in a relevant manner.
Let $P$ be a ranking distribution, $d$ a distance on $\mathfrak{S}_n$, the properties desirable for any ranking depth $D_P:\mathfrak{S}_n\rightarrow \mathbb{R}_+$ are listed below.

\begin{restatable}{property}{propertyinvariance}
\label{property:invariance}
{\sc (Invariance)} For any $\pi\in \mathfrak{S}_n$, consider the ranking distribution $\pi P$ defined by: $(\pi P)(\sigma) = P(\sigma\pi^{-1})$ for all $\sigma\in \mathfrak{S}_n$. It holds that: $D_P(\sigma) = D_{\pi P}(\sigma\pi)$ for all $(\sigma,\; \pi)\in \mathfrak{S}_n^2$.
\end{restatable}

\begin{restatable}{property}{propertymaximality}
\label{property:maximality}
{\sc (Maximality at center)} 
For any probability distribution $P$ on $\mathfrak{S}_n$ that possesses a symmetry center $\sigma_P$ (in a certain sense, \textit{e.g.} w.r.t. to a given metric $d$ on $\mathfrak{S}_n$), the depth function $D_P$ takes its maximum value at it, \textit{i.e.} $D_P(\sigma_P)=\max_{\sigma\in \mathfrak{S}_n}D_P(\sigma)$.
\end{restatable}

\begin{restatable}{property}{propertylocalmonotonicity}
\label{property:local_monotonicity}
{\sc (Local monotonicity relative to deepest ranking)} Assume that the deepest ranking $\sigma^*$ is unique. The quantity $D_P(\sigma)$ decreases as $d(\sigma^*,\sigma)$ locally increases, i.e. for any $\pi$ such that $d(\sigma^*, \sigma \pi) = d(\sigma^*, \sigma) + 1$, then we have $D_P(\sigma) > D_P(\sigma \pi)$. 
\end{restatable}

Note that, insofar as $\mathfrak{S}_n$ is of finite cardinality, there is no relevant analogue of the 'vanishing at infinity' property for multivariate depth. A stronger monotonicity property can also be formulated.

\begin{restatable}{property}{propertyglobalmonotonicity}
\label{property:global_monotonicity}
{\sc (Global monotonicity)} Assume that the deepest ranking $\sigma^*$ is unique. The quantity $D_P(\sigma)$ decreases as $d(\sigma^*,\sigma)$ globally increases, i.e. $d(\sigma^*, \sigma') > d(\sigma^*, \sigma) \Rightarrow D_P(\sigma') < D_P(\sigma)$. 
\end{restatable}

\subsection{Metric-based Ranking Depth Functions}
\label{subsec:depth_def}

Seeking to define a ranking depth that satisfies the properties listed above and such that the medians $\sigma_P^*$ of $P$ have maximal depth, the metric approach provides natural candidates, just like for consensus ranking.

\begin{definition}{\sc (Metric-based ranking depth)}
\normalfont Let $d$ be a distance and $P$ a distribution on $\mathfrak{S}_n$. The ranking depth based on $d$ is defined as:  $D_P^{(d)}$: $\forall \sigma \in \mathfrak{S}_n$, 
$D_P^{(d)}(\sigma)=\mathbb{E}_P[\vert\vert d\vert\vert_{\infty}-d(\sigma,\Sigma)]=\vert\vert d\vert\vert_{\infty}-L_P(\sigma)$, with $\vert\vert d\vert\vert_{\infty}=\max_{(\sigma,\sigma')\in \mathfrak{S}_n^2}d(\sigma,\sigma')$.
\label{def:depth_rankings}
\end{definition}
 The shift induced by $\vert\vert d\vert\vert_{\infty} \geq L^{\star}=\max_{\sigma\in \mathfrak{S}_n}L_P(\sigma)$ simply guarantees non-negativity, in accordance with Definition 2.1 in \cite{ZuoSerfling00}, while defining the same center-outward ordering of the permutations $\sigma$ in $\mathfrak{S}_n$ as $-L_P$. Notice that metric-based ranking depths can be viewed as extensions of multivariate depth functions of type A in the nomenclature proposed in \cite{ZuoSerfling00}. For simplicity, we omit the superscript $(d)$ and rather write $D_P$ when no confusion is possible about the distance considered.
\par A ranking $\sigma$ in $\mathfrak{S}_n$ is said to be \textit{deeper} than another one $\sigma'$  relative to the ranking distribution $P$ iff $D_P(\sigma')\leq D_P(\sigma)$ and we write $\sigma'\preceq_{D_P} \sigma$. The \textit{ranking depth ordering} $\preceq_{D_P}$ is the preorder related to the depth function $D_P$.
Equipped with this notion of depth on $\mathfrak{S}_N$, medians $\sigma^*$ of $P$ w.r.t. the metric $d$ correspond to the deepest rankings.
If $P$ is a Dirac mass $\delta_{\sigma_0}$, the ranking depth then simply reduces to the measure of closeness defined by the distance $d$ chosen: $D_P(\sigma)=\vert\vert d\vert\vert_{\infty}-d(\sigma_0,\sigma)$. In contrast, if $P$ is the uniform distribution, the ranking depth relative to a classic distance on $\mathfrak{S}_n$ is constant over $\mathfrak{S}_n$.
The depth function also permits to partition the space $\mathfrak{S}_n$ into subsets of rankings with equal depth.
 
\begin{definition}{\sc (Depth regions/contours)}
\label{def:ranking_regions}
\normalfont For any $u\in \mathbb{R}$, the region of depth $u$ is the superlevel set $\mathcal{R}_P(u)=\{\sigma\in \mathfrak{S}_n:\; D_P(\sigma)\geq u \}$ of $D_P$, while the ranking contour of depth $u$ is the set $\partial \mathcal{R}_P(u)=\{\sigma\in \mathfrak{S}_n:\; D_P(\sigma)= u \}$.
\end{definition}

Equipped with this notation, $\partial \mathcal{R}_P(-L_P^*)$ is the set of medians of $P$ w.r.t. the metric $d$.

\begin{definition}{\sc (Depth survivor function)}
\label{def:survivor_fct}
\normalfont The ranking depth survivor function is $S_P:u\in \mathbb{R}\mapsto S_P(u)=\mathbb{P}\{ D_P(\Sigma)\geq u\}$.
\end{definition}

Based on the metric-based ranking depth, the quantile regions are defined as follows.

\begin{definition}{\sc (Quantile regions in $\mathfrak{S}_n$)}
\label{def:quantile_regions}
\normalfont Let $\alpha\in(0,1)$. The depth region with probability content $\alpha$ is the region of depth $S_P^{-1}(\alpha)=\inf\{u\in \mathbb{R}:\; S_P(u)\leq 1-\alpha\}$: $R_P(\alpha)=\mathcal{R}_P(S_P^{-1}(\alpha))$. The mapping $\alpha\in (0,1)\mapsto S_P^{-1}(\alpha)$ is called the ranking quantile function.
\end{definition}


\subsection{The Metric Approach - Main Properties}
\label{subsec:depth_satisfy_prop}

We now state results showing that, under mild conditions and for popular choices of $d$, the metric-based ranking depth introduced in Definition \ref{def:depth_rankings} satisfies the key properties listed in subsection \ref{subsec:desirable_prop}. Technical proofs are postponed to the Supplementary Material.

\begin{restatable}{proposition}{propinvariance}
\label{prop:invariance}
{\sc (Invariance)} Suppose that $d$ is right-invariant, i.e. $d(\nu \pi, \sigma \pi) = d(\nu, \sigma)$ for all $(\nu,\pi,\sigma)\in \mathfrak{S}_n^3$, the ranking depth $D_P^{(d)}$ satisfies the Property 1.
\end{restatable}
We point out that Spearman $\rho$, Spearman footrule, Kendall $\tau$, Hamming, Ulam and Cayley distances are all right-invariant. Hence, the invariance property is satisfied for any ranking distribution in many situations. Checking the other properties is more challenging. We recall the following notion.
\begin{definition}{\sc (Stochastic transitivity)}
\label{def:stoch_trans}
A probability distribution $P$ on $\mathfrak{S}_n$ is said to be stochastically transitive (ST) iff, for all $(i,j,k)\in \n^3$, we have: $p_{i,j}\geq 1/2 \text{ and } p_{j,k}\geq 1/2 \; \Rightarrow\; p_{i,k}\geq 1/2.$ If, in addition, $p_{i,j}\neq 1/2$ for all $i<j$, one says that $P$ is strictly stochastically transitive (SST).
\end{definition}

The \textit{stochastic transitivity} property \cite{Fishburn73, davidson1959experimental} is fulfilled by some widely used ranking distributions (\textit{e.g.} Mallows) and shown to facilitate various statistical tasks, see \textit{e.g.} \cite{shah2015stochastically, shah2015simple}. In particular, if $P$ is SST, Kemeny's median (\textit{i.e.} the median $\sigma^*$ w.r.t. Kendall $\tau$ distance) is unique, see \textit{e.g.} \cite{CKS17}.

\begin{restatable}{proposition}{propmaximality}
\label{prop:maximality}
{\sc (Maximality at the center)}: The Spearman's footrule ranking depth satisfies Property 2 for any distribution $P$ with a symmetry center. If $P$ is SST in addition, then Kendall $\tau$ ranking depth satisfies Property 2 as well.
\end{restatable}

\begin{restatable}{proposition}{proplocalmonotonicity}
\label{prop:local_monotonicity}
{\sc (Local monotonicity)} If the distribution $P$ is SST, then the Kendall $\tau$ ranking depth satisfies Property 3.
\end{restatable}

\begin{restatable}{proposition}{propglobalmonotonicity}
\label{prop:global_monotonicity}
{\sc (Global monotonicity)} If the distribution P is SST and $\vert\vert d_{\tau}\vert\vert_{\infty}=\binom{n}{2}<h/s$ with $h = \min_{i,j} \vert p_{i,j} - 1/2 \vert$ and $s = \max_{(i,j)\neq (k,l)} \vert p_{i,j} - p_{k,l}\vert$, then the Kendall $\tau$ ranking depth satisfies Property 4.
\end{restatable}

In the Kendall $\tau$ case,  additional useful results can be stated. In particular, the ranking depth is then entirely determined by the \textit{pairwise probabilities} $p_{i,j}=\mathbb{P}\{  \Sigma(i)<\Sigma(j)\}$, $1\leq i\neq j\leq n$.
\begin{restatable}{proposition}{propdepthkendall}
\label{prop:depth_kendall}
We have: $\forall \; \sigma \in \mathfrak{S}_n$, 
$D_P(\sigma)=\binom{n}{2}-\sum_{i<j} p_{i,j}\mathbb{I}\{\sigma(i)>\sigma(j)\} -\sum_{i<j} (1-p_{i,j})\mathbb{I}\{\sigma(i)<\sigma(j) \}$.
\end{restatable}

This case is computationally attractive, the complexity being of order $O(n^2)$. In addition, note that the computation of $D_P$ involves pairwise comparisons solely, which means an alternative statistical framework can be considered, where observation take the form of binary variables $\{\Sigma(\mathbf{i})<\Sigma(\mathbf{j})\}$, $(\mathbf{i},\; \mathbf{j})$ being a random pair in $\{(i,j):\; 1\leq i<j\leq n\}$, independent from $\Sigma$. 


 

\begin{restatable}{proposition}{propkendall}
\label{prop:Kendall}
Suppose that the ranking distribution $P$ is stochastichally transitive. The following assertions hold true.
\begin{itemize}
\item[(i)] The largest ranking depth value is $D_P^*= \sum_{i<j}\left\{ \frac{1}{2}+\left\vert p_{i,j}-\frac{1}{2} \right\vert  \right\}$. The deepest rankings relative to $P$ and $d_{\tau}$ are the permutations $\sigma\in\mathfrak{S}_n$ such that: $\forall i<j$ s.t. $p_{i,j}\neq 1/2$, $(\sigma(j)-\sigma(i))\cdot (p_{i,j}-1/2)>0$.
\item[(ii)] The smallest ranking depth value is $\underline{D}_P= \sum_{i<j}\left\{ \frac{1}{2}-\left\vert p_{i,j}-\frac{1}{2} \right\vert  \right\}$. The least deep rankings relative to $P$ and $d_{\tau}$ are the permutations $\sigma\in\mathfrak{S}_n$ such that: $\forall i<j$ s.t. $p_{i,j}\neq 1/2$, $(\sigma(j)-\sigma(i))\cdot (p_{i,j}-1/2)<0$.
\item [(iii)] If, in addition, $P$ is SST, then we have $\partial\mathcal{R}_P(D_P^*)=\{\sigma^*\}$ and $\partial\mathcal{R}_P(\underline{D}_P)=\{\underline{\sigma}\}$, where  $\sigma^*(i)=1+\sum_{j\neq i}\mathbb{I}\{p_{i,j}<1/2  \}=n-\underline{\sigma}(i)$ for  $i\in\{1, \; \ldots,\;  n\}$. We also have $D_P^*-D_P(\sigma)=2\sum_{i<j}\vert p_{i,j}-1/2 \vert +D_P(\sigma)-\underline{D}_P
=2\sum_{i<j}\vert p_{i,j}-1/2 \vert \cdot \mathbb{I}\{ (\sigma(j)-\sigma(i))(p_{i,j}-1/2)<0 \}$.
\end{itemize}
\end{restatable}


\section{Statistical Issues}

 The ranking depth $D_P$ is generally unknown, just like the ranking distribution $P$, and must be replaced by an empirical estimate based on supposedly available ranking data in practice. Here we establish nonasymptotic statistical guarantees for the empirical counterpart of the ranking depth and other related quantities. We also propose an algorithm, based on the ranking depth, that permits to build, from any ranking dataset, an empirical ranking distribution fulfilling the crucial (strict) stochastic transitivity property, see subsection \ref{subsec:depth_satisfy_prop}.

\subsection{Generalization - Learning Rate Bounds}
\label{subsec:stat_learning_depth}

Based on the observation of an i.i.d. sample $\Sigma_1,\; \ldots,\; \Sigma_N$ drawn from $P$ with $N\geq 1$, statistical versions of the quantities introduced in subsection \ref{subsec:depth_def} can be built by replacing $P$ with the empirical distribution $\widehat{P}_N$. The empirical ranking depth is thus given by: $\forall \sigma\in \mathfrak{S}_n$, $\widehat{D}_N(\sigma)=D_{\widehat{P}_N}(\sigma)=\vert\vert d\vert\vert_{\infty} -\widehat{L}_N(\sigma)$. 
Similarly, the empirical ranking depth regions are
$\widehat{\mathcal{R}}_N(u)=\{\sigma\in \mathfrak{S}_n:  \widehat{D}_N(\sigma)\geq u\}$ for $u\geq 0$.
In order to build an estimator of the ranking depth survivor function $S_P(u)$ with a tractable dependence structure, a \textit{2-split} trick can be used, yielding the statistic
\begin{equation*}
\widehat{S}_N(u)=\frac{1}{N-\lfloor N/2\rfloor}\sum_{i=1+\lfloor N/2\rfloor}^N\mathbb{I}\{\widehat{D}_{\lfloor N/2\rfloor}(\Sigma_i)\geq u\}.
\end{equation*}

As the r.v. $D_P(\Sigma)$ is discrete, the use of smoothing/interpolation procedures is required to ensure good statistical properties for the survivor function estimator and for the empirical quantiles it defines, see \cite{Sheather90, Ma11}. For instance, a kernel smoothed version of $S_P$ can be computed by means of a non-negative differentiable Parzen-Rosenblatt kernel $K:\mathbb{R}\to \mathbb{R}_+$ s.t. $\vert\vert K'\vert\vert_{\infty}=\sup_{u\in \mathbb{R}}\vert K'(u)\vert <\infty$ and $\int_{\mathbb{R}}K(u)du=+1$ and a smoothing bandwidth $h>0$, namely: $\widetilde{S}_P(u)=K_h*S_P$, which can be estimated by $\widetilde{S}_N(u)=K_h*\widehat{S}_N$, where $K_h(u)=K(u/h)/h$ for $u\in \mathbb{R}$.
One may then define a smooth estimate of the ranking depth region with probability content $\alpha\in [0,1]$ as well:
$\widehat{R}_N(\alpha)=\widehat{\mathcal{R}}_N(\widetilde{S}_N^{-1}(\alpha))$.
The result below provides bounds of order $O_{\mathbb{P}}(1/\sqrt{N})$ for the maximal deviations between $D_P$ (resp. $\widetilde{S}_P$) and its empirical version.

\begin{restatable}{proposition}{propsample}
\label{prop:sample} The following assertions hold true.
\begin{itemize}[leftmargin=4mm]
\item[(i)] For any $\delta\in (0,1)$, we have with probability at least $1-\delta$: $\forall N\geq 1$,
\begin{equation*}
\sup_{\sigma\in \mathfrak{S}_n}\vert \widehat{D}_N(\sigma)-D_P(\sigma)  \vert \leq \vert\vert d\vert\vert_{\infty}\sqrt{\frac{\log(2\; n!/\delta)}{2N}}.
\end{equation*}
\item[(ii)] For any $\delta\in (0,1)$ and $h>0$, we have with probability at least $1-\delta$: $\forall N\geq 1$,
\begin{equation*}
\hspace*{-0.4cm}\sup_{u\geq 0}\vert \widetilde{S}_N(u)- \widetilde{S}_P(u) \vert  \leq   \sqrt{\frac{\log(4/\delta)}{2N}}+\vert\vert d\vert\vert_{\infty}\sqrt{\frac{\log(4n!/\delta)}{2N}}.
\end{equation*}
\end{itemize}
\end{restatable}
For the technical proof, refer to the Supplementary Material, where the asymptotic rate for the empirical ranking quantile function is also given.

\subsection{Depth Trimming for Consensus Ranking}
\label{subsec:trimming}

As discussed in subsection \ref{subsec:depth_satisfy_prop}, (strict) stochastic transitivity greatly facilitates the computation of Kemeny medians (see Proposition \ref{prop:Kendall}) as well as the verification of the maximality or monotonicity properties, \textit{cf} Propositions \ref{prop:maximality}, \ref{prop:local_monotonicity} and \ref{prop:global_monotonicity}. However, although this occurs with a controlled probability (see Proposition 14 in \cite{CKS17}), the empirical counterpart $\widehat{P}_N$ of a (strictly) stochastically transitive ranking distribution $P$ can be of course non (S)ST. 
We propose below a trimming strategy based on the empirical ranking depth to recover a close (S)ST empirical ranking distribution and overcome this issue.

\begin{restatable}{algorithm}{algotrimming}
\SetKwInOut{Input}{Input}
\SetKwInOut{Output}{Output}
\Input{Ranking dataset $\mathcal{D}_N=\{\Sigma_1, ... \Sigma_N\}$ and distribution $\widehat{P}_N=(1/N)\sum_{i=1}^N\delta_{\Sigma_i}$.}
\Output{Dataset $\mathcal{D}\subset \mathcal{D}_N$ of size $N_{\mathcal{D}}\leq N$ and (S)ST ranking distribution $\widehat{P}_{\mathcal{D}}=(1/N_{\mathcal{D}})\sum_{\sigma\in \mathcal{D}}\delta_{\sigma}$}

- Initialize: $\mathcal{D} = \mathcal{D}_N$\;
\While{$\hat{P}_{\mathcal{D}}$ is not (S)ST}{
        - Determine the least deep rankings in $\mathcal{D}$: $ \mathcal{O}_{\mathcal{D}}:=\arg \min_{\sigma\in \mathcal{D}} D_{\hat{P}_N}(\sigma)$\;
        - Update the ranking dataset $\mathcal{D}\setminus \mathcal{O}_{\mathcal{D}} \to \mathcal{D}$
}
\caption{Ranking Depth Trimming}
\label{algo:trimming}
\end{restatable}

Based on the ranking dataset $\mathcal{D}$ output by Algorithm \ref{algo:trimming}, a (S)ST empirical distribution $\widehat{P}_{\mathcal{D}}$ can be computed, whose Kemeny medians are obtained in a straightforward manner, \textit{cf} Proposition \ref{prop:Kendall}, avoiding the search of solutions of a NP-hard minimization problem of type \eqref{eq:consensus}, see \cite{Hudry08}. As empirically supported by the experiments displayed in the next section, this procedure allows for a fast, accurate and robust recovery of consensus rankings.
\section{Applications - Experiments}\label{sec:exp}

In order to illustrate the relevance of ranking depth notion, we now show that it can be used to perform a wide variety of tasks in the statistical analysis of ranking data, including those listed below:
\begin{itemize}[leftmargin=4mm, topsep=0mm, noitemsep]
    \item Fast and robust consensus ranking
    \item Ranking data visualization
    \item Detection of outlying rankings
    \item The two-sample (homogeneity) problem in $\mathfrak{S}_n$.
\end{itemize}
Further experimental results on real ranking data are provided and discussed in the Supplementary Material.

\subsection{Fast/Robust Consensus Ranking}

The trimming strategy proposed in sec \ref{subsec:trimming} shows that we can recover smooth SST distributions from any empirical data, and perform ranking aggregation by simply identifying the deepest ranking: this procedure is fast, straightforward, and robust, in the sense that we can recover accurate medians even in contaminated settings. We support this claim by both experiments and a theoretical proposition below.




\begin{figure}[!h]
    \begin{center}
    \begin{tabular}{cc}
        {\scriptsize (a)} & {\scriptsize (b)} \\
        \includegraphics[scale=0.28,clip=true,page=1]{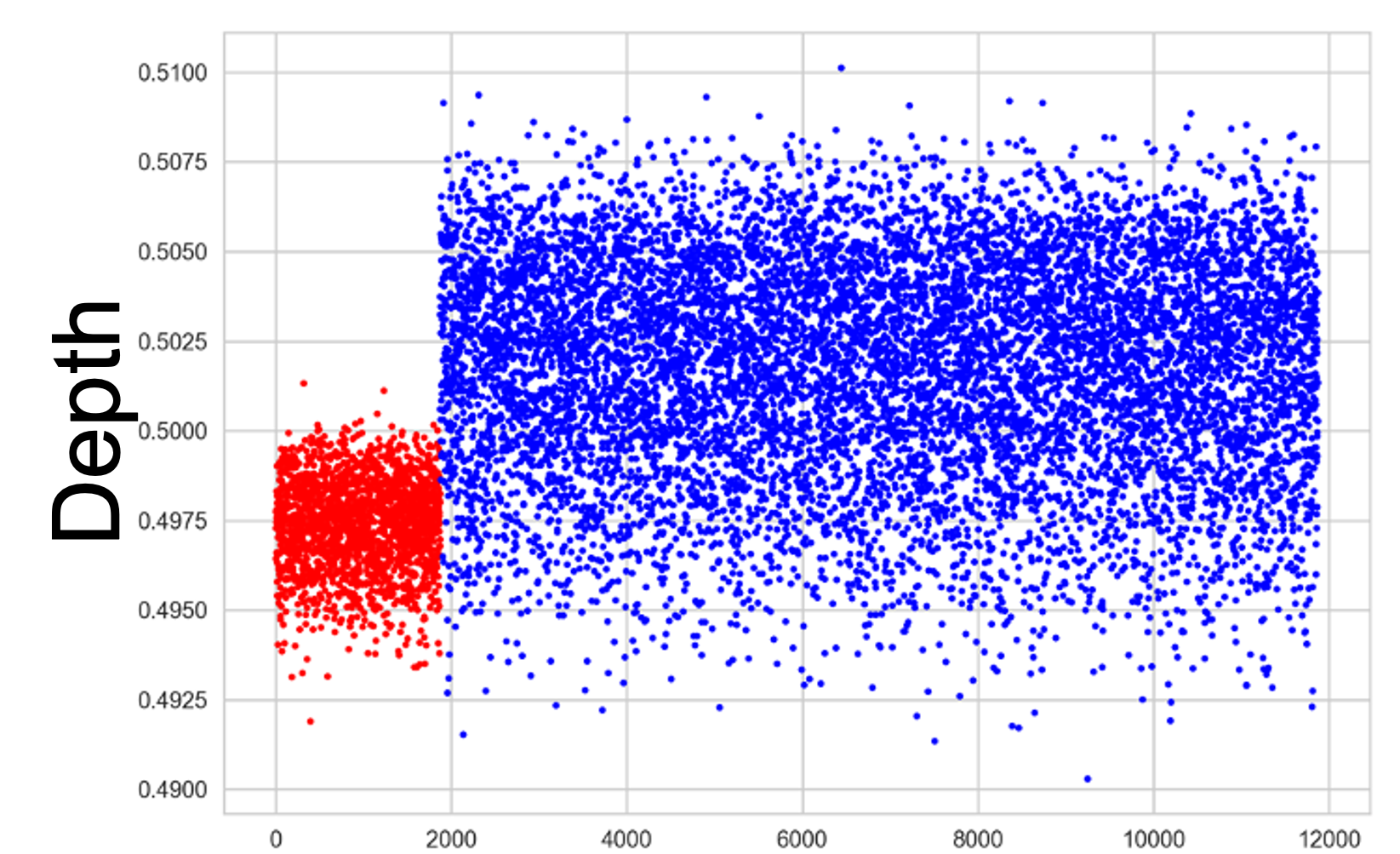} & \includegraphics[scale=0.28,clip=true,page=1]{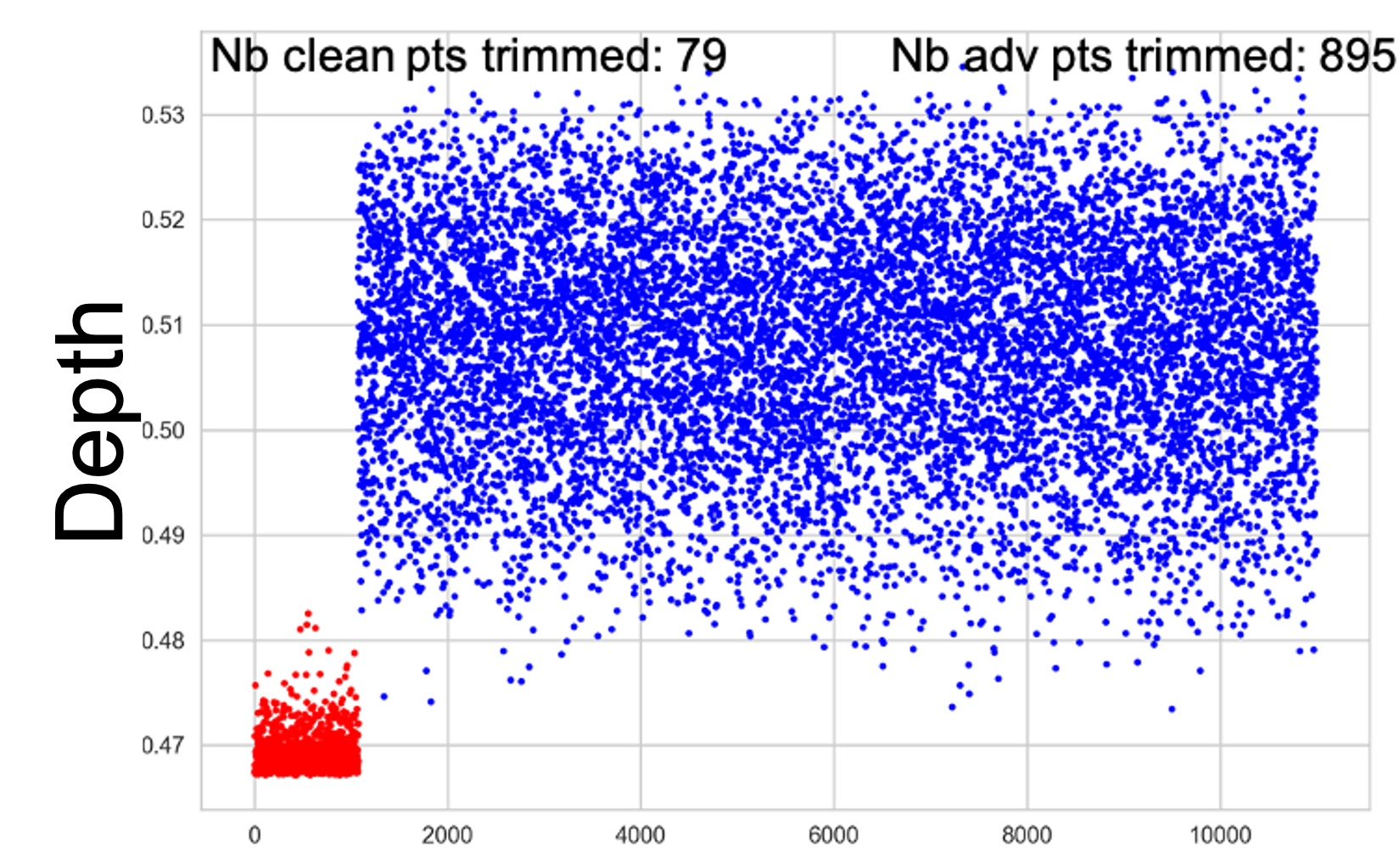} \\
        {\scriptsize (c)} & {\scriptsize (d)} \\
        \includegraphics[scale=0.28,clip=true,page=1]{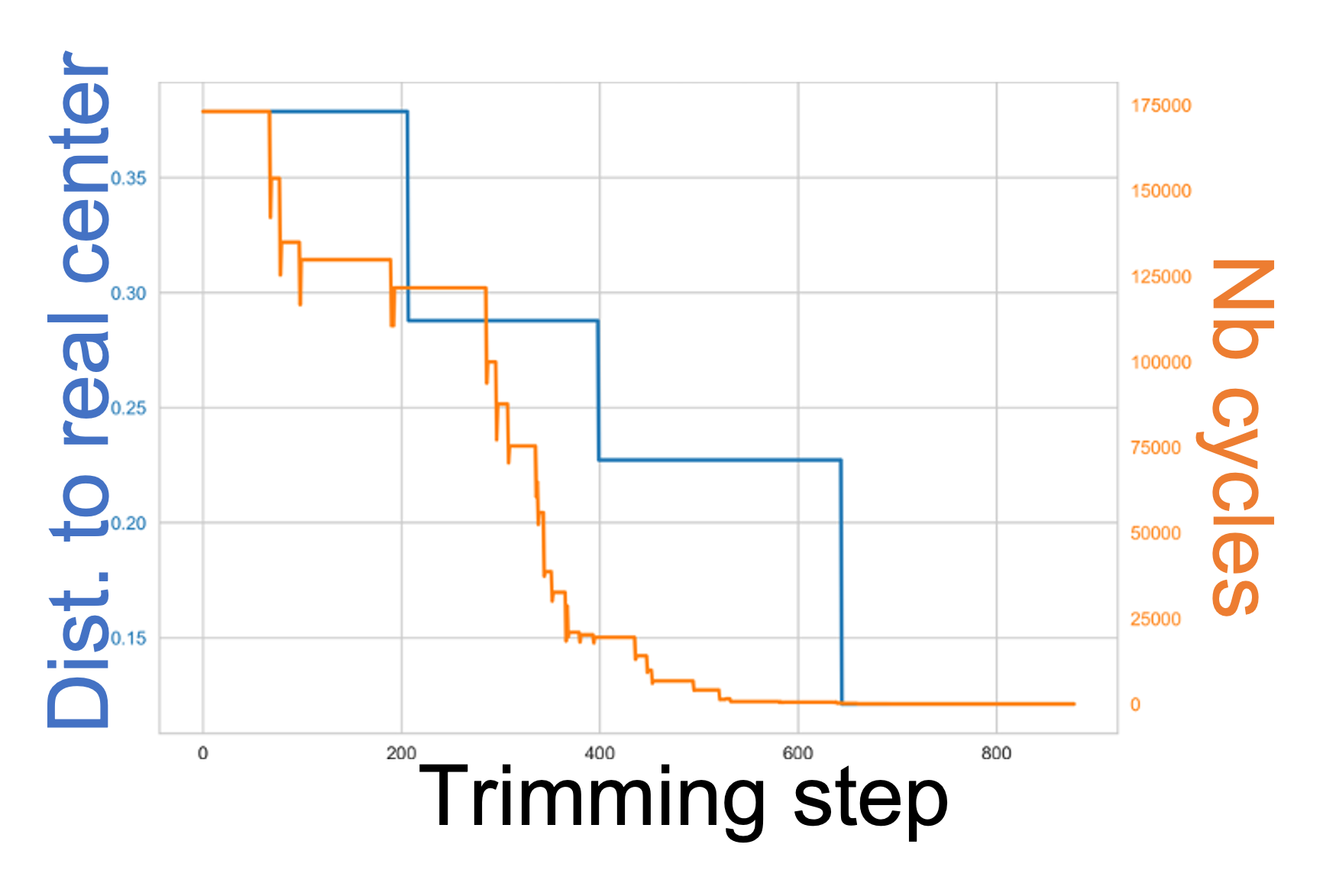} & \includegraphics[scale=0.28,clip=true,page=1]{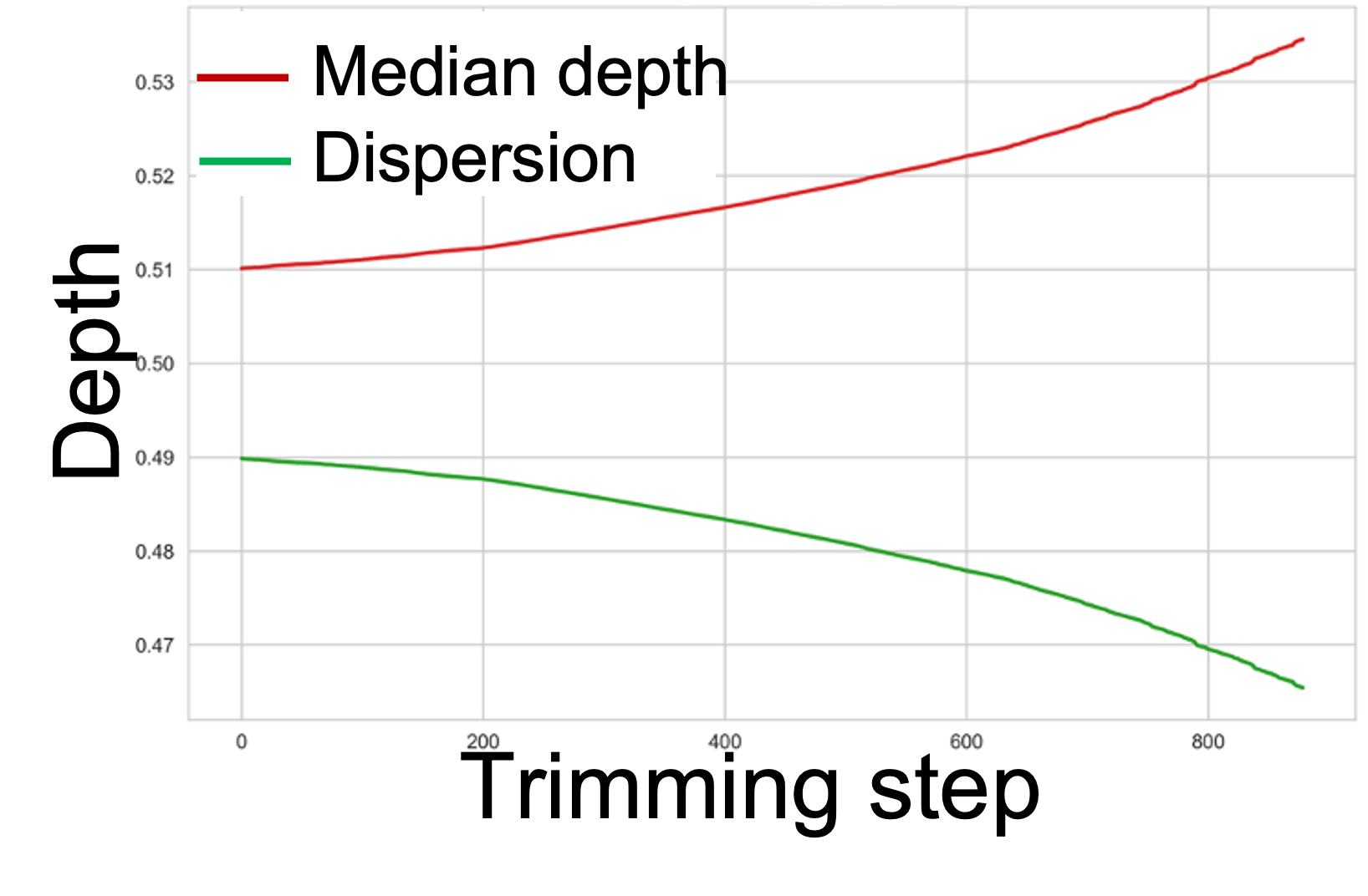} \\
    \end{tabular}
    \caption{Depth plots before (a) and after (b) trimming with adversarial (red) and clean (blue) points; evolution of candidate median (deepest ranking) distance to real median and number of cycles through trimming (c); evolution of median depth and sample dispersion through trimming (d).}
    \label{fig:trimming_plot}
    \end{center}
\end{figure}
We consider a dataset drawn from a "clean" distribution $P$ (10000 points drawn from a Mallows distribution with $n=12$ items, center $\sigma_0$ and $\phi=0.90$) that has been contaminated by rankings from another distribution (2000 points drawn from a Mallows distribution with opposite center and $\phi=0.40$). We use the trimming strategy described in algorithm \ref{algo:trimming} to remove rankings until the empirical distribution becomes SST and thus considered clean once again. We show in Figure \ref{fig:trimming_plot} the depth of clean (blue) and adversarial (red) rankings before trimming (a) and after trimming (b), the performance of the median computed at each step of the trimming procedure evaluated as its Kendall $\tau$ distance to the real center of the clean Mallows distribution (c), and the depth of the median during the trimming procedure (d). The depth function is able to identify mainly adversarial rankings and remove them during the trimming procedure, which conducts to a cleaner dataset after the procedure and a far more accurate median $\sigma^*$.

\paragraph{Mechanical Turk Dots dataset.} We show the robustness of depth-based medians on a real dataset where participants ranked point clouds according to their size \cite{mao2013better}. A ground truth ranking exists, and we contaminated 1/4 of the dataset by swapping random rankings before trimming: figure \ref{fig:mecha_turk_plot} (b) shows that we indeed recovered the ground truth ranking after the trimming strategy even if contaminated rankings were not obviously different from clean one (fig. \ref{fig:mecha_turk_plot} (a)).

\begin{figure}[!h]
    \begin{center}
    \begin{tabular}{cc}
        {\scriptsize (a)} & {\scriptsize (b)} \\
        \includegraphics[scale=0.225,clip=true,page=1]{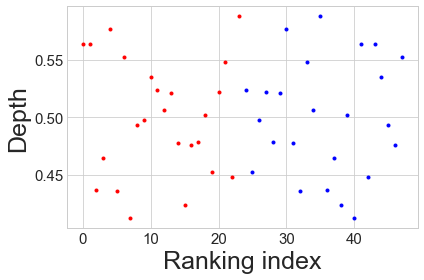} & \includegraphics[scale=0.15,clip=true,page=1]{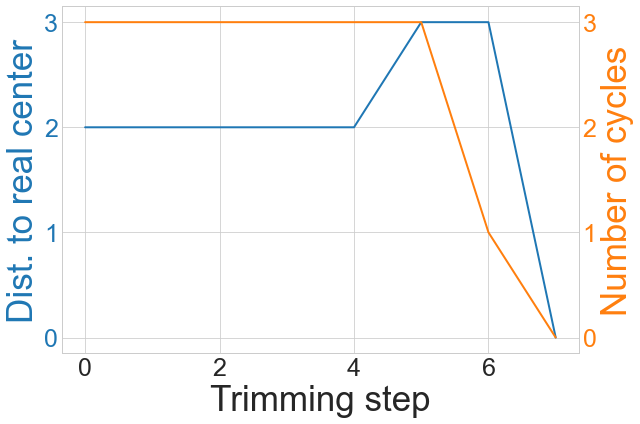}
    \end{tabular}
    \caption{Depth plots before trimming with swapped (red) and clean (blue) points; evolution of candidate median (deepest ranking) distance to real median and number of cycles through trimming (b)}
    \label{fig:mecha_turk_plot}
    \end{center}
\end{figure}

\paragraph{Theoretical robustness result.} We derive specific robustness results when using depth-based trimming by emulating the classical notion of breakdown point (see \cite{DonohoG92}). Let us consider the classical Borda estimator (which orders the items based on the score $B(i) = \sum_{\sigma \in S_N} \sigma(i)$, see \cite{Dwork:2001:RAM:371920.372165, Fligner1988, Caragiannis2013, collas21}) and a \textit{depth-trimmed} Borda estimator based on the scores $B_{\mu}(i) = \sum_{\sigma \in S_N} w(\sigma)\sigma(i)$, where $w(\sigma) = \mathbb{I}(D_N(\sigma) > \mu)$ (only the rankings with depth higher than $\mu$ are kept). Let $\sigma_{S}^{\text{B}}$ (resp. $\sigma_{S}^{\text{DT-B}}$) be the Borda (resp. depth-trimmed Borda) estimator of a sample $S$. The Borda estimator is said to be $\delta$-broken for sample size $N$ and for a distribution $P$ if for any sample $S_N \sim P$ of size $N$, there exists an adversarial sample $A$ such that $d_{\tau}( \sigma_{S_N}^{\text{B}}, \sigma_{S_N \cup A}^{\text{B}} ) \geq \delta$. The smallest cardinality of the adversarial sample $A$ such that the estimator is $\delta$-broken for size $N \to \infty$ is called here the $\delta$-\textit{breakdown points} of the estimator on distribution $P$, and we write $\epsilon^{\text{B}}_{\delta}(P)$ (resp. $\epsilon^{\text{DT-B}}_{\delta}(P)$) such statistic for the Borda (resp. depth-trimmed Borda) estimator. Breakdown points measure the robustness of an estimator on a given distribution: we state that the classical Borda estimator is less robust than the depth-trimmed one on generic distributions. 


\begin{restatable}{proposition}{propborda}
\label{prop:borda_breakdown_ratio}
Let $\mu$ be the trimming threshold and $P$ a distribution such that $\mathbb{E}_{P} [ D_P(\Sigma) ] > \mu$. Let $\sigma^* = \arg \max_{\sigma \in \mathfrak{S}_n} D_P(\sigma)$ be the deepest ranking and $\pi = \arg\max_{\sigma | d_{\tau}(\sigma^*, \sigma) = \delta} D(\sigma)$ the ranking with highest depth among those at distance $\delta$ from the deepest ranking $\sigma^*$. Then, the breakdown points for Borda and depth-trimmed-Borda on $P$ are related as follows,
\begin{equation}
\begin{split}
\frac{\epsilon^{\text{B}}_{\delta}(P)}{\epsilon^{\text{DT-B}}_{\delta}(P)} 
< \frac{ D_{P}( \pi ) }{ \mu } < 1.
\end{split}
\end{equation}
\end{restatable}

The proof, as well as more results on the robustness of Borda estimators, are provided in section \ref{suppl:proof_borda} of the supplementary. 

\subsection{Graphical Methods and Visual Inference}

The analysis of rankings suffers from the lack of graphical displays and diagrams, such as \textit{probability plots} or \textit{histograms}, for gaining insight into the structure of the data. Ranking depths can be readily used to design a visual diagnostic tool for ranking data, extending the Depth \emph{vs.} Depth plot  ($DD$-plot in abbreviated form) originally introduced \cite{LiuPS99} for multivariate data. For two samples of rankings $\boldsymbol{\Sigma^1}=\{\sigma^1_1,\; \ldots,\; \sigma^1_{N_1}\}$ and $\boldsymbol{\Sigma^2}=\{\sigma^2_1,\; \ldots,\; \sigma^2_{N_2}\}$, with corresponding empirical measures $\widehat{P^1}_{N_1}$ and $\widehat{P^2}_{N_2}$, the ranking $DD$-plot is obtained by plotting in the Euclidean plane the points:
\begin{equation}\label{eq:ddplot}
	\bigl\{\bigl(D_{\widehat{P^1}_{N_1}}(\sigma),D_{\widehat{P^2}_{N_2}}(\sigma\bigr)\,:\,\sigma\in\boldsymbol{\Sigma^1}\cup\boldsymbol{\Sigma^2}\bigr\}.
\end{equation}

\begin{table}[h!]
	\begin{center}
	\begin{tabular}{lcllll}
	Position & $d_{\tau}(\sigma_1^*,\sigma_2^*)$ & {\color{red}$\phi_1$} & {\color{darkgreen}$\phi_2$} & {\color{red}$N_1$} & {\color{darkgreen}$N_2$} \\ \hline
	(a) & $15$ & {\color{red}$\mathrm{e}^{-1}$} & {\color{darkgreen}$\mathrm{e}^{-1}$} & {\color{red}$250$} & {\color{darkgreen}$250$} \\
	(b) & $0$ & {\color{red}$\mathrm{e}^{-0.5}$} & {\color{darkgreen}$\mathrm{e}^{-2}$} & {\color{red}$250$} & {\color{darkgreen}$250$} \\
	(c) & $15$ & {\color{red}$\mathrm{e}^{-0.5}$} & {\color{darkgreen}$\mathrm{e}^{-2}$} & {\color{red}$250$} & {\color{darkgreen}$250$} \\
	(d) & $15$ & {\color{red}$\mathrm{e}^{-0.5}$} & {\color{darkgreen}$\mathrm{e}^{-2}$} & {\color{red}$400$} & {\color{darkgreen}$100$} \\
	\end{tabular}
	\caption{Parameters for pairs of samples drawn from Mallows-Kendall distribution used for Figure~\ref{fig:ddplotdiag}. }
	\label{tab:ddplotdiag}
	\end{center}
\end{table}

Depending on the distance $d$ chosen, such a plot allows to reflect location and scatter of two distributions on $\mathfrak{S}_n$, and their mutual position. To illustrate its diagnostic capacity, we plot in Figure~\ref{fig:ddplotdiag} the ranking $DD$-plots relative to the Kendall $\tau$ distance and four pairs of samples stemming from Mallows distribution with parameters defined in Table~\ref{tab:ddplotdiag}. (In this and subsequent figures the depth is re-scaled to $[0,1]$ by diving by $\|d\|_\infty$.) A few remarks can be made: For distributions differing in: 1) location only (a), the ranking $DD$-plot is symmetric w.r.t. the diagonal, 2) scatter only (b), observations from one distribution will be attributed systematically higher depth values, 3) both location and scatter (c), they can be distinguished and 4) number of the observations, it does not influence the general picture (d).

\begin{figure}[!h]
	\begin{center}
	\begin{tabular}{cc}
		{\scriptsize (a)} & {\scriptsize (b)} \\
		\includegraphics[scale=0.23,trim=3mm 0.5mm 0.5mm 14mm,clip=true,page=1]{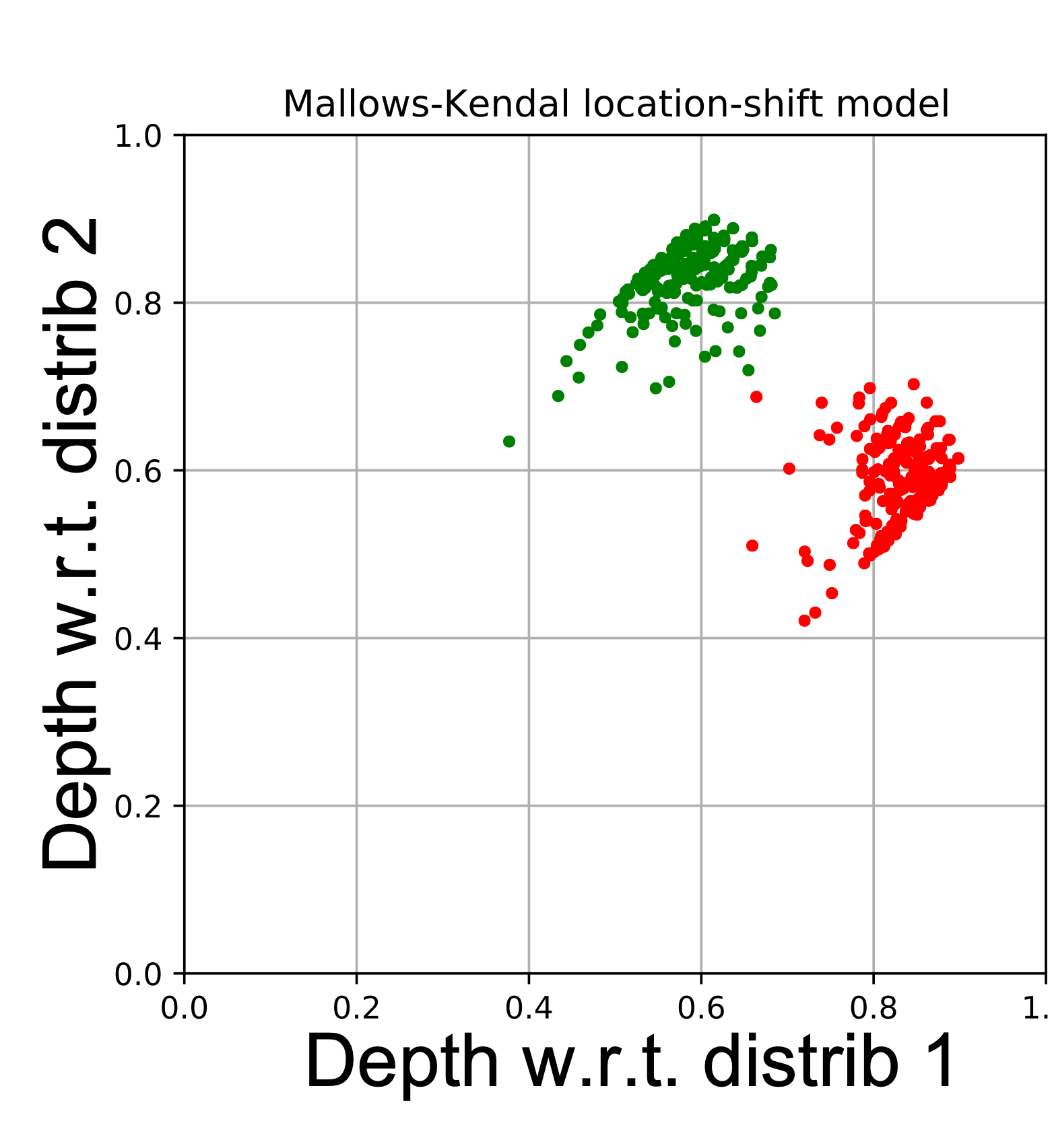} & \includegraphics[scale=0.23,trim=3mm 0.5mm 0.5mm 14mm,clip=true,page=1]{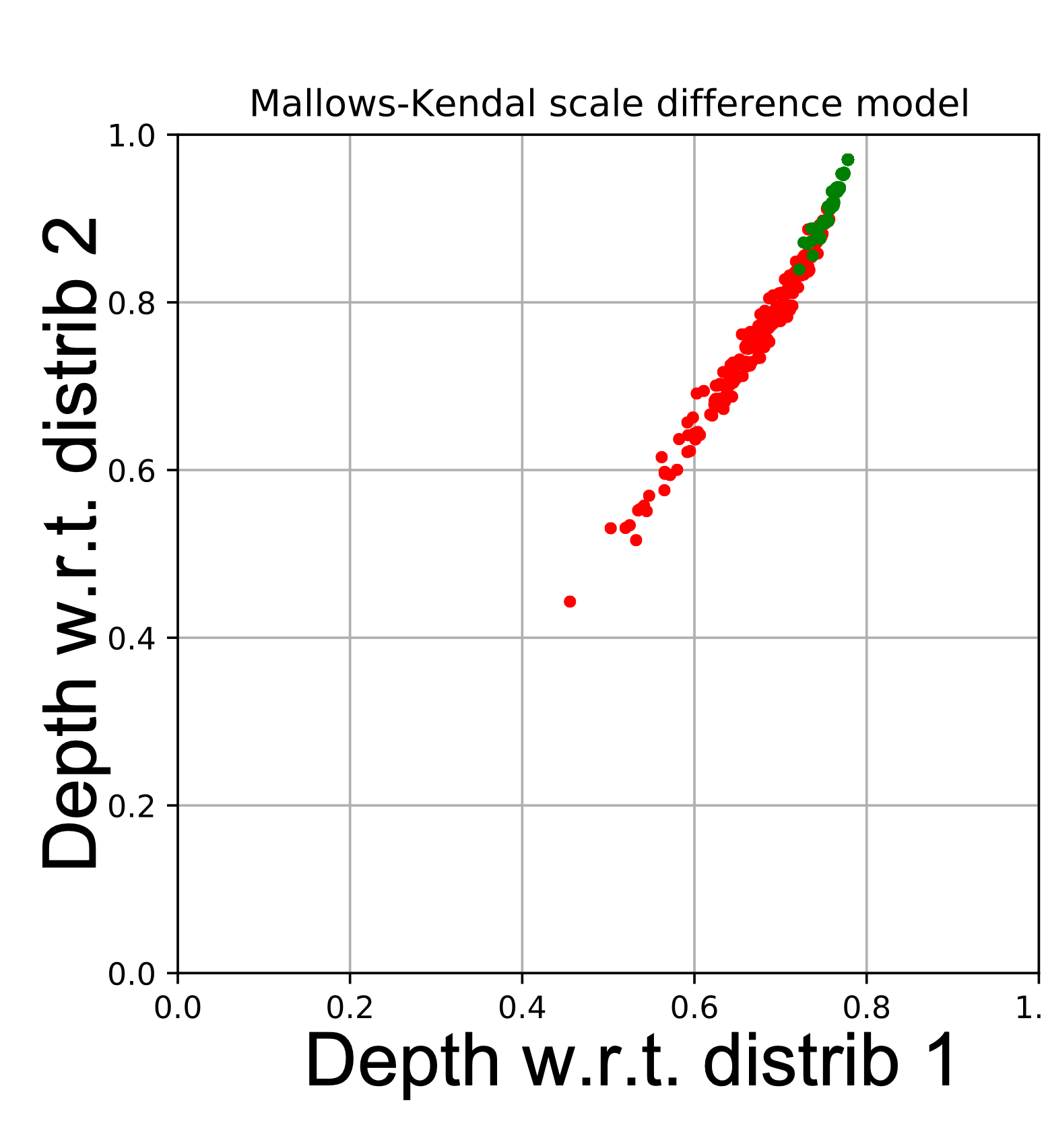}\\
		{\scriptsize (c)} & {\scriptsize (d)} \\
		\includegraphics[scale=0.23,trim=3mm 0.5mm 0.5mm 14mm,clip=true,page=1]{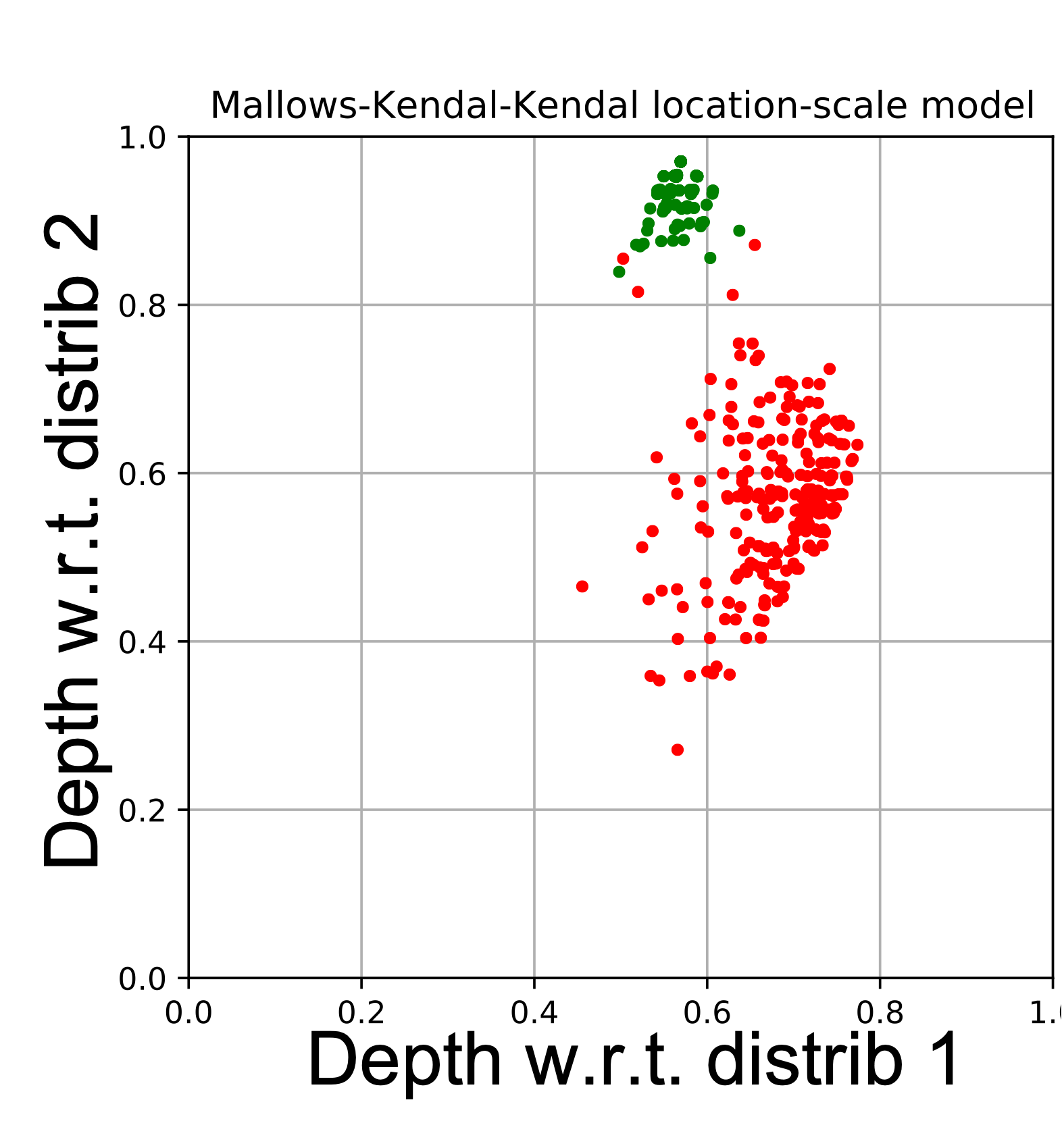} & \includegraphics[scale=0.23,trim=3mm 0.5mm 0.5mm 14mm,clip=true,page=1]{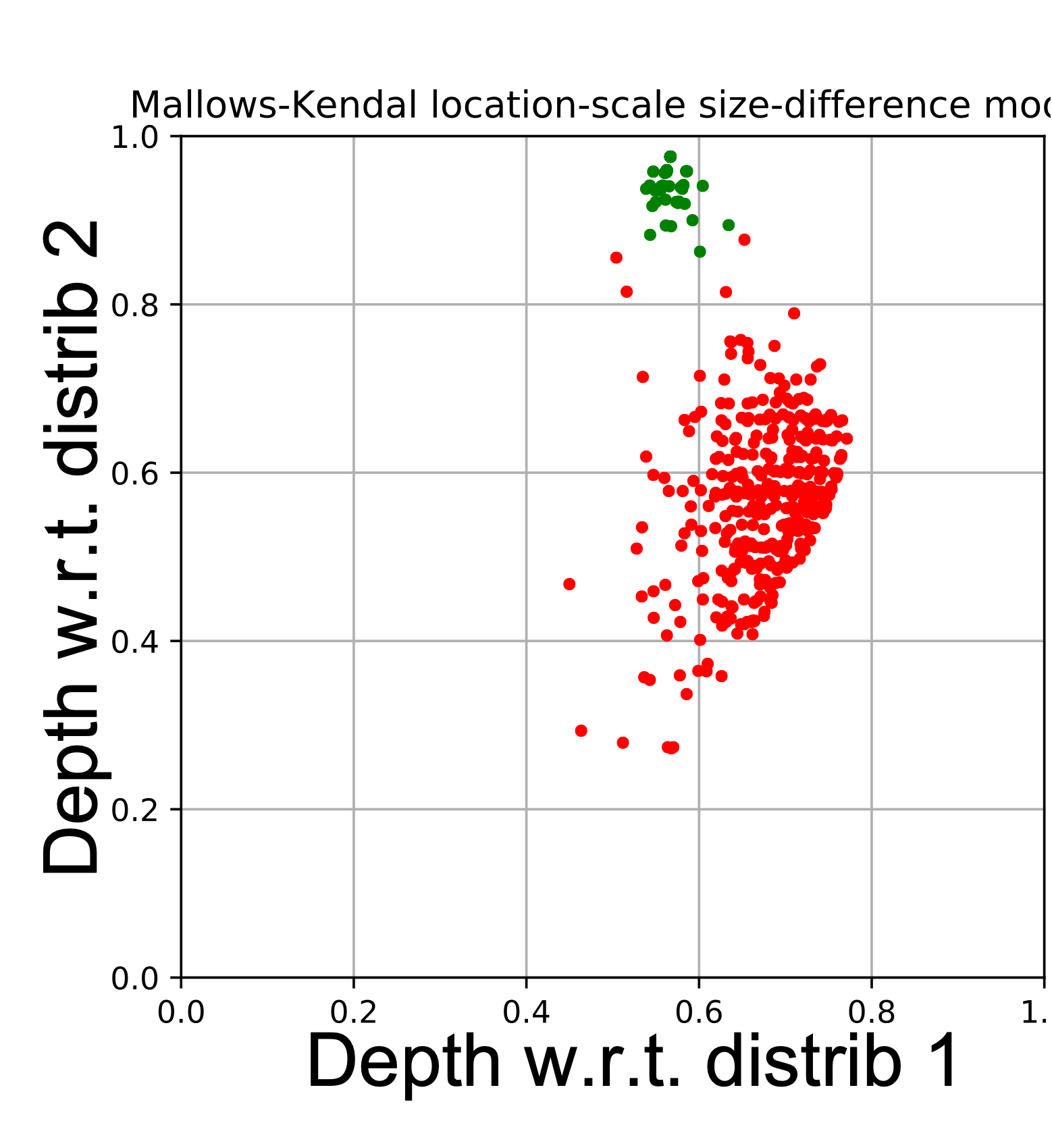} \\
	\end{tabular}
	\caption{Ranking $DD$-plot corresponding to Mallows distributions with parameters described in Table~\ref{tab:ddplotdiag}.}
	\label{fig:ddplotdiag}
	\end{center}
\end{figure}

\subsection{Outlier Detection in Ranking Data}

\begin{figure}[h!]
	\begin{center}
	\begin{tabular}{cc}
		{\scriptsize (a)} & {\scriptsize (b)} \\
		\includegraphics[scale=0.25,trim=3mm 0.5mm 0.5mm 14mm,clip=true,page=1]{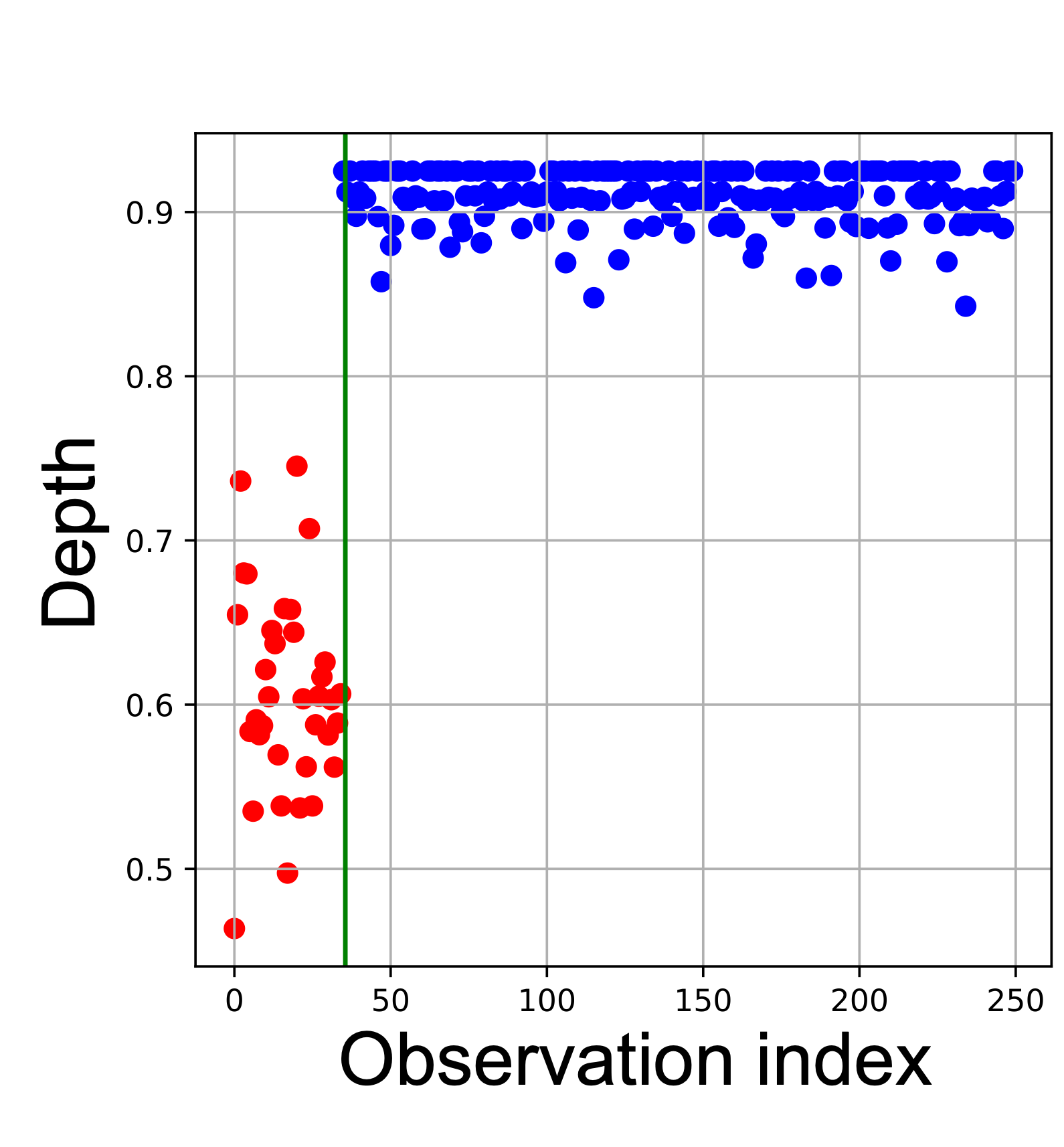} & \includegraphics[scale=0.25,trim=3mm 0.5mm 0.5mm 14mm,clip=true,page=1]{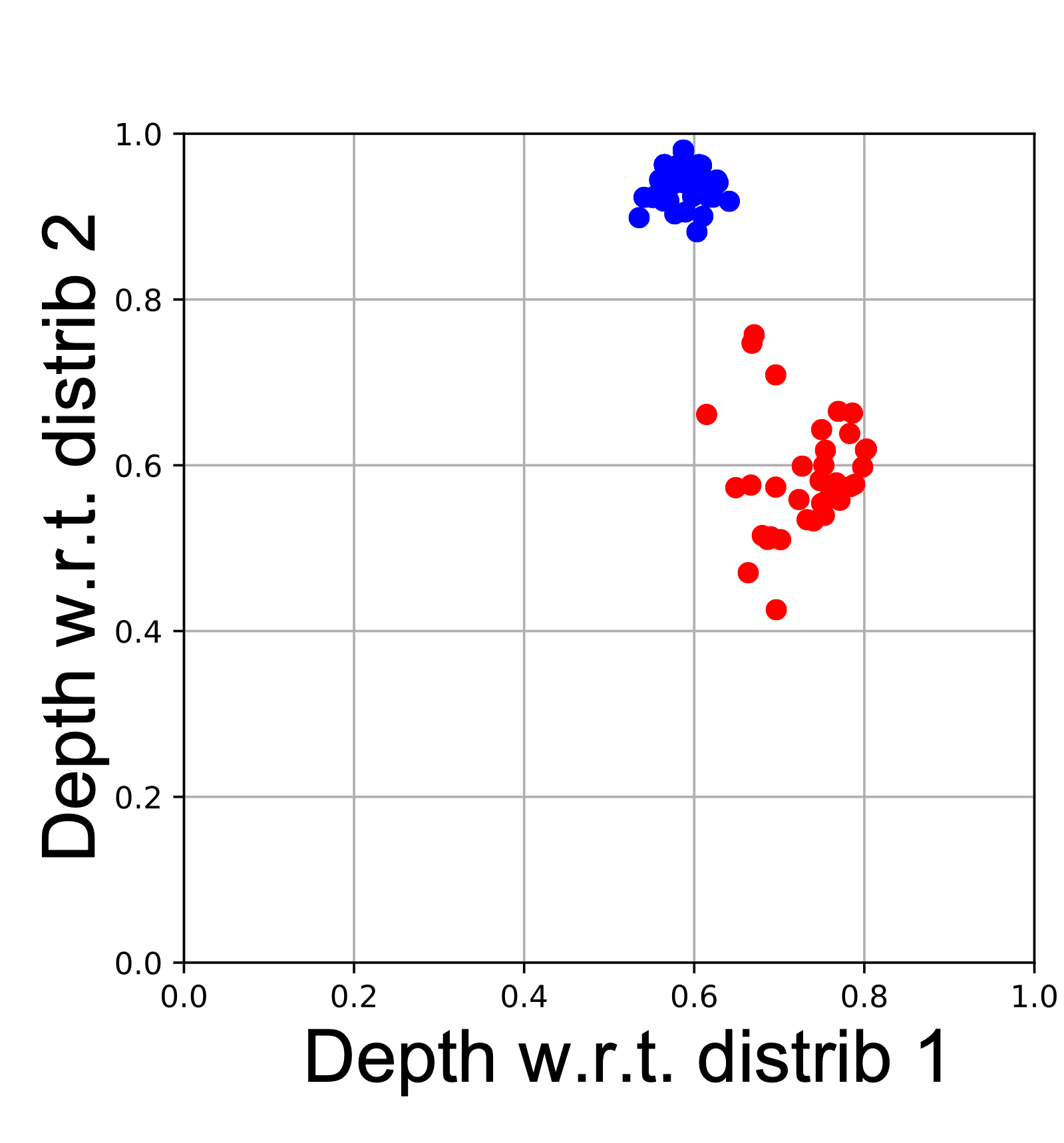} \\
		{\scriptsize (c)} & {\scriptsize (d)} \\
		\includegraphics[scale=0.25,trim=0mm 0.5mm 0.5mm 14mm,clip=true,page=1]{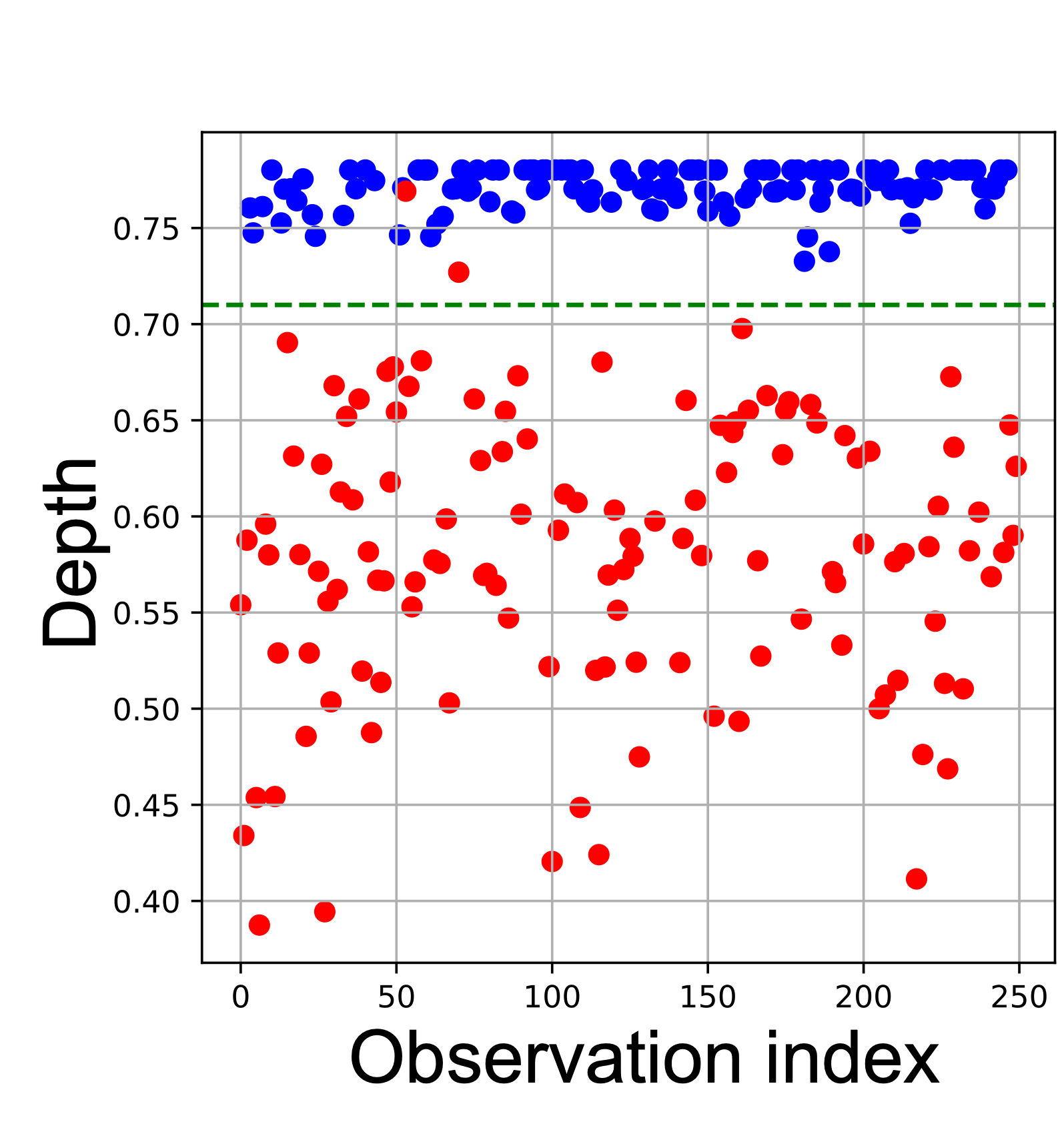} & \includegraphics[scale=0.25,trim=3mm 0.5mm 0.5mm 14mm,clip=true,page=1]{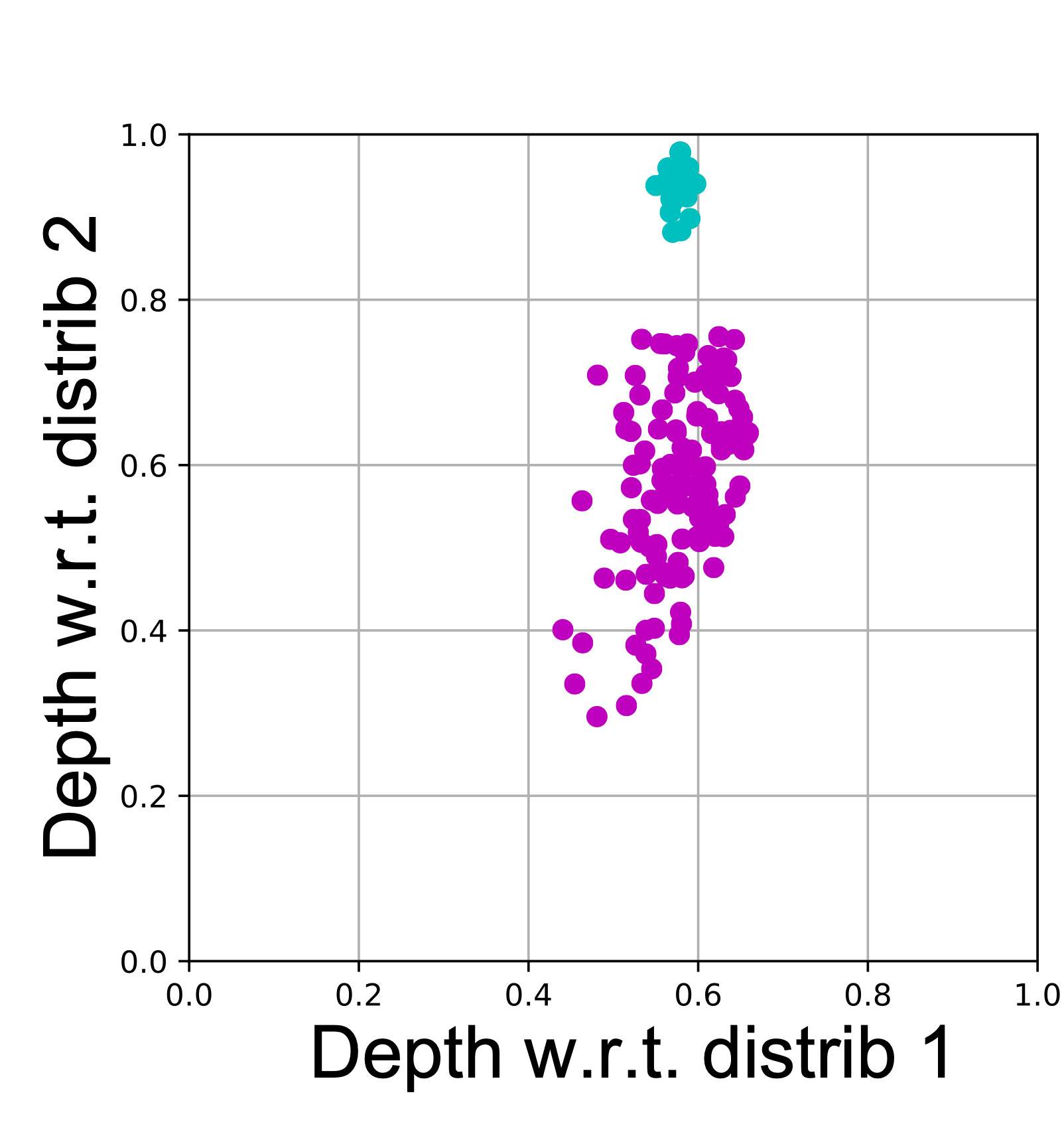} \\
	\end{tabular}
	\end{center}
	\caption{Depth plots (a,c) and $DD$-plots (b,d) for a mixture of Mallows-Kendall distributions. (a)-(b): distant centers and different size for the two components of the mixture. (c)-(d): closer centers and same size.}
	\label{fig:anomdet}
\end{figure}

We now place ourselves in the situation where a single sample of rankings is observed. For simplicity, we consider the case where the underlying ranking distribution is an unbalanced mixture of two Mallows distributions (for $n=10$), strongly differing in size ($N_1=35$ and $N_2=215$), with distant centers ($d_{\tau}(\sigma^*_1,\sigma^*_2)=15$) and parameters $\phi_1=\mathrm{e}^{-0.5}$ and $\phi_2=\mathrm{e}^{-2.5}$. 
Figure~\ref{fig:anomdet} (a) shows the ranking depth (relative to Kendall $\tau$) of each observation computed w.r.t. to the entire sample. We observe, that despite the unavailability of labels, the ranking depth clearly distinguishes the two different components. It thus permits to perform a typical anomaly detection task in the context of ranking data, where the differing minority of permutations are viewed as abnormal rankings. The diagnostic ranking $DD$-plot (b) based on the identified information about the components confirms the differences.

Consider next the case of a mixture with closer centers ($d_{\tau}(\sigma_1^*,\sigma_2^*)=11$) and equal sizes ($N_1=N_2=125$), with parameters $\phi_1=\mathrm{e}^{-0.25}$ and $\phi_2=\mathrm{e}^{-2.5}$. The depth plot (c) w.r.t. to the entire sample reflects how easily we can cluster the ranking dataset into two components (we deliberately shuffle the indices and keep colors for illustrative purposes), and we suggests a separating threshold (on the level of depth $=0.71$), which in this particular case allows for two mistaking assignments. For the diagnostic ranking $DD$-plot (d), we honestly include this mistake, and change the colors to underline this impurity.

\subsection{Rankings - Homogeneity Testing}

Depth can further be used to provide a formal inference, which we exemplify as a nonparametric test of homogeneity between two Plackett-Luce distributions \cite{critchlow91} with $n=10$. The first one (red in Figure~\ref{fig:tests}) is generated using the parameters $\boldsymbol{w}_1=(\mathrm{e}^9,...,\mathrm{e}^0)$, the second one represents its changed version $\boldsymbol{w}_2=(\mathrm{e}^{\gamma 9},...,\mathrm{e}^{\gamma 0})$. We gradually increase $\gamma$ from $0.5$ (substantial difference) to $1$ (equal in distribution), and provide the $p$-values of the Wilcoxon rank-sum test averaged over $100$ repetitions in Figure~\ref{fig:tests}. The test is performed using the reference sample (of size $500$) from the first distribution, with tested sample sizes being equal ($=50$) for both distributions (see \cite{LafayeDeMicheauxMV20} for details on the testing procedure and \cite{LiuS93} for more details). Figure \ref{fig:tests} shows how the $p$-values detect very well the difference between the two distributions when it is the case, giving a formal inference to the ranking $DD$-plot visualization, whereas, remarkably, the (parametric) nature of the underlying ranking models is not used at all by the procedure.
We also underline that, in a similar fashion, ranking depth-based \textit{goodness-of-fit} statistics could be computed, in order to evaluate how well a specific ranking model fits a ranking dataset.

\begin{figure}
	\begin{center}
	\begin{tabular}{cc}
		{\scriptsize $\gamma = 0.5$} & {\scriptsize $\gamma = 0.75$} \\
		\includegraphics[width=0.18\textwidth,trim=10 0 0mm 5mm,clip=true,page=1]{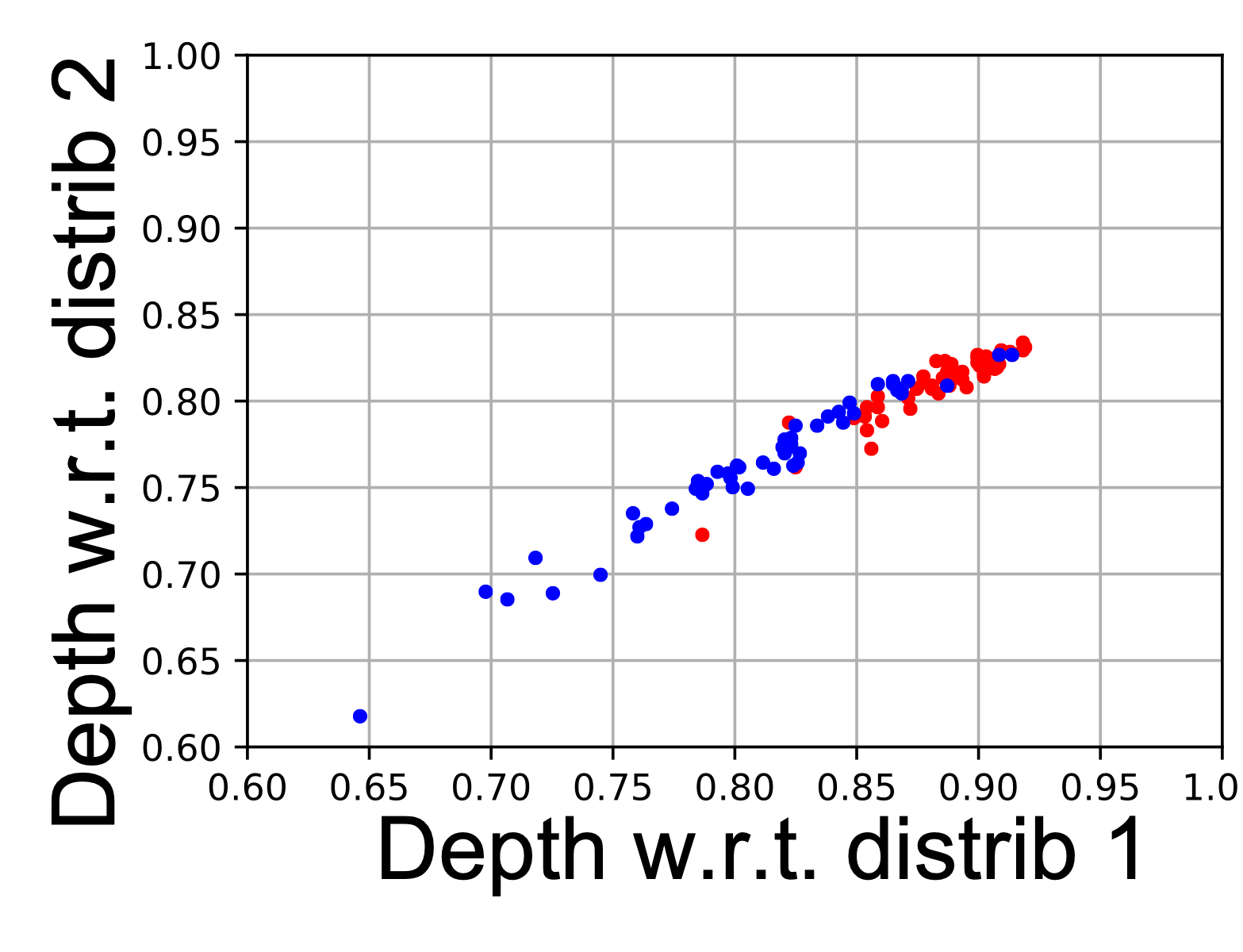} & \includegraphics[width=0.18\textwidth,trim=0 0 0mm 5mm,clip=true,page=1]{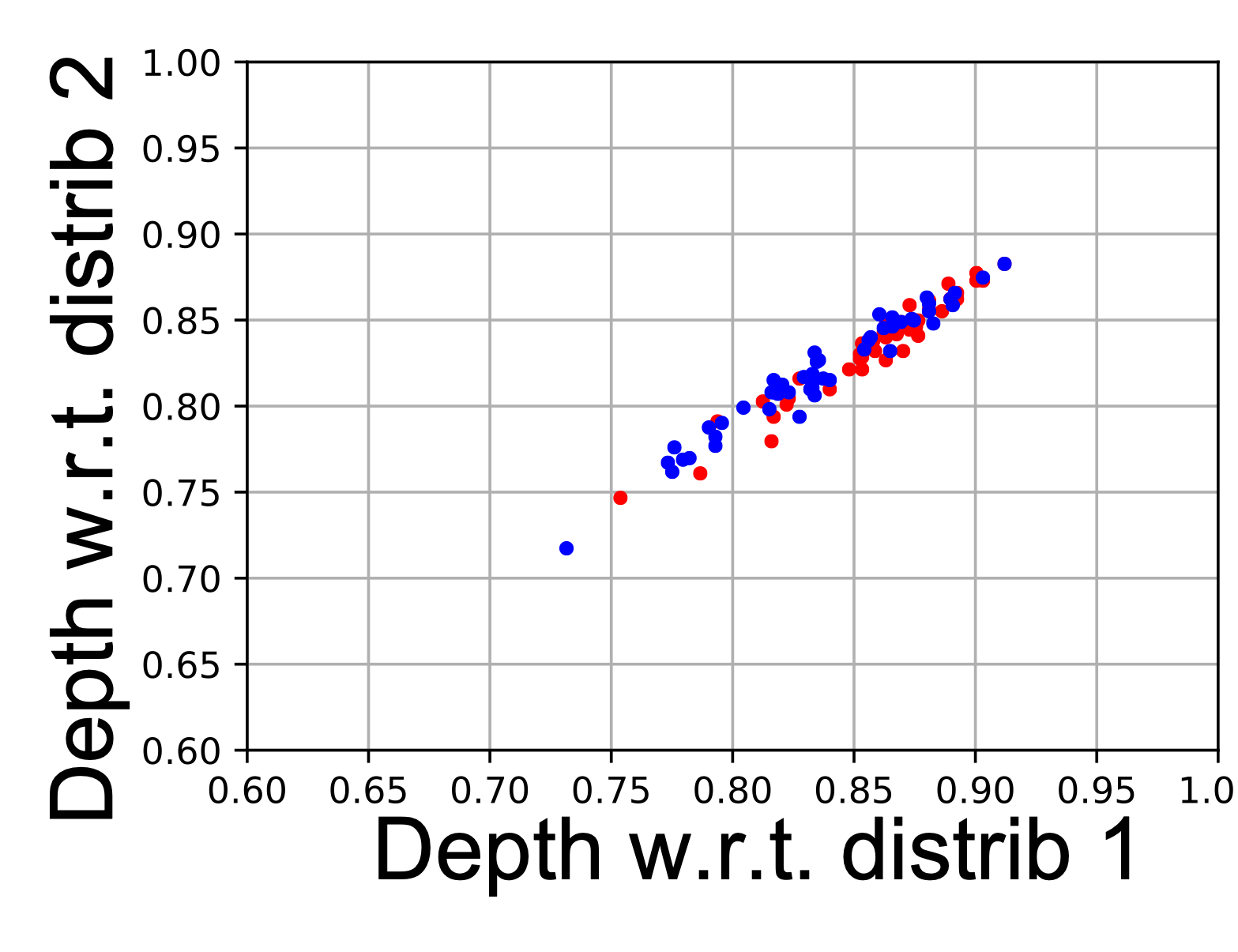} \\
		{\scriptsize $\gamma = 1$} & {\scriptsize $p$-values} \\
		\includegraphics[width=0.18\textwidth,trim=0 0 0mm 5mm,clip=true,page=1]{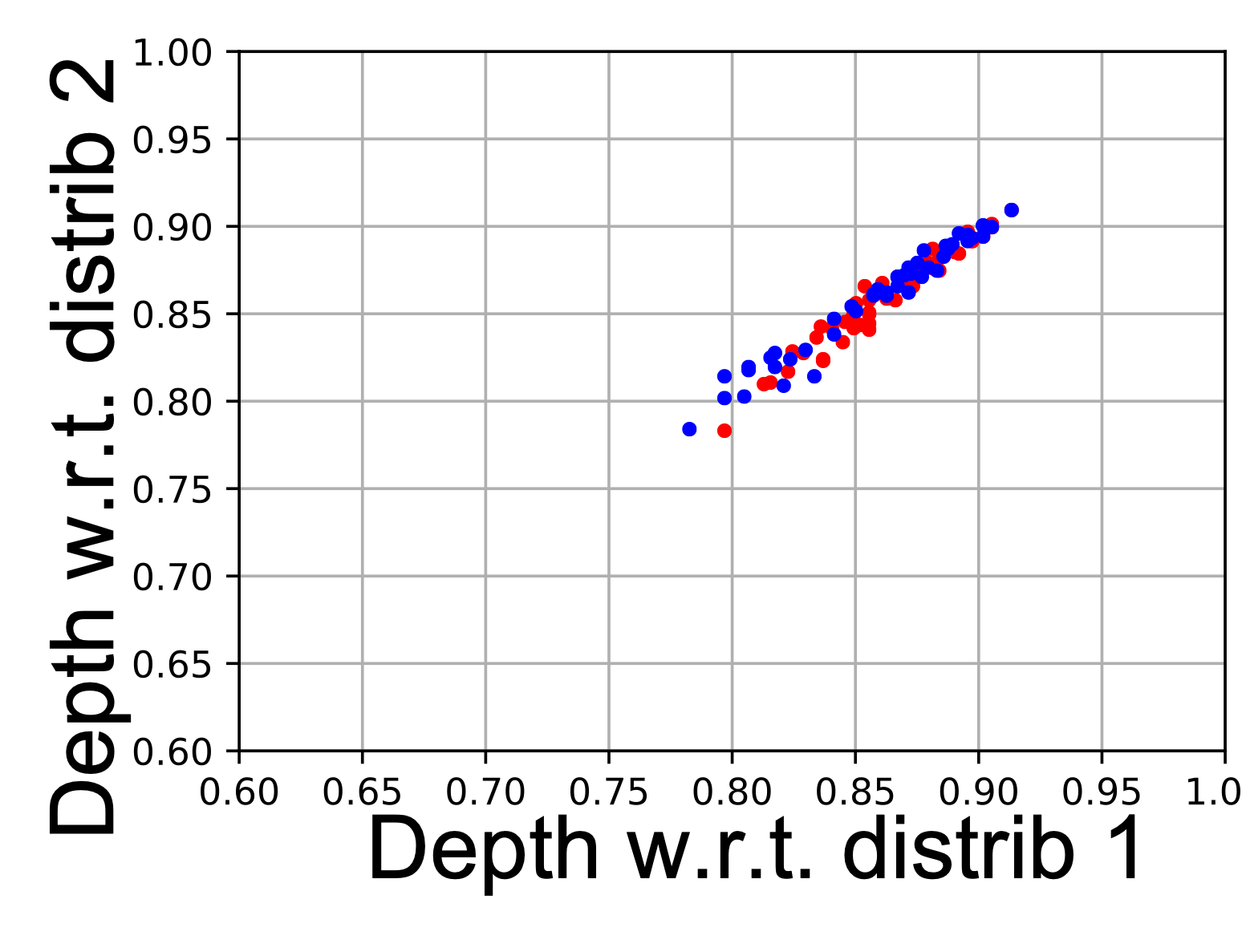} & 
		\includegraphics[width=0.27\textwidth,trim=0mm 0 0mm 0mm,clip=true,page=1]{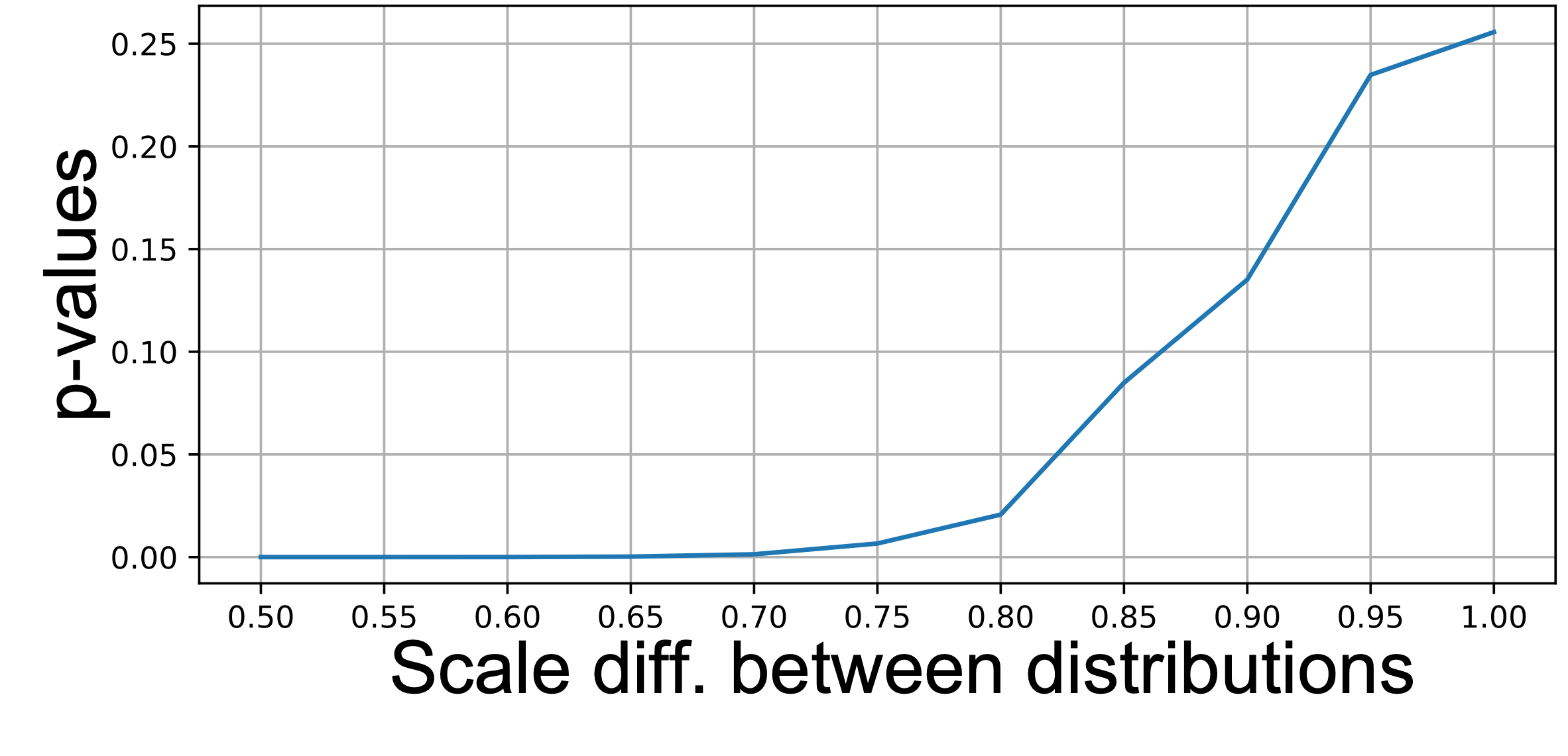}
	\end{tabular}
	\end{center}
	\caption{$DD$-plots of a pair of P-L distributions with gradually decreasing difference between them based on parameter $\gamma$ and the corresponding average $p$-values for the test of homogeneity.}
	\label{fig:tests}
\end{figure}

\paragraph{Student dataset.} We now explore our homogeneity testing machinery on a real dataset (available at \url{https://github.com/ekhiru/students-dataset}) composed of rankings from students (with a ground truth answer) before (red) and after (blue) taking the related course. The diagnostic $DD$-plot of the two cohorts together with $p$-values over $1000$ random repetitions and the asymptotic density under $H_0$ are indicated in Figure~\ref{fig:testrealdata}: they illustrate the improvement of the students' knowledge after the class.

\begin{figure}[h!]
	\begin{center}
	\begin{tabular}{cc}
		\includegraphics[width=0.155\textwidth,trim=0mm 0 0mm 0mm,clip=true]{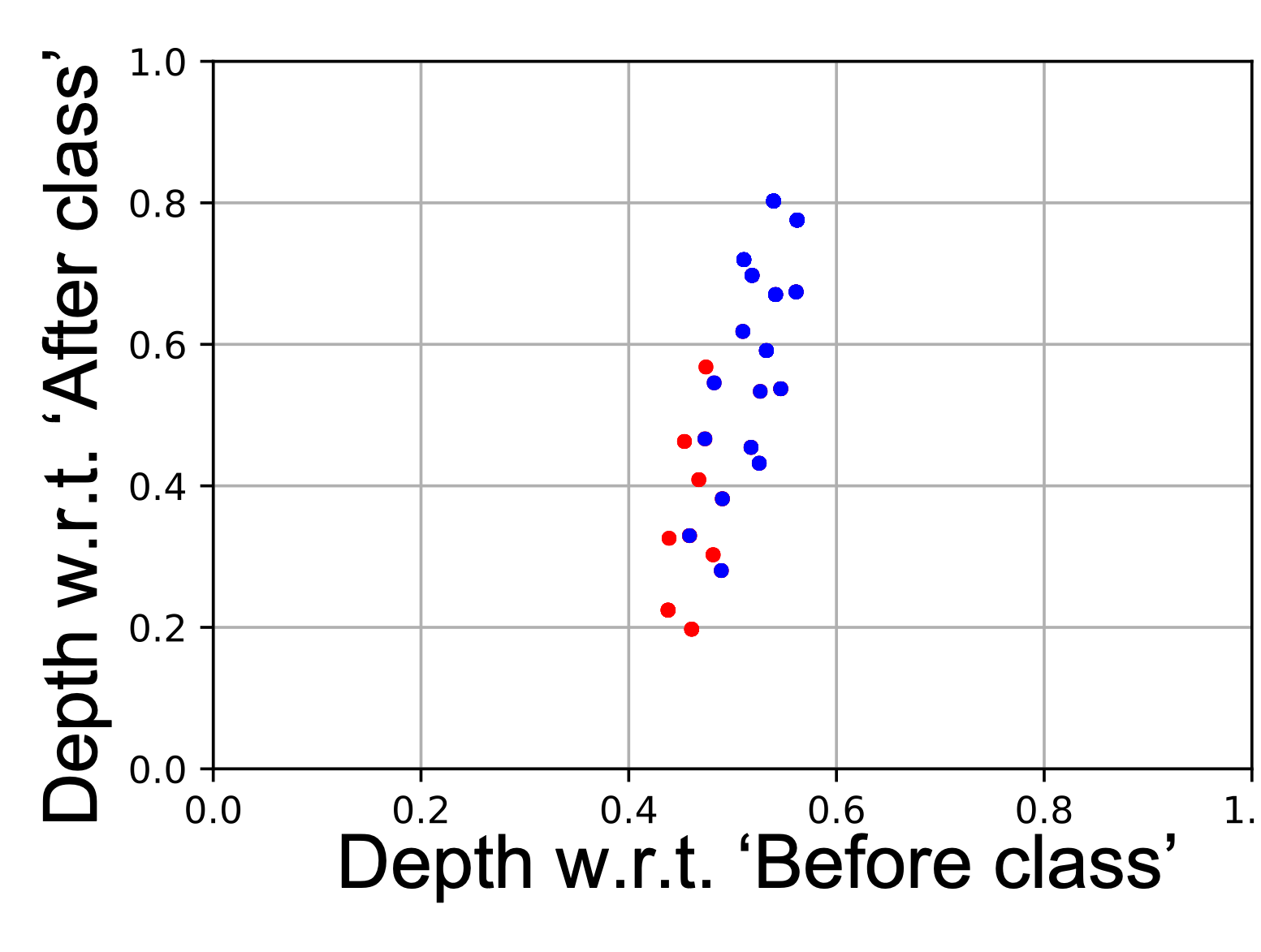} & \includegraphics[width=0.285\textwidth,trim=0mm 0 0mm 0mm,clip=true]{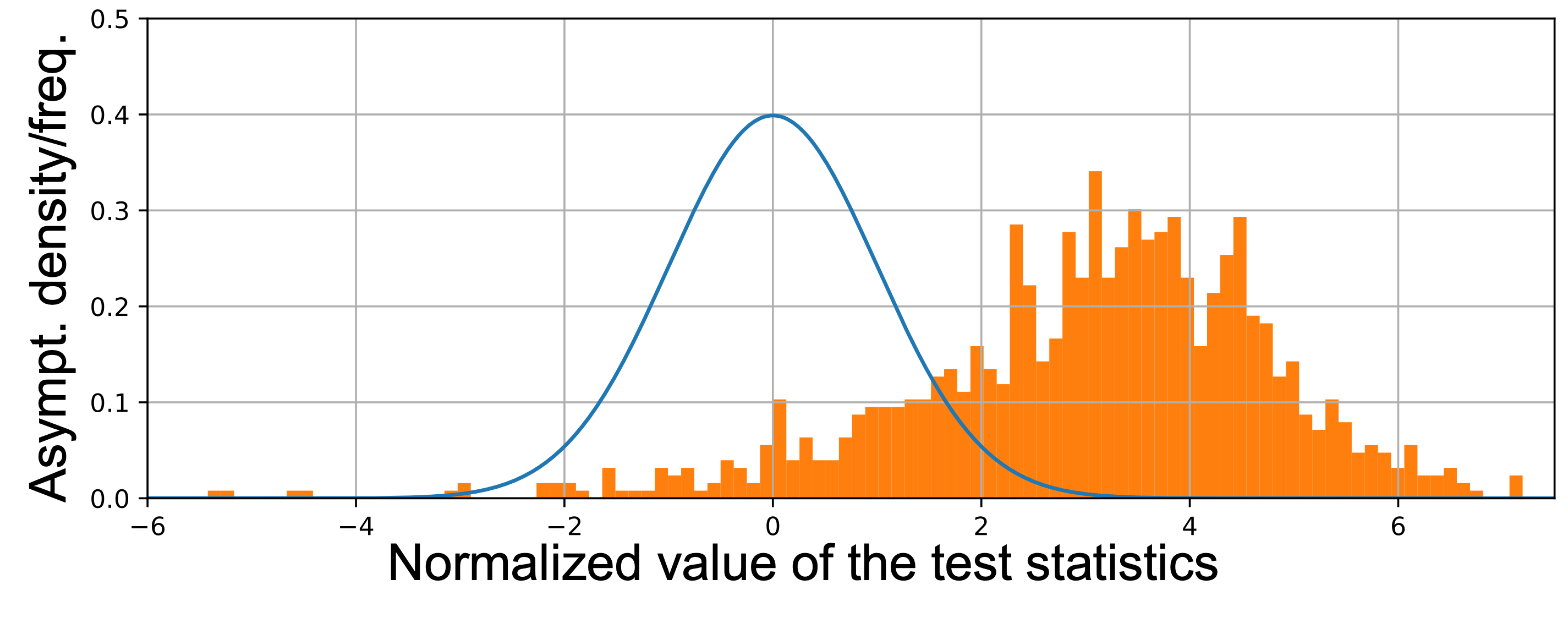}
	\end{tabular}
	\end{center}
	\caption{Left: $DD$-plot for 'before class' (red) and 'after class' (blue) students. Right: $p$-values of the homogeneity test.}
	\label{fig:testrealdata}
\end{figure}



\section*{Conclusion}

In this paper, we have extended the concept of statistical depth to ranking data, in order to apply the notions of quantiles, order statistics and ranks to the latter, overcoming hence the lack of natural order and vector space structure on $\mathfrak{S}_n$. We have listed the desirable properties a ranking depth should satisfy to emulate these notions appropriately and shown that the same metric approach as that, widely used, to deal with ranking aggregation, permits to build depth functions on $\mathfrak{S}_n$ that fulfill them in many situations. Theoretical results proving that ranking depths and related quantities can be accurately estimated by their empirical versions with guarantees have been established. We have also shown that the methodology promoted can be successfully applied to a wide variety of problems, ranging from fast and robust consensus ranking to the design of ranking data visualization techniques through the detection of outlying rankings. Both the theoretical and empirical results are very encouraging and paves the way to a more systematic use of the ranking depth concept for the statistical analysis of ranking data.

\nocite{*}
\bibliography{bibliography}
\bibliographystyle{plain}

\addtocontents{toc}{\protect\setcounter{tocdepth}{3}}

\onecolumn
\appendix

\begin{center}
    \noindent\rule{17cm}{3pt} \vspace{0.4cm}
    
    \huge \textbf{Supplementary Material}
    
    \noindent\rule{17cm}{1.2pt}
\end{center}

\tableofcontents

\section{Ranking Distributions - Popular Examples}
Below we recall some popular ranking models. They will be next used to illustrate some of the properties involved in the theoretical analysis carried out.

\begin{proposition}\label{prop:maxim_rankings}
The symmetry center property for rankings has two versions, a \textit{weak} and a \textit{strong} one, see~\cite{critchlow91}.
\begin{enumerate}
\item Strong unimodality: ranking model $P$ is said to be strongly unimodal iff there exists a modal ranking $\sigma^*$ such that for every pair $i,j$  such that $\sigma^*(i) < \sigma(j)$ and any permutations $\sigma$ such that $\sigma(i) = \sigma(j)-1$ then $P(\sigma) \geq P(\sigma\tau_{ij})$, where $\sigma\tau_{ij}(i) = \sigma(j)$, $\sigma\tau_{ij}(j) = \sigma(i)$ and $\sigma\tau_{ij}(k) = \sigma(k)$ for $k\neq i,j$. 
\item Complete consensus: ranking model $P$ is said to have complete consensus iff there exists a modal ranking $\sigma^*$ such that for every pair $i,j$  such that $\sigma^*(i) < \sigma(j)$ and any permutations $\sigma$ such that $\sigma(i) < \sigma(j)$ then $P(\sigma) \geq P(\sigma\tau_{ij})$, where $\sigma\tau_{ij}(i) = \sigma(j)$, $\sigma\tau_{ij}(j) = \sigma(i)$ and $\sigma\tau_{ij}(k) = \sigma(k)$ for $k\neq i,j$. Complete consensus implies strong unimodality. 
\end{enumerate}
\end{proposition}

\begin{example}{\sc (Mallows distribution)}\label{ex:Mallows1}
Taking $d=d_{\tau}$, the Mallows model introduced in \cite{Mallows57} is the unimodal distribution $P_{\theta}$ on $\mathfrak{S}_n$ parametrized by $\theta=(\sigma_0, \phi_0)\in \mathfrak{S}_n\times (0,1]$: $\forall \sigma \in \mathfrak{S}_n$,
 \begin{equation}\label{eq:Mallows}
 P_{\theta}(\sigma)=(1/Z_0)\exp( d_{\tau}(\sigma_0,\sigma) \log \phi_0),
\end{equation}
 where $Z_0=\sum_{\sigma\in \mathfrak{S}_n} \exp(d_{\tau}(\sigma_0,\sigma)\log \phi_0 )$ is a normalization constant. One may easily show that $Z_0$ is independent from $\sigma_0$ and that $Z_0=\prod_{i=1}^{n-1}\sum_{j=0}^i \phi_0^j$. When $\phi_0<1$, the permutation $\sigma_0$ of reference is the mode of distribution $ P_{\theta_0}$, as well as its unique median relative to $d_{\tau}$. Observe in addition that the smallest the parameter $\phi_0$, the spikiest the distribution $ P_{\theta_0}$. In contrast, $ P_{\theta_0}$ is the uniform distribution on $\mathfrak{S}_n$ when $\phi_0=1$. As explained in section \ref{sec:depth_for_rankings}, ranking depth functions relative to the Kendall $\tau$ distance can be expressed as a function of the pairwise probabilities $p_{i,j}=\mathbb{P}\{  \Sigma(i)<\Sigma(j)\}$, $1\leq i\neq j\leq n$. Notice also that $\vert\vert d_{\tau}\vert\vert_{\infty}=\binom{n}{2}$.
Consider again the Mallows model $P_{\theta}$ recalled in Example \ref{ex:Mallows1}. In this case, a closed-from expression of the $p_{i,j}$'s is available, see \textit{e.g.} Theorem 2 in \cite{DBLP:conf/icml/Busa-FeketeHS14}.  Setting $h(k,\phi_0)=k/(1-\phi_0^k)$ for $k\geq 1$, one can then show that the ranking depth function relative to $P_{\theta}$ and $d_{\tau}$ is: $\forall \sigma\in \mathfrak{S}_n$,
 $D_{P_{\theta}}(\sigma) = {n \choose 2} - \sum_{\sigma(i)>\sigma(j)} H(\sigma_0(j)-\sigma_0(i),\; \phi_0)$,
where $H(k,\phi_0)=h(k+1,\phi_0)-h(k,\phi_0)$ and $H(-k,\phi_0)=1-H(k,\phi_0)$  for $k\geq 1$. Mallows is adapted naturally to work with extensions of rankings, such as from pairwise preferences \cite{Lu2014}, and partial rankings \cite{Vitelli2018}

Mallows satisfies the complete consensus property, see Property~\ref{prop:maxim_rankings}, when $\theta<1$.

The most popular extensions in the literature  are Generalized Mallows models~\cite{gMallows}, \cite{Irurozki2019a} and Mallows Block models~\cite{busa2019optimal}. They define different dispersion parameters for different ranking positions to model distributions in which there is high certainty in the top-ranked items and uncertainty at the bottom. 
These models still satisfy the complete consensus property, see Property~\ref{prop:maxim_rankings}, when $\theta<1$. 

We also point out that model \eqref{eq:Mallows} can be extended in a straightforward manner, by considering alternative distances $d$, including those described in Section~\ref{subsec:consensus_ranking} and other right invariant distances such as Cayley and Ulam, all of which satisfy the complete consensus property, see Property~\ref{prop:maxim_rankings}, when $\theta<1$.

The maximality at center is broken in more general ranking distributions with the form of mixtures of Mallows models. Mixtures have been studied in practical and theoretical settings, see \textit{e.g.} \cite{LL02,liu2018efficiently,collas21}. 
\end{example}

\begin{example}{\sc (Plackett-Luce (PL) distribution)} \label{ex:PL} 
PL assumes that rankings are generated in a stage wise manner: the most preferred item is chosen first, then the second preferred one, \ldots There is  independence among stages, that is, the probability of an item being chosen at a particular stage is only proportional to the remaining items at this stage and independent of the order of the items that have already been chosen. Thus, PL is parametrized  by 
$\boldsymbol{v} \in \mathbb{R}^n$, 
where $v(i)$ is proportional to the probability of choosing item $i$ as the preferred item at any stage (among the remaining ones). The probability of each ranking is given as 
 \begin{equation}
P_v(\sigma) = \prod_{i=1}^n \frac{\sigma^{-1}(i)}{\sum_{j=i}^n \sigma^{-1}(j)}.
\end{equation} 

The median ranking is the permutation that orders the weights decreasingly. 
The pairwise probabilities of items $i$ and $j$ have a closed-form expression involving only the weights of bot items, $ { p_{i,j} = \frac{v_i}{v_i + v_j} } $. PL's stage wise ranking process implies that adaptation to top-$k$ and  rankings is natural~\cite{Liu2019}. 

The PL models satisfy the complete consensus property, see Property~\ref{prop:maxim_rankings} for every distribution other than the uniform. Clearly, the maximality at center does not hold for mixtures of PL in general. Note that there is a body of research on PL mixtures~\cite{Liu2019,zhao2019learning}. 
\end{example}

\begin{example}{\sc (Mallows-Bradley-Terry distribution)}\label{ex:MBT}
Mallows-Bradley-Terry is a ranking model induced by paired comparisons in which the pairwise probability of items $i$ and $j$ have the form $$p_{i,j} = \frac{v_i}{v_i+v_j},$$
where $v_i$ is the parameter associated to item $i$ for $\boldsymbol{v} \in \mathbb{R}^n$. The probability of ranking $\sigma$ is then $$p(\sigma) = Z(\boldsymbol{v})\prod_{i=1}^{n-1} (v_{\sigma^{-1}(i)})^{n-i},$$
where $Z$ is a normalization constant. See \cite{HWL06} for generalizations.
\end{example}

\begin{example}{\sc (Pairwise distributions)}\label{ex:MBT}
All the above models can be written as a $n\times n$ matrix of  pairwise probabilities $p_{i,j}$ (describing the probability of item $i$ being preferred to item $j$) with restricted forms of its entries. Each of the models  imposes different restrictions in the entries of the pairwise probabilities $p_{i,j}$ but one could consider arbitrary values. We next lines characterize the properties of models with arbitrary entries $p_{i,j}$.
\begin{itemize}
    \item $P$ is strongly unimodal if and only if its entries are  weakly stochastically transitive for some reordering of the rows an columns, as defined in Proposition~\ref{def:stoch_trans}. 
    \item $P$ has complete consensus if and only if its entries are  strongly stochastically transitive for some reordering of the rows an columns. A probability distribution $P$ on $\mathfrak{S}_n$ is said to be strongly stochastically transitive iff, for all $(i,j,k)\in \n^3$, we have: $p_{i,j}\geq 1/2 \text{ and } p_{j,k}\geq 1/2 \; \Rightarrow\; p_{i,k}\geq \max\{p_{i,j},p_{j,k}\}$ and $p_{i,j}\neq 1/2$ for all $i<j$.
    
\end{itemize}

\end{example}

\section{Technical Proofs}
\label{suppl:sec_proofs}

\subsection{Conditions for satisfying the desirable properties}
\label{suppl:proof_desirable_prop}

\subsubsection{Proof of Proposition \ref{prop:invariance} (invariance)}
\label{suppl:proof_invariance}

We elaborate now on the invariance property~\ref{property:invariance}. Recall that a distance is right invariant iff for every triplet of permutations $(\sigma,\pi,\nu)\in \mathfrak{S}_n d(\sigma,\nu) = d(\sigma\pi,\nu\pi)$. Finally, the inverse of permutation $\sigma$ is denoted by $\sigma^{-1}$.

Let us first recall the invariance property for distributions and for depths~\ref{property:invariance} and our proposition~\ref{prop:invariance}: \propertyinvariance*  \propinvariance*

\begin{proof}
\begin{equation}
\begin{split}
D_{\pi P}(\sigma \pi)=&\mathbb{E}_{\pi P}[\vert\vert d\vert\vert_{\infty}-d(\sigma \pi,\Sigma)] 
=  \vert\vert d \vert\vert_{\infty} - \sum_{\nu \in \mathfrak{S}_n} (\pi P)(\nu) d(\sigma \pi, \nu) \\
= &  \vert\vert d \vert\vert_{\infty} - \sum_{\nu \in \mathfrak{S}_n} P(\nu \pi^{-1}) d(\sigma \pi, \nu)
= \vert\vert d \vert\vert_{\infty} - \sum_{\nu' \in \mathfrak{S}_n} P(\nu' \pi \pi^{-1}) d(\sigma \pi, \nu' \pi) \\
= & \vert\vert d \vert\vert_{\infty} - \sum_{\nu' \in \mathfrak{S}_n} P(\nu') d(\sigma, \nu')
=  D_{P}(\sigma). \\
\end{split}
\end{equation}
\end{proof}

\subsubsection{Proof of Proposition \ref{prop:maximality} (maximality at the center)}
\label{suppl:proof_maximality}

First, we study the relation between the depth and the probability of permutations which will be key for the results on the following sections. 

\begin{proposition}
\label{prop:rays}
Let $P$ be a SST distribution whose Kemeny's median is $\sigma^*$, and $\sigma^*(a) < \sigma^*(b)$. Let $\sigma$ be a ranking such that $\sigma(a) +1 = \sigma(b)$ and let $t_{ab}$ be a transposition, i.e., $t_{ab}(a)=b$, $t_{ab}(b)=a$ and $t_{ab}(k)=k$ for all $k\neq a,b$. 
Then, $$D(\sigma) > D(\sigma t).$$ 
\end{proposition}

\begin{proof}
First, note that the composition $\sigma t_{ab}$ exchanges the ranks of items $a$ and $b$, so $d(\sigma,\sigma^*) = d(\sigma t,\sigma^*) -1$. We can rewrite $D(\sigma)$ in the following way,

\begin{equation}
\begin{split}
D(\sigma) = & 
\binom{n}{2}-\sum_{i<j} p_{i,j}\mathbb{I}\{\sigma(i)>\sigma(j) \} -\sum_{i<j} p_{j,i}\mathbb{I}\{\sigma(i)<\sigma(j) \} \\ =
& \binom{n}{2}-\sum_{i<j \land i,j \neq a,b} p_{i,j}\mathbb{I}\{\sigma(i)>\sigma(j) \} -\sum_{i<j \land i,j \neq a,b} p_{j,i}\mathbb{I}\{\sigma(i)<\sigma(j) \} - \\
& -p_{a,b}\mathbb{I}\{\sigma(a)>\sigma(b) \} -  p_{ba}\mathbb{I}\{\sigma(a)<\sigma(b) \}   \\
= &  \binom{n}{2}-\sum_{i<j \land i,j \neq a,b} p_{i,j}\mathbb{I}\{\sigma(i)>\sigma(j) \} -\sum_{i<j \land i,j \neq a,b} p_{j,i}\mathbb{I}\{\sigma(i)<\sigma(j) \} -  p_{b,a}.
\end{split}
\end{equation}

Where  the first equality is the given by proposition~\ref{prop:depth_kendall}. In  the second we split the sum for positions $i=a$ and $j=b$ in the latter term and the rest of the pairs in the previous terms. In the third one, we recall that by assumption $\sigma(a) = \sigma(b)-1$ and therefore $\mathbb{I}\{\sigma(a)<\sigma(b) \} = 1$ and $\mathbb{I}\{\sigma(a)>\sigma(b) \} = 0$. We rewrite in a similar way $D(\sigma t_{ab})$. For this part, recall that $\sigma t_{ab}(a)=\sigma(b)$, $\sigma t_{ab}(b)=\sigma(a)$ and $\sigma t_{ab}(k)=k$ for all $k\neq a,b$. 

\begin{equation}
\begin{split}
D(\sigma t_{ab}) = 
\binom{n}{2} & -\sum_{i<j} p_{i,j}\mathbb{I}\{\sigma t_{ab}(i)>\sigma t_{ab}(j) \} -\sum_{i<j} p_{j,i}\mathbb{I}\{\sigma t_{ab}(i)<\sigma t_{ab}(j) \}  \\
= \binom{n}{2} & -\sum_{i<j \land i,j \neq a,b} p_{i,j}\mathbb{I}\{\sigma t_{ab}(i)>\sigma t_{ab}(j) \} -\sum_{i<j \land i,j \neq a,b} p_{j,i}\mathbb{I}\{\sigma t_{ab}(i)<\sigma t_{ab}(j) \}  \\
& -p_{a,b}\mathbb{I}\{\sigma t_{ab}(a)>\sigma t_{ab}(b) \} - p_{b,a}\mathbb{I}\{\sigma t_{ab}(a)<\sigma t_{ab}(b) \}   \\
= \binom{n}{2} & -\sum_{i<j \land i,j \neq a,b} p_{i,j}\mathbb{I}\{\sigma(i)>\sigma(j) \} -\sum_{i<j \land i,j \neq a,b} p_{j,i}\mathbb{I}\{\sigma(i)<\sigma(j) \} - \\
& -p_{a,b}\mathbb{I}\{\sigma(a)<\sigma(b) \} -p_{ba}\mathbb{I}\{\sigma(a)>\sigma(b) \}   \\
= \binom{n}{2} & -\sum_{i<j \land i,j \neq a,b} p_{i,j}\mathbb{I}\{\sigma(i)>\sigma(j) \} -\sum_{i<j \land i,j \neq a,b} p_{j,i}\mathbb{I}\{\sigma(i)<\sigma(j) \} -p_{a,b}
\end{split}
\end{equation}

Therefore, for any two rankings $\sigma$ and $\sigma t_{ab}$ such that $D(\sigma) > D(\sigma t_{ab})$, the following holds,
\begin{equation}\label{eq:depth_prob_pairwise}
\begin{split}
 D(\sigma)  & > D(\sigma t_{ab})  \\
\Leftrightarrow & \binom{n}{2}-\sum_{i<j \land i,j \neq a,b} p_{i,j}\mathbb{I}\{\sigma(i)>\sigma(j) \} -\sum_{i<j \land i,j \neq a,b} p_{j,i}\mathbb{I}\{\sigma(i)<\sigma(j) \} -p_{b,a}  \\
& >  \binom{n}{2}-\sum_{i<j \land i,j \neq a,b} p_{i,j}\mathbb{I}\{\sigma(i)>\sigma(j) \} -\sum_{i<j \land i,j \neq a,b} p_{j,i}\mathbb{I}\{\sigma(i)<\sigma(j) \} -p_{a,b}\\
   \Leftrightarrow 
& p_{b,a} < p_{a,b}.
\end{split}
\end{equation}
For any SST model $P$ with whose median is $\sigma^*$, and where $\sigma^*(a) < \sigma^*(b)$, it holds (by definition)  that $p_{b,a} < p_{a,b}$, which concludes the proof.
\end{proof}

Let us first recall Property~\ref{property:maximality} (maximality) and Proposition~\ref{prop:maximality}. \propertymaximality* \propmaximality*

We now discuss what  is precisely meant by \textit{center} in the ranking context. We derive two main definitions for a center:
\begin{itemize}
    \item Following \cite{Tukey75, ZuoSerfling00}, we emulate the notion of \textit{half-space} symmetry (which is a very generic notion of symmetry) and define a notion of $H$-center, from which our proposition in the main paper stems from. Appart from our maximality proposition, we further provide results for distributions $P$ having a $H$-center.
    \item We define a simpler notion of center based on a natural metric approach. We also provide maximality results based on this different notion of center, called a $M$-center.
\end{itemize}

\paragraph{H-center and maximality at center.}~\\

The following results (1) define a symmetry center inspired in the classical formulation of half-space symmetry~\cite{Tukey75, ZuoSerfling00} and (2) shows that the \ken and Spearman's footrule distances satisfy the maximality at center for the defined center. 

\begin{proposition}\label{def:hcenter}
Let us call "hyperplane" the sets $ H_{i,j} =\{\sigma :  \sigma(i) < \sigma(j) \}$, 
we define the H-center $\sigma$ as $\sigma = \cap H_{i,j}$ for all $\{ (i,j) : \sigma_0(i) < \sigma_0(j) \}$.
For any $P$ such that $p_{i,j}>p_{j,i}$ for all $\{ (i,j) : \sigma_0(i) < \sigma_0(j) \}$ the H-center is $\sigma_0$. 
\end{proposition}
\begin{proof}
Firstly, we show that  $P(\Sigma \in  H_{i,j}) > P(\Sigma \in  H_{j,i})$. This can be done by construction: For any ranking $\sigma \in H_{i,j}$ (for which $\sigma(i) = \sigma(j)$) we can construct $\sigma'\in H_{j,i}$ that swaps positions $i$ and $j$. This construction defines a bijection between the rankings in both sets. The following relation holds:
$p(\sigma') = p(\sigma) p_{j,i}/p_{i,j} < p(\sigma)$. Therefore, $P(\Sigma \in\cap H_{i,j}) > P(\Sigma \in \cap H_{j,i})$. 

Secondly, it is clear that there is one and only one permutation in $\cap H_{i,j}$ and this is $\sigma_0$. We remark that its possible an H-center is defined (for this choice of $P$) by a smaller number of subsets, i.e., those $H_{i,j}$ for which $\sigma(i) = \sigma(j) -1$.
\end{proof}

\begin{proposition}
Let $P$ be distribution for which there is an H-center  
both \ken and Spearman's footrule based depths satisfy the maximality at center property for the H-center in Definition~\ref{def:hcenter}. 
\end{proposition}
\begin{proof}
As shown in Proposition~\ref{def:hcenter}, the H-center is $\sigma_0$
It remains to recall that Equation\eqref{eq:depth_prob_pairwise} in Proposition~\ref{prop:rays} states that for SST models and the \ken distance
$ D(\sigma) > D(\sigma t_{i,j}) \Leftrightarrow  p_{j,i} < p_{i,j}$.

For the Spearman's distance, let us show that $D_P(\sigma_0) \geq D_P(\sigma_1) \Leftrightarrow \mathbb{E}_P(d(\Sigma, \sigma_0)) \leq \mathbb{E}_P(d(\Sigma, \sigma_1))$, and the proof of our proposition will follow from direct application of this result.

Let $\sigma$ be any permutation.
\begin{equation*}
\begin{split}
    d(\sigma, \sigma_1) & = \sum_{k=1}^N | \sigma(k) - \sigma_1(k) | \\
    & = \sum_{k\neq i,j} | \sigma(k) - \sigma_0(k) | + | \sigma(i) - \sigma_0(i) - 1 | + | \sigma(j) - \sigma_0(j) + 1 | \\
    & = \begin{cases}
      d(\sigma, \sigma_0) \text{ if } \textcolor{NavyBlue}{\sigma(i)<\sigma(j) \leq \sigma_0(i)<\sigma_0(j) \text{ or } \sigma_0(i)<\sigma_0(j) \leq \sigma(i)<\sigma(j)} \\
      \quad \quad \text{ or } \textcolor{BurntOrange}{\sigma(j)<\sigma(i) \leq \sigma_0(i)<\sigma_0(j) \text{ or } \sigma_0(i)<\sigma_0(j) \leq \sigma(j)<\sigma(i)}  \\
      d(\sigma, \sigma_0) + 2 \text{ if } \textcolor{NavyBlue}{\sigma(i)<\sigma_0(i)<\sigma_0(j)<\sigma(j)}\\
      d(\sigma, \sigma_0) - 2 \text{ if } \textcolor{BurntOrange}{\sigma(j) \leq \sigma_0(i)<\sigma_0(j) \leq \sigma(i)}
    \end{cases}  
\end{split}
\end{equation*}
Notice the use of color: in blue are cases where $i$ and $j$ are ranked by $\sigma$ the same way as does $\sigma_0$, and in orange are the opposite cases. Then
\begin{equation*}
    \begin{split}
        & \mathbb{E}_P(d(\sigma, \sigma_0)) \leq \mathbb{E}_P(d(\sigma, \sigma_1)) \\
        \Leftrightarrow & \sum_{\sigma} \left[ \mathbb{I}(\textcolor{NavyBlue}{\text{blue cases}}) - \mathbb{I}(\textcolor{BurntOrange}{\text{orange cases}}) \right] \mathbb{P}(\Sigma = \sigma) \geq 0 \\
        \Leftrightarrow & p_{i,j} - (1 - p_{i,j}) \geq 0 \\
        \Leftrightarrow & p_{i,j} \geq 1/2
    \end{split}
\end{equation*}
This concludes the proof.
\end{proof}

\paragraph{$M$-center definition.}~\\

Let us focus on a more natural, metric-based center definition.

\begin{definition}
$\sigma_0$ is $M$-center for distance $d$ and distribution $P$ if: $\forall (\sigma_1, \sigma_2, \sigma_3)$ such that $d(\sigma_0, \sigma_1) = d(\sigma_0, \sigma_2) < d(\sigma_0, \sigma_3)$, we have:
$\mathbb{P}(\Sigma = \sigma_1) = \mathbb{P}(\Sigma = \sigma_2) \geq \mathbb{P}(\Sigma = \sigma_3)$.
\end{definition}

We have the following proposition:

\begin{proposition}
    If $d$ is a symmetric distance, and if distribution $P$ has a $M$-center for $d$, the maximality property is satisfied for distance $d$.
\end{proposition}

Most distances (as the one studied in this paper) are symmetric. In addition, the proposition applies to Mallows models as they do exhibit a $S$-center.

\begin{proof}
Let $\sigma_0$ be a $M$-center for $P$ and distance $d$, with $(i,j)$ such that $\sigma_0(i) < \sigma_0(j) = \sigma_0(i)+1$. Let $\sigma_1$ be the same ranking as $\sigma_0$ except it swaps the ranks of $i$ and $j$.

We show that $D_P(\sigma_0) > D_P(\sigma_0)$ i.e. $\mathbb{E}_P(d(\Sigma, \sigma_0)) < \mathbb{E}_P(d(\Sigma, \sigma_1))$ i.e. $\sum_{\sigma} \mathbb{P}(\Sigma = \sigma) \left[ d(\sigma_1, \sigma) - d(\sigma_0, \sigma) \right] > 0$.

Let $\sigma$ be any ranking such that $d(\sigma_0, \sigma) = d$. We have:
\begin{itemize}
    \item (1) $d(\sigma_0, \sigma) < d(\sigma_1, \sigma) = d + c_{i,j}$ iff $(i,j)$ is ranked the same way in $\sigma_0$ and $\sigma$
    \item (2)  $d(\sigma_0, \sigma) > d(\sigma_1, \sigma) = d - c_{i,j}$ else,
\end{itemize} where $c_{i,j} > 0$ is a constant depending only on $(i,j)$. For example, if $d$ is Kendall's tau, $c_{i,j} = 1$, if $d$ is Spearman's footrule, $c_{i,j} = 2$, if $d$ is Spearman's rho, $c_{i,j} = 2|\sigma(j) - \sigma(i)|$.

In addition, let us write $\#d$ the number of rankings at distance $d$ from $\sigma_0$, which we can divide into the two groups (1) and (2). Let us then write $\#d(1)$ (resp. $\#d(2)$) the number of rankings $\sigma$ at distance $d$ from $\sigma_0$ that rank $i$ and $j$ the same way (resp. differently) as $\sigma_0$. We suppose the following: if $d \leq || d ||_{\infty} / 2$, then $\#d(1) \geq \#d(2)$ (and if $d > || d ||_{\infty} / 2$, then $\#d(1) \leq \#d(2)$), and more precisely, $| \#d(1) - \#d(2) | = k(d) = k( ||d||_{\infty} - d ) \; \forall \; d \leq ||d||_{\infty}/2$, meaning that this cardinality difference depends only on the distance to half of the maximal distance.

Let us also write $P_d = \mathbb{P}(\Sigma = \sigma)$ for any $\sigma$ at distance $d$ from $\sigma_0$.

\begin{equation*}
    \begin{split}
        \sum_{\sigma} \mathbb{P}(\Sigma = \sigma) \left[ d(\sigma_1, \sigma) - d(\sigma_0, \sigma) \right] & = \sum_{d=0}^{||d||_{\infty}} P_d \times \#d \times |c_{i,j}| \\
        & = \sum_{d=0}^{||d||_{\infty}} P_d \times (\#d(1) - \#d(2)) \times c_{i,j} \\
        & = \sum_{d=0}^{||d||_{\infty}/2} P_d \times k(d) \times c_{i,j} - \sum_{d'=||d||_{\infty}/2 + 1}^{||d||_{\infty}} \underbrace{P_{d'}}_{ < P_{||d||_{\infty} - d'} } \times k(d') \times c_{i,j}\\
        & > \sum_{d=0}^{||d||_{\infty}/2} P_d \times c_{i,j} \times ( k(d) - k(||d||_{\infty} - d) ) \\
        & > 0
    \end{split}
\end{equation*}

\end{proof}

\subsubsection{Proofs of Propositions \ref{prop:local_monotonicity} and \ref{prop:global_monotonicity} (monotonicity)}
\label{suppl:proof_monotonicity}

The Monotonicity properties~\ref{property:local_monotonicity} and~\ref{property:local_monotonicity} do not hold in general. As an illustration, fig.~\ref{fig:rank_mat} shows the distance to the median as a function of depth for every rankings in sample generated by Mallows or Placket-Luce distributions.

\begin{figure*}[htbp]
\begin{center}
    \subfigure[]{\includegraphics[width=0.24\textwidth]{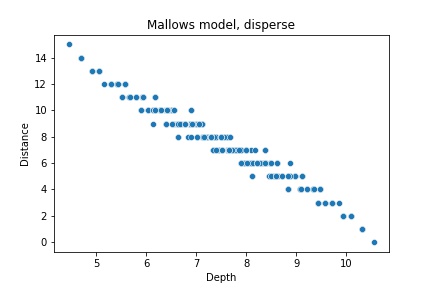}}
    \subfigure[]{\includegraphics[width=0.24\textwidth]{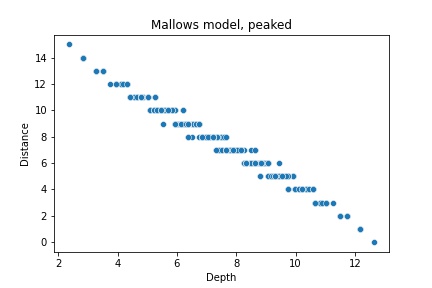}}
    \subfigure[]{\includegraphics[width=0.24\textwidth]{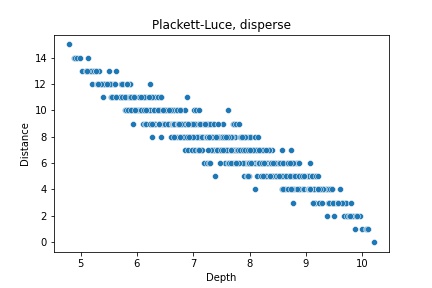}}
    \subfigure[]{\includegraphics[width=0.24\textwidth]{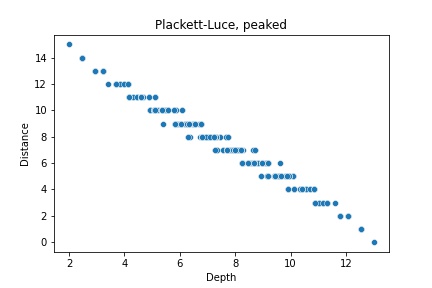}}
\caption{Each permutation in $\mathfrak{S}_6$ is a point displaying its depth (X-axis) and distance to the median (Y-axis). The ranking models are (a) Mallows model with $\phi=\mathrm{e}^{-0.375}$, (b) Mallows model with $\phi=\mathrm{e}^{-0.625}$, (c) Plackett-Luce with $\boldsymbol{w}_2=(\mathrm{e}^{n},...,\mathrm{e}^{1})$, (b) Plackett-Luce with $\boldsymbol{w}_2=({n},...,{1})$.}
\label{fig:rank_mat}
\end{center}
\end{figure*}

However, we derive two conditions making these monotonicity properties to hold, restricting ourselves to the case where the distance $d$ used is Kendall $\tau$, and also to distributions $P$ that are strictly stochastically transitive (SST) to ensure uniqueness of the \textit{central} ranking $\sigma^{*}$ (see \cite{CKS17}).

First, let us recall the local monotonicity property and our local monotonicity proposition: \propertylocalmonotonicity* \proplocalmonotonicity*

The first part of this proposition follows immediately form Proposition~\ref{prop:rays}: as we move further from the median (as measured by the Kendall $\tau$ distance) swapping adjacent ranks, the depth is strictly decreasing.

Now, we derive a second, stronger local monotonicity property. The following propostion explicit the conditions under which it is satisfied.

\begin{proposition}
For a generic SST distribution $P$, if two rankings $\sigma$ and $\sigma'$ with $d = d(\sigma^{*}, \sigma)$ and $d' = d(\sigma^{*}, \sigma')$ satisfies the following:
$$ 2 \left[ \sum_{(i,j) | p_{i,j}<1/2, \; \sigma \text{ correct}, \; \sigma' \text{ incorrect}} p_{i,j} - \sum_{(i,j) | p_{i,j}<1/2, \;\sigma \text{ incorrect}, \; \sigma' \text{ correct}} p_{i,j} \right] - (d'-d) \leq 0 $$

Then the following property holds:
\begin{equation}
    d=d(\sigma^{*}, \sigma) < d'=d(\sigma^{*}, \sigma^{'}) \Longrightarrow D_P(\sigma) \geq D_P(\sigma^{'}),
\end{equation}
where $d$ is Kendall $\tau$, $\sigma^{*}$ is Kemeny's median and "$\sigma$ correct on $(i,j)$" means that $\sigma$ and $\sigma^{*}$ order the pair $(i,j)$ the same way.
\label{prop:monotonicity_local_2}
\end{proposition}

The proof of this proposition can be directly derived from the proof of Proposition~\ref{prop:global_monotonicity} (eq. \ref{eq:monotonicity_proof0})

Second, we recall the global monotonicity property and our proposition: \propertyglobalmonotonicity* \propglobalmonotonicity*

\begin{proof}
$P$ is SST so $\forall \; (i,j,l), p_{i,j} > 1/2$ and $p_{j,l} > 1/2 \Longrightarrow p_{i,l} > 1/2$.
WLOG, let us suppose that $\forall i < j, p_{i,j} < 1/2$. As $\sigma^{*}$ is the unique Kemeny's median, we have $\sigma^{*}(n) < \sigma^{*}(n-1) < ... < \sigma^{*}(1)$ (i.e $n \succ n-1 \succ ... \succ 1$).

Let $(\sigma, \sigma^{'})$ be two rankings such that $d = d(\sigma^{*}, \sigma) < d(\sigma^{*}, \sigma^{'}) = d'$. 
Let us write $k:=\#  \{ (i,j) | \mathbb{I} \left( (\sigma^*(i) - \sigma^*(j)) (\sigma(i) - \sigma(j)) \right) > 0 \times \mathbb{I} \left( (\sigma^*(i) - \sigma^*(j)) (\sigma'(i) - \sigma'(j)) \right) < 0 \}$, which means that there are $k$ pairs $(i,j)$ on which $\sigma$ agrees with $\sigma^*$ (i.e. $\sigma$ is "correct" on $(i,j)$) but $\sigma'$ disagrees with $\sigma^*$ (i.e. $\sigma'$ is "incorrect" on $(i,j)$). We define $k'$ similarly by interchanging the roles of $\sigma$ and $\sigma'$.

Our goal is then to find a condition on the distribution of rankings $P$ such that: $$\max_{\sigma, \sigma'} L_P(\sigma) - L_P(\sigma') < 0, \quad \text{ with } k>k'$$

First, let us study the range of possible values for $k$. Let us divide the $n(n-1)/2$ pairs $i<j$ following:
\begin{align*}
    &1) \quad \sigma \text{ agrees with } \sigma^{*} \text{ and } \sigma' \text{ disagrees with } \sigma^{*} & \to k \text{ pairs} \\
    &2) \quad \sigma \text{ agrees with } \sigma^{*} \text{ and } \sigma' \text{ agrees with } \sigma^{*} & \to a \text{ pairs} \\
    &3) \quad \sigma \text{ disagrees with } \sigma^{*} \text{ and } \sigma' \text{ agrees with } \sigma^{*} & \to k' \text{ pairs} \\
    &4) \quad \sigma \text{ disagrees with } \sigma^{*} \text{ and } \sigma' \text{ disagrees with } \sigma^{*} & \to b \text{ pairs}
\end{align*}

We then have
\begin{equation*}
  \left\{
    \begin{aligned}
      & k' + b = d \\
      & k + b = d' \\
      & k + a + k' + b = n(n-1)/2
    \end{aligned}
  \right. \text{ so } \\
  \left\{
    \begin{aligned}
      & k' = k + d - d' \\
      & b = d' - k \\
      & a = n(n-1)/2 - k - d
    \end{aligned}
  \right.
\end{equation*}

Finally, since we have $0 \leq k, k', a, b \leq n(n-1)/2$, we end up having the following relevant conditions on $k$:
$$ d'-d \leq k \leq d' $$
Now, let us write $p^{(m)}$ the $m$-th hightest element of the vector $(p_{i,j})_{i<j}$ of size $n(n-1)/2$, so that $1/2 > p^{(1)} > p^{(2)} > ... > p^{(n(n-1)/2)}$.
Then, we have

\begin{align}
    \max_{\sigma, \sigma'} L_P(\sigma) - L_P(\sigma^{'}) &= \max_{\sigma, \sigma'} \sum_{i<j} p_{i,j}\left[ \mathbb{I} \{\sigma(i)-\sigma(j)>0 \} - \mathbb{I} \{\sigma^{'}(i)-\sigma^{'}(j)>0 \} \right] + \nonumber \\
    & \quad \quad \quad \quad (1-p_{i,j})\left[ \mathbb{I} \{\sigma(i)-\sigma(j)<0 \} - \mathbb{I} \{\sigma^{'}(i)-\sigma^{'}(j)<0 \} \right] \nonumber \\
    & = \max_{\sigma, \sigma'} \sum_{i<j} (2p_{i,j} - 1) \left[\mathbb{I} \{\sigma(i)-\sigma(j)>0 \} + \mathbb{I} \{\sigma^{'}(i)-\sigma^{'}(j)<0 \} - 1 \right] \nonumber \\
    &= \max_{\sigma, \sigma'} \sum_{i<j; \sigma \text{corr.}, \sigma^{'} \text{ incorr.}} (2p_{i,j}-1) - \sum_{i<j; \sigma \text{ incorr.}, \sigma^{'} \text{ corr.}} (2p_{i,j}-1) \nonumber \\
    &\leq \max_{\sigma, \sigma'} 2 \left[ \sum_{i<j; \sigma \text{ corr.}, \sigma^{'} \text{ incorr.}} p_{i,j} -  \sum_{i<j; \sigma \text{ incorr.}, \sigma^{'} \text{ corr.}} p_{i,j} \right] - (k - k') \nonumber \\
    & \text{ with } k'=k-(d'-d) \label{eq:monotonicity_proof0} \\
    & \leq 2 \left[\underbrace{p^{(1)} + ... + p^{(k)}}_{k \text{ elements}} - \underbrace{p^{(n(n-1)/2 - k' - 1)} - ... -p^{(n(n-1)/2)}}_{k'=k-(d'-d) \text{ elements}} \right] - (d'-d) \nonumber \\
    & \leq 2 \left[(p^{(1)} - p^{(n(n-1)/2 - k' - 1)}) + ... + (p^{(k')} - p^{(n(n-1)/2)}) + \right. \nonumber \\ 
    & \quad \quad \quad \quad \quad \quad \quad \quad \quad \quad \quad \quad \quad  \left. p^{(k'+1)} + ... + p^{(k)}  \right] - (d'-d) \nonumber \\
    & \leq 2 \left[(k' \times s + (1/2 - h) (d'-d) \right] - (d'-d) \nonumber \\
    & \leq 2(d \times s - h) \nonumber \\
    & \leq 0 \nonumber
\end{align}
\end{proof}

Note also that if $P$ is non-SST, then the global monotonicity property never holds, which can be easily proven by taking a counter-example and following the same proof structure.

\subsection{Proof Proposition \ref{prop:sample} (learning rate bounds)}
\label{suppl:proof_bounds}

Here we prove the finite sample results stated in the proposition below.

\propsample*

\begin{proof}
Hoeffding inequality combined with  the union bound yields: $\forall t>0$,
\begin{multline*}
\mathbb{P}\left\{ \sup_{\sigma\in \mathfrak{S}_n} \left\vert \widehat{D}_N(\sigma)-D_P(\sigma)\right\vert  >t  \right\}\leq 
 \sum_{\sigma\in \mathfrak{S}_n}\mathbb{P}\left\{ \frac{1}{N}\left\vert\sum_{i=1}^N\left\{ d(\Sigma_i, \sigma)-\mathbb{E}_P[d(\Sigma,\sigma)]    \right\}  \right\vert     >t \right\}
\leq   2n!\exp\left(-\frac{N2t^2}{\vert\vert d\vert\vert_{\infty}^2}  \right),
\end{multline*}
which establishes assertion $(i)$. 

Turning to the proof of assertion $(ii)$, we introduce
$$
\bar{S}_P(u)=\mathbb{P}_{\Sigma}\{\widehat{D}_{\lfloor N/2 \rfloor}(\Sigma)\geq u\},\; \; u\geq 0.
$$
By triangular inequality, we have with probability one:
\begin{multline}\label{eq:triang}
\sup_{u\geq 0}\left\vert (K_h\ast\widehat{S}_N)(u)-(K_h\ast S_P)(u)\right\vert \leq 
\sup_{u\geq 0}\left\vert (K_h*\widehat{S}_N)(u)-(K_h*\bar{S}_P)(u)\right\vert +\\\sup_{u\geq 0}\left\vert (K_h*S_P)(u)-(K_h*\bar{S}_P)(u)\right\vert .
\end{multline}
Observe that we almost-surely have:
$$
  \sup_{u\geq 0}\left\vert (K_h*\widehat{S}_N)(u)-(K_h*\bar{S}_P)(u)\right\vert \leq \sup_{u\geq 0}\left\vert \widehat{S}_N(u)-\bar{S}_P(u)\right\vert.
$$
By virtue of Dvoretsky-Kiefer-Wolfovitz inequality, we have, for all $t\geq 0$,
\begin{multline}\label{eq1}
\mathbb{P}\left\{ \sup_{u\geq 0}\left\vert \widehat{S}_N(u)-\bar{S}_P(u)\right\vert  \geq t \right\}=\mathbb{E}\left[\mathbb{P}\left\{ \sup_{u\geq 0}\left\vert \widehat{S}_N(u)-\bar{S}_P(u)\right\vert  \geq t \mid \Sigma_1,\; \ldots,\; \Sigma_{\lfloor N/2 \rfloor}\right\}\right]
\leq 2\exp(-2nt^2).
\end{multline}

Let $s>0$, we introduce the event, independent from $\Sigma$,
$$
\mathcal{E}_{N,s}=\left\{  \sup_{\sigma\in \mathfrak{S}_n}\left\vert \widehat{D}_{\lfloor N/2 \rfloor} (\sigma)-D_P(\sigma)\right\vert \leq s  \right\}.
$$
We almost-surely have: $\forall u\geq 0$,
$$
\bar{S}_P(u)=\mathbb{P}_{\Sigma}\{ D_{P}(\Sigma)\geq u + D_{P}(\Sigma)-\widehat{D}_{\lfloor N/2 \rfloor}(\Sigma)\}.
$$
Consequently, on the event $\mathcal{E}_{N,s}$, it holds that: $\forall u\geq 0$,
\begin{multline*}
(K_h*S_P)(u+s)-(K_h*S_P)(u)\leq (K_h*\bar{S}_P)(u)-(K_h*\widehat{S}_N)(u)\leq (K_h*S_P)(u)-(K_h*S_P)(u-s),
\end{multline*}
as well as
\begin{equation}\label{eq2}
\sup_{u\geq 0}\left\vert (K_h*S_P)(u)-(K_h*\bar{S}_P)(u)\right\vert \leq \vert\vert K'\vert\vert_{\infty}(s/h),
\end{equation}
since the mapping $K_h*S_P$ being differentiable, with derivative bounded by $\vert\vert K'\vert\vert_{\infty}/h$ in absolute value. Hence, using the union bound, combining \eqref{eq:triang} with assertion $(i)$ and \eqref{eq1}-\eqref{eq2}, we get that for all $\delta\in (0,1)$, with probability larger than $1-\delta$:
\begin{equation*}
\sup_{u\geq 0}\left\vert (K_h\ast\widehat{S}_N)(u)-(K_h\ast S_P)(u)\right\vert \leq \left( \sqrt{\log(4/\delta)}+\vert\vert d\vert\vert_{\infty}\sqrt{\log(4n!/\delta)} \right)/\sqrt{2N}.
\end{equation*}
This proves assertion $(ii)$.
\end{proof}

\subsection{Proofs of Propositions \ref{prop:depth_kendall}, \ref{prop:Kendall} and \ref{prop:mallows} (results for Kendall $\tau$ - Mallows model)}
\label{suppl:proof_kendall}

\propdepthkendall*
\begin{proof}
The proof is a simple computation, recalling that $\forall i \neq j, p_{i,j}=\mathbb{P}(\Sigma(i) < \Sigma(j))$. Then, $D_P(\sigma) = \vert \vert d \vert \vert_{\infty} - \mathbb{E}_P( d_{\tau}(\Sigma, \sigma) ) = \binom{n}{2} - \sum_{i<j} \mathbb{P}\left( (\Sigma(i) - \Sigma(j)) (\sigma(i) - \sigma(j)) < 0 \right) = \binom{n}{2} - \sum_{i<j}p_{i,j}\mathbb{I}(\sigma(i)>\sigma(j)) - \sum_{i<j}(1-p_{i,j})\mathbb{I}(\sigma(i)<\sigma(j))$ by simple conditioning.
\end{proof}

\propkendall*
\begin{proof}
Observing that $n(n-1)/2=L_P(\sigma)+L_P(n-\sigma)$ for all $\sigma\in \mathfrak{S}_n$ in the Kendall $\tau$ case, the result is essentially a reformulation of Theorem 5 in \cite{CKS17} in terms of ranking depth, insofar as $D_P=n(n-1)/2-L_P$.
\end{proof}

Let us recall some classical results about the Mallows distribution. Taking $d=d_{\tau}$, the Mallows model introduced in \cite{Mallows57} is the unimodal distribution $P_{\theta}$ on $\mathfrak{S}_n$ parametrized by $\theta=(\sigma_0, \phi_0)\in \mathfrak{S}_n\times (0,1]$: $\forall \sigma \in \mathfrak{S}_n$,
$P_{\theta}(\sigma)=(1/Z_0)\exp( d_{\tau}(\sigma_0,\sigma) \log \phi_0)$, where $Z_0=\sum_{\sigma\in \mathfrak{S}_n} \exp(d_{\tau}(\sigma_0,\sigma)\log \phi_0 )$ is a normalization constant.

One may easily show that $Z_0$ is independent from $\sigma_0$ and that $Z_0=\prod_{i=1}^{n-1}\sum_{j=0}^i \phi_0^j$. When $\phi_0<1$, the permutation $\sigma_0$ of reference is the mode of distribution $P_{\theta_0}$, as well as its unique median relative to $d_{\tau}$. Observe in addition that the smallest the parameter $\phi_0$, the spikiest the distribution $P_{\theta_0}$. In contrast, $P_{\theta_0}$ is the uniform distribution on $\mathfrak{S}_n$ when $\phi_0=1$. 

A closed-from expression of the pairwise probabilities $p_{i,j}$ is available (see \textit{e.g.} Theorem 2 in \cite{DBLP:conf/icml/Busa-FeketeHS14}). Setting $h(k,\phi_0)=k/(1-\phi_0^k)$ for $k\geq 1$, one can then show the following: that the ranking depth function relative to $P_{\theta}$ and $d_{\tau}$ is given by:
\begin{restatable}{proposition}{propmallows}
\label{prop:mallows}
If $P = P_{\theta}$ the Mallows distribution and $d=d_{\tau}$ the Kendall $\tau$ distance, then $\forall \sigma\in \mathfrak{S}_n$, $D_{P_{\theta}}(\sigma) = {n \choose 2} - \sum_{\sigma(i)>\sigma(j)} H(\sigma_0(j)-\sigma_0(i),\; \phi_0)$, where $H(k,\phi_0)=h(k+1,\phi_0)-h(k,\phi_0)$ and $H(-k,\phi_0)=1-H(k,\phi_0)$  for $k\geq 1$.
\end{restatable}

where $h(k, \phi_0) = k/(1-\phi_0^k)$ for $k \geq 1$
\begin{proof}
Theorem 2 in \cite{DBLP:conf/icml/Busa-FeketeHS14} states that, for the Mallows model and using our notations, $\forall i \neq j, p_{i,j} = H(\sigma_0(j)-\sigma_0(i), \phi_0)$. The results follows from direct application of proposition~\ref{prop:depth_kendall}
\end{proof}.

\subsection{Proof of Proposition \ref{prop:borda_breakdown_ratio} (Borda estimators' robustness)}
\label{suppl:proof_borda}

Proposition~\ref{prop:borda_breakdown_ratio} refers to the robustness of the depth-trimmed-Borda compared to the classical Borda.
\propborda*

In this subsection, we will in fact proves some auxiliary results as well as a generalization of this proposition.

Let us first recall some definitions and results about the Borda estimators. Borda is an approximation to the barycentric ranking median (which is NP-hard for $n>4$~\cite{Dwork:2001:RAM:371920.372165}) for a sample of complete rankings drawn from a MM~\cite{Fligner1988}. Moreover, Borda is quasi-linear in time and outputs the correct median w.h.p. with a polynomial number of samples~\cite{Caragiannis2013}. A robust aggregation procedure for top-$k$ rankings in very noisy settings is proposed in~\cite{collas21}.


The Borda median estimator for sample $X$ orders the items increasingly by their Borda score, defined as $B(i) = \sum_{\sigma\in X}  \sigma(i)$.

We define the depth-weighted-Borda as a generalization of the classic and depth-trimmed-Borda in which there exists a weight associated with each ranking. It generalizes Borda in the following way: For each item $i$, the Borda score is computed as $B(i) = \sum_{\sigma\in X} w(\sigma) \sigma(i)$. The final estimator for the median is the ranking that orders the items by their Borda score. The depth-weighted-Borda is equivalent to replicating the rankings proportionally to their weight. 
This analysis generalizes to any weights are \textit{increasing} function of the depths. In particular, the depth-trimmed-Borda is the case of depth-weighted-Borda in which $w(\sigma) = \mathbb{I} \{D(\sigma) > \mu \}$.

We settle here the notation for the following lines. We denote by $S_N \sim P$ a sample of rankings (of size N) and $A$ an adversarial sample.
\begin{definition}
Let $P$ be a distribution, let us write $S_N \sim P$ a sample drawn from $P$ of size $N$ and $\sigma^{T}_{S_N}$ the median based on the estimator method $T$ on sample $S_N$.

The estimator $T$ is said to be $\delta$-broken (for Kendall's $\tau$) for sample size $N$ and distribution $P$ if for any $S_N \sim P$ of size $N$, there exists an adversarial sample $A$ such that $d_{\tau}(\sigma^{T}_{S_N}, \sigma^{T}_{S_N \cup A}) \geq \delta$.
\end{definition}


The next result characterizes the carnality of a sample that breaks the Borda estimator of a sample $S_N$ distributed according to $P$. This is an auxiliary result for Proposition~\ref{prop:borda_breakdown_ratio}.

\begin{proposition}
\label{thm:break_borda_d}
Let $S_N \sim P$. Let $A^{-}$ be the adversarial sample that $\delta$-breaks the Borda estimator (for sample size $N$ and distribution $P$) such that $A^{-}$ is of minimal cardinality. Let $\bar r_N(i) = N^{-1} \sum_{\sigma\in S_N} \sigma( i )$ and $\bar r(i) =  (\# A^{-})^{-1} \sum_{\sigma\in A^{-}} \sigma( i )$ be the average ranking of item $i$ in $S_N$ and $A^{-}$ respectively. Finally, let $\bar R$ be the ordered vector composed of $ \frac{ \bar r_N(j) - \bar r_N(i)}{  \bar r( i ) - \bar r( j )  } $ for all $(i,j)$ such as both the numerator and denominator are positive. 
Then 
$$    \# A^{-}  = \lceil N \left[ \bar R \right]_{(\delta)} \rceil$$
where $[x]_{(\delta)}$ denotes the $\delta$-{th} quantile of a vector $x$.
\end{proposition}

\begin{proof}
By definition, $A^{-}$ $\delta$-breaks Borda iff the following holds.
\begin{equation}
    \begin{split}
    d & (\sigma^{B}_{S_N}, \sigma^{B}_{S_N \cup A^{-}}) = \delta   \\
    &  \Leftrightarrow \delta = \# \{  (i<j) : \sum_{\sigma \in S_N} \sigma(i) + \sum_{\sigma \in A^{-}} \sigma(i) \geq \sum_{\sigma \in S_N} \sigma(j) + \sum_{\sigma \in A^{-}} \sigma(j) \}  \\
    &  \Leftrightarrow \delta = \#  \{  (i<j) : \sum_{\sigma \in S_N} \sigma(i) - \sigma(j) \geq \sum_{\sigma \in A^{-}} \sigma(j) - \sigma(i) \}  \\
    &  \Leftrightarrow \delta = \#  \{  (i<j) : \sum_{\sigma \in S_N} \sigma( j ) - \sigma( i ) \leq \sum_{\sigma \in A^{-}} \sigma( i ) - \sigma( j ) \} \\
    & \Leftrightarrow \delta = \#  \{  (i,j) : 0 < \sum_{\sigma \in S_N} \sigma( j ) - \sigma( i ) \leq \sum_{\sigma \in A^{-}} \sigma( i ) - \sigma( j ) \}
    \label{eq:break_borda_d}
    \end{split}
\end{equation}

From a statistical perspective, we can bound the cardinality of $A^{-}$ as follows: let $(i,j)$ be a pair of index belonging to the set define just above.
\begin{equation}
\begin{split}
&   \sum_{\sigma \in S_N} \sigma( j ) - \sigma( i ) \leq \sum_{\sigma \in A^{-}} \sigma( i ) - \sigma( j ) \\
&  \Leftrightarrow  N   \: (\bar r_N(j) - \bar r_N(i) ) \leq   \# A^{-}  (\bar r( i ) - \bar r( j ) )\\  
&  \Rightarrow  \# A^{-}  \geq \frac{ N   \: (\bar r_N(j) - \bar r_N(i) ) }{  \bar r( i ) - \bar r( j )  },
\label{eq:card_y_borda}
\end{split}
\end{equation}
which holds for exactly $\delta$ pairs of items $(i,j)$. We conclude the proof by recalling that $A^{-}$ is of minimal cardinality.
\end{proof}

The next auxiliary result shows that provided certain conditions, if a sample breaks the depth-weighted-Borda then it breaks Borda.

\begin{proposition}
\label{thm:Y_ratio}
Let $S_N \sim P$. Let $A^{-}$ (resp. $A^{-}_w$) be the adversarial sample that $\delta$-breaks the Borda (resp. depth-weighted Borda) estimator (for sample size $N$ and distribution $P$) such that $A^{-}$ (resp. $A^{-}_w$) is of minimal cardinality.
Let $\bar r_N(i) = N^{-1} \sum_{\sigma\in S_N} \sigma( i )$ and $ \bar r_w( i ) = (\# A^{-}_w) ^{-1} \sum_{\sigma\in A^{-}_w} \sigma( i )  $ be the average ranking of item $i$ in $S_N$ and $A^{-}_w$ respectively. 

Let $ \pi_w = \arg\max_{\sigma \in  A^{-}_w } w(\sigma) $ and $\mu = w(\pi_w)$ the threshold of maximum depth for adversarial rankings. 

Finally, suppose that $\hat{P}_N$ and $w$ satisfy: $ \mathbb{E}_{\hat{P}_N}(w(\Sigma)) > w(\pi_w) = \mu$  and $\; \forall \; (i,j) \text{ s.t. } \mathbb{E}_{\hat{P}_N}(\Sigma(j) - \Sigma(i)) > 0$, $\mathbb{E}_{\hat{P}_N}(w(\Sigma)(\Sigma(j) - \Sigma(i)) \geq \mathbb{E}_{\hat{P}_N}(w(\Sigma))\mathbb{E}_{\hat{P}_N}(\Sigma(j) - \Sigma(i))$ (these two assumptions enforce the use of a weight function that is in accordance with $\hat{P}_N$)

Then, the cardinality of $A^{-}$ and $A^{-}_w$ are related as follows:
$$ \# A^{-}_w  \geq  \frac{ N^{-1} \sum_{\sigma \in S_N} w(\sigma) }{ \mu }  \# A^{-}.$$
\end{proposition}

\begin{proof}

Since $A^{-}_w$ $\delta$-breaks the depth-weighted-Borda, we can follow the same proof outline as for proposition \ref{thm:break_borda_d} and bound the cardinality $ \# A^{-1}_w $ as follows,
\begin{equation}
\begin{split}
&   \sum_{\sigma \in S_N} w(\sigma) ( \sigma( j ) - \sigma( i ) ) \leq \sum_{\sigma \in A^{-}_w} w(\sigma) ( \sigma( i ) - \sigma( j ) )\\
&  \Rightarrow N \times  \: N^{-1} \sum_{\sigma \in S_N} w(\sigma)(\sigma(j) - \sigma(i)) \leq   \# A^{-}_w  w( \pi ) (\bar r_w( i ) - \bar r_w( j ) )\\  
&  \Rightarrow  \# A^{-}_w  \geq \frac{ N   \: (\bar r_n(j) - \bar r_N(i)) }{ \bar r_w( i ) -  \bar r_w( j )  } \frac{ N^{-1} \sum_{\sigma \in S_N} w(\sigma) }{ \mu } \\
\label{eq:card_y_borda_dinf}
\end{split}
\end{equation}
Since $ \frac{ N^{-1} \sum_{\sigma \in S_N} w(\sigma) }{ \mu }$ is independent of $i,j$ and $A^{-}_w$ also $\delta$-breaks the Borda estimator, we can conclude:
\begin{equation}
    \begin{split}
        \# A^{-}_w  \geq \# A^{-}  \frac{ N^{-1} \sum_{\sigma \in S_N} w(\sigma) }{ \mu }.
    \end{split}
\end{equation}

\end{proof}

We are finally ready to prove a generalization of our proposition \ref{prop:borda_breakdown_ratio} stated in the main paper. Let us first define our notion of $\delta$-\textit{breakdown point}, which extends the classical concept.

\begin{definition}
Let $P$ be a distribution. The $\delta$-breakdown point for an estimator $T$ with respect to distribution $P$ is defined as the smallest cardinality of an adversarial sample that $\delta$-breaks $T$ in the limit when $N \to \infty$ for distribution $P$.

More specifically, $\epsilon_{\delta}^{T} (P) = \min \# A$ s.t. $\lim_{N \to \infty} d_{\tau}(\sigma^T_{S_N}, \sigma^T_{S_N \cup A}) = \delta$
\end{definition}

In the following proposition, we write $\epsilon^{B}_{\delta} (P) $ (resp. $\epsilon^{DW-B}_{\delta} (P)$) the $\delta$-breakdown point for the Borda (resp. depth-weighted Borda) estimator with respect to distribution $P$.

\begin{proposition}[breakdown points ratio]\label{thm:breakdown_ratio}
Let $P$ be a distribution such that $\mathbb{E}_{ P } [ w(\Sigma) ] > w(\pi) $, where $\pi = \arg\max_{\sigma \; | \; d_{\tau}(\sigma^*, \sigma) = \delta} w(\sigma)$ and $\sigma^* = \arg\max_{\sigma \in \mathfrak{S}_n} D_{P}(\sigma)$. Let $P$ and $w$ satisfy: $\forall (i,j) \text{ s.t. } \mathbb{E}_P(\Sigma(j) - \Sigma(i)) > 0$, $\mathbb{E}_P(w(\Sigma)(\Sigma(j) - \Sigma(i)) ) \geq \mathbb{E}(w(\Sigma)) \mathbb{E}(\Sigma(j) - \Sigma(i))$. Then,

\begin{equation}
\begin{split}
\lim_{N \to\infty} \frac{\epsilon_{\delta}^{B} (P)}{\epsilon_{\delta}^{DW-B} (P)} 
< \frac{ w( \pi ) }{ \mathbb{E}_{ P } [ w(\Sigma) ] } < 1.
\end{split}
\end{equation}
\end{proposition}

\begin{proof}
We start by noting that for $S_N$ to be $\delta$-broken then the adversarial sample has to be at least at distance $\delta$ regardless the distribution for the weights.  
Then, we denote $z = \mathbb{E}_{ P } [ w(\Sigma) ] / w( \pi ) = \lim_{N \to \infty} N^{-1} \sum_{\sigma \in S_N} w(\sigma) / w(\pi)$ (by the law of large numbers) and take Proposition~\ref{thm:Y_ratio} to write the limiting ratio of the breakdown points when the number of samples tends to infinity as follows. 

\begin{equation}
\begin{split}
\lim_{N  \to\infty} \frac{\epsilon_{\delta}^{B}(P)}{\epsilon_{\delta}^{DW-B}(P)} 
= \lim_{N  \to\infty} \frac{ \frac{\# A^{-} }{\# A^{-} + N } }{ \frac{\# A^{-}_w  }{\# A^{-}_w  + N  } }  
< \lim_{N  \to\infty} \frac{ \frac{\# A^{-} }{\# A^{-} + N } }{ \frac{\# A^{-} \cdot z}{\# A^{-} \cdot z\; +\; N  } }  
< \frac{1}{z} =  \frac{ w( \pi ) }{ \mathbb{E}_{ P } [ w(\Sigma) ] } < 1
\end{split}
\end{equation}
\end{proof}

This is the main result related to the robustness of the Borda median estimator. It shows that the breakdown point  of Borda is smaller than the breakdown point for the depth-trimmed-Borda provided certain conditions. We denote by $\mu$ the threshold of the depth-trimmed-Borda. 

Then, our proposition \ref{prop:borda_breakdown_ratio} is straightforward when we choose the weight function $w$ so that $w(\sigma) = \mathbb{I}(D_P(\sigma) \geq \mu)$ in Proposition~\ref{thm:breakdown_ratio}.

\section{Further results}
\label{suppl:further_results}

\subsection{Ranking quantile function}
\label{suppl:ranking_quantile_fct}

In Proposition \ref{prop:sample}, rate bounds for the deviation between empirical and theoretical versions of the depth function (respectively, of the smoothed depth survivor function) have been stated. We here give some indications for obtaining similar results for the ranking quantile function.

As the considered distribution is discrete, finite and real-valued (because depth is real-valued), the results of \cite{Ma11} can be directly applied. Using their notations, we have:

$\Sigma$ is a random variable of distribution $P$ on $\mathfrak{S}_n$ that takes distinct values $\sigma_1,\; \ldots,\; \sigma_m$ with respective propabilities $\rho_1,\; \ldots,\; \rho_m$. Each element $\sigma_i \in \{ \sigma_1, ..., \sigma_m \}$ has a depth $\delta_i$, which leads us to write $D_P(\Sigma)$ the random variable associated to the depth.

Now, let us reorder the indices and write the distinct depth values $\delta_1 <...<\delta_d$ with respective probabilities of occurrence $p_1,\; \ldots,\; p_d$, where $\forall \; i \in [\![1,d]\!], p_i = \sum_{1 \leq j \leq m} \rho_j \mathbb{1}(\delta_j = \delta_i)$. Let us define the \textit{mid-function} $F_{mid}(x) = \mathbb{P}(D_P(\Sigma) \leq x) - 1/2 \mathbb{P}(D_P(\Sigma) = x)$ and the ranking quantile function based on \textit{mid-functions}:
\begin{equation}
  Q(\alpha) = F_{mid}^{-1}(\alpha) =
    \begin{cases}
      \delta_1 & \text{if $\alpha \leq p_1/2$}\\
      \lambda \delta_k + (1-\lambda)\delta_{k+1} & \text{if $\alpha=\lambda \pi_k + (1-\lambda)\pi_{k+1}$} \\ 
      &\quad\quad\text{for any $\lambda \in [0,1]$ and $1 \leq k \leq d-1$}\\
      \delta_d & \text{if $\alpha \geq \pi_d$}
    \end{cases} ,
\label{def:smooth_quant_ma}     
\end{equation}
where $\forall k \in [\![2,d]\!], \pi_k = \sum_{i=1}^{k-1}p_i + p_k/2 = F_{mid}(\delta_k)$.

Then, the following results hold:
\begin{align*}
  1) & \; \hat{Q}_N(\alpha) \overset{\mathbb{P}}{\longrightarrow} \delta_1 \text{ if $\alpha < p_1/2$} \\
  2) & \; \hat{Q}_N(\alpha) \overset{\mathbb{P}}{\longrightarrow} \delta_d \text{ if $\alpha > \pi_d$} \\
  3) & \; \sqrt{N}(\hat{Q}_N(\alpha) - (\lambda \delta_{k+1} + (1-\lambda)\delta_{k+2})) \overset{\mathbb{P}}{\longrightarrow} \mathcal{N}(0, sd(\alpha, \lambda, p_{k+1}, p_{k+2}) ) \\ 
  & \quad\quad\quad\quad \text{ if $\alpha = \lambda\pi_{k+1} + (1-\lambda)\pi_{k+2}$ for $0<\lambda<1$ and $0\leq k \leq d-2 $} \\
  4) & \; \sqrt{N}(\hat{Q}_N(\alpha) - \delta_{k+1})f(\hat{Q}_N(\alpha), \delta_{k+1})  \overset{\mathbb{P}}{\longrightarrow} \mathcal{N}(0, \alpha(1-\alpha) - p_{k+1}/4) \\
  & \quad\quad\quad\quad  \text{ if $\alpha = \pi_{k+1}$ for $0 \leq k \leq d-1$},
\end{align*}
where $sd(\alpha, \lambda, p_{k+1}, p_{k+2}) = \alpha (1-\alpha) - (1-(\lambda-1)^2)p_{k+1}/4 - (1-\lambda^2)p_{k+2}/4$ and $f(\hat{Q}_N(\alpha), \delta_{k+1}) = 1/2 (p_{k+1}+p_{k+2})/(\delta_{k+2}-\delta_{k+1}) \text{ if $\hat{Q}_N(\alpha) > \delta_{k+1}$ and } 1/2 (p_{k+1}+p_{k})/(\delta_{k+1}-\delta_{k})$ else.

These results provide us with asymptotic guarantees about the ranking quantile function based on mid-functions, as defined in eq.~\ref{def:smooth_quant_ma}. However, non-asymptotic bounds as well as similar results for the depth regions should be investigated further and are left for future work, like the discrepancy between empirical and theoretical ranking depth regions, which can be measured by \textit{e.g.} the cardinality of their symmetric difference.

\subsection{Pairwise comparisons as an alternative statistical framework}
\label{suppl:pairwise_comp}

Since the computation of Kendall $\tau$ distance involves pairwise comparisons only, one could compute empirical versions of the risk functional $L$ in a statistical framework stipulating that the observations are less complete than $\{\Sigma_1,\; \ldots,\; \Sigma_N  \}$ and formed by i.i.d. pairs:
$$
(\mathbf{e}_k,\; \epsilon_k),\;  k=1,\; \ldots,\; N,
$$ 
where the $\mathbf{e}_k=(\i_k,\j_k)$'s are independent from the $\Sigma_k$'s  and drawn from an unknown distribution $\nu$ on the set $\mathcal{E}_n$ such that $\nu(e)>0$ for all $e\in \mathcal{E}_n$ and $\epsilon_k=sgn( \Sigma_k(\j_k)- \Sigma_k(\i_k) )$ with $\mathbf{e}_k=(\i_k,\j_k)$ for $1\leq k\leq N$. Based on these observations, an estimate of the risk of any median candidate $\sigma\in \mathfrak{S}_n$ is given by:
\begin{equation}
\widetilde{L}_N(\sigma)=\sum_{i<j}\frac{1}{N_{i,j}}\sum_{k=1}^N\mathbb{I} \{ \mathbf{e}_k=(i,j),\;  \epsilon_k(\sigma(j)-\sigma(i))<0 \},
\end{equation}
where $N_{i,j}=\sum_{k=1}^N\mathbb{I}\{ \mathbf{e}_k=(i,j)\}$.


\section{Additional experiments}
\label{suppl:additional_exp}

Here we display additional numerical results, completing those presented in the main text. First, in Section~\ref{ssec:synthetic}, we analyze the sensitivity of the proposed depth notion to a difference between distributions and its subsequent ability to provide formal inference. Second, in Section~\ref{ssec:sushi}, we detail a further application to real data. Please note, that, as this is the case in the main text, in all visualizations the data depth is re-scaled to $[0,1]$ interval by division by maximal possible distance for given $n$.

\subsection{Trimming strategy}
\label{suppl:trimming_strat}

Now we have characterized under which conditions the different properties of Property~\ref{property:property_depth_rankings} hold, we explore how to use them in practice. For example, using Kendall $\tau$ distance and samples drawn from a Mallows distribution easily make the invariance and maximality at the center properties hold, but not necessarily the monotonicity property.

First, even though a Mallows model is SST, its empirical distribution counterpart may not be.
Second, the adjacent condition for local monotonicity indicates that under such a model,the monotonicity property hold for the median and any of its adjacent ranking. Moreover, the second local condition is more likely to be satisfied for rankings close to (in terms of Kendall $\tau$ distance) the median.

These two observations and the fact that the depth of a ranking $\sigma$ represents its centrality within the dataset make a \textit{trimming strategy} highly relevant for consensus ranking experiments. The intuition behind this strategy is that least deep points corresponds to \textit{outliers} for the dataset: removing them step by step should make the dataset less noisy, so that the depth of the remaining points get more and more accurate. To the extreme, when successively trimming rankings in the dataset until there is only one ranking left should leave us with an accurate median for the dataset. However, since the depth function satisfies useful properties when the underlying distribution is SST, a sufficient trimming strategy would stop there. The algorithm corresponding to this strategy is defined in Algorithm \ref{algo:trimming} of the main paper, recalled here.
\algotrimming*

\begin{figure}[!h]
    \begin{center}
    \begin{tabular}{cc}
        (a) & (b) \\
         \includegraphics[scale=0.3,clip=true,page=1]{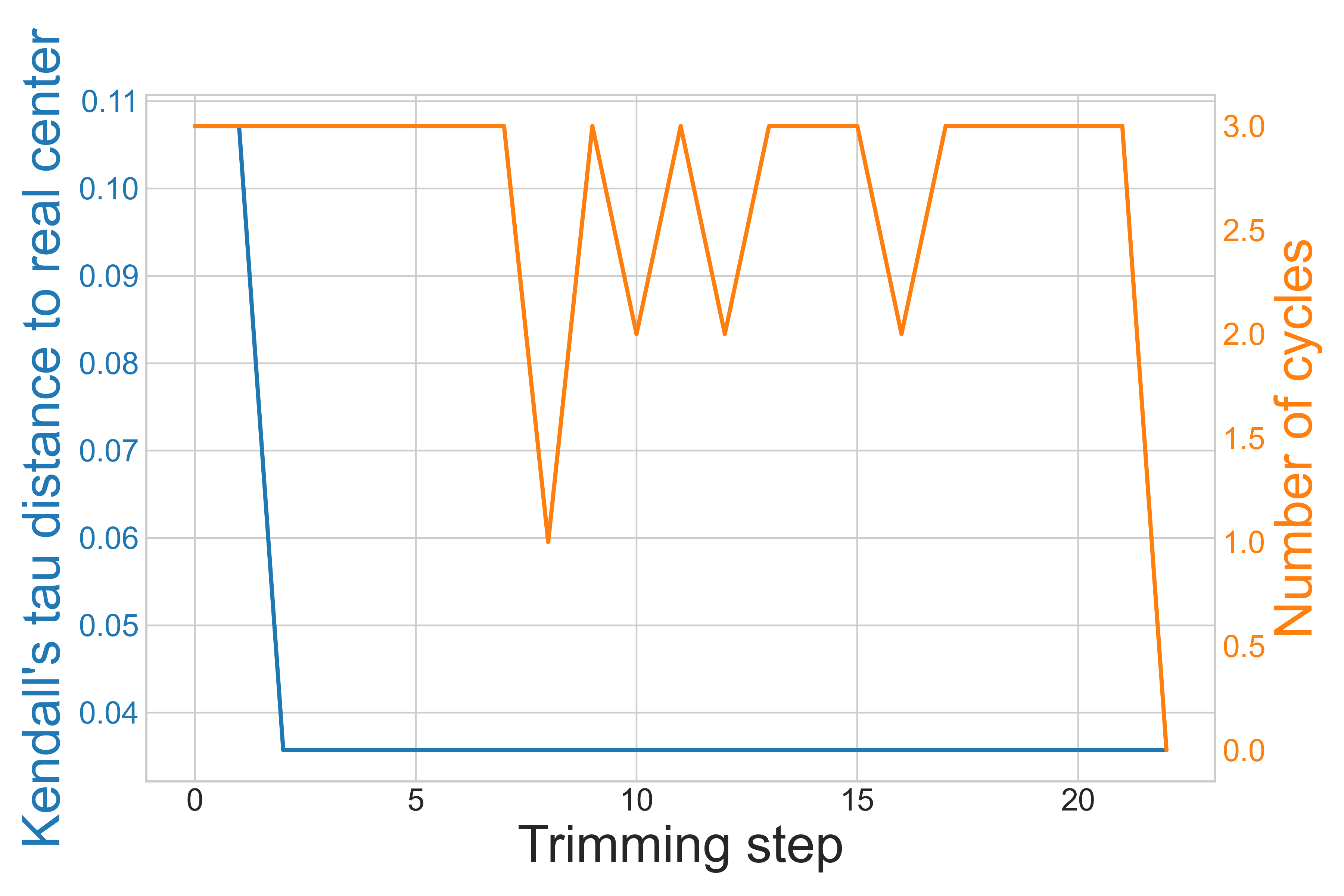} &
         \includegraphics[scale=0.3,clip=true,page=1]{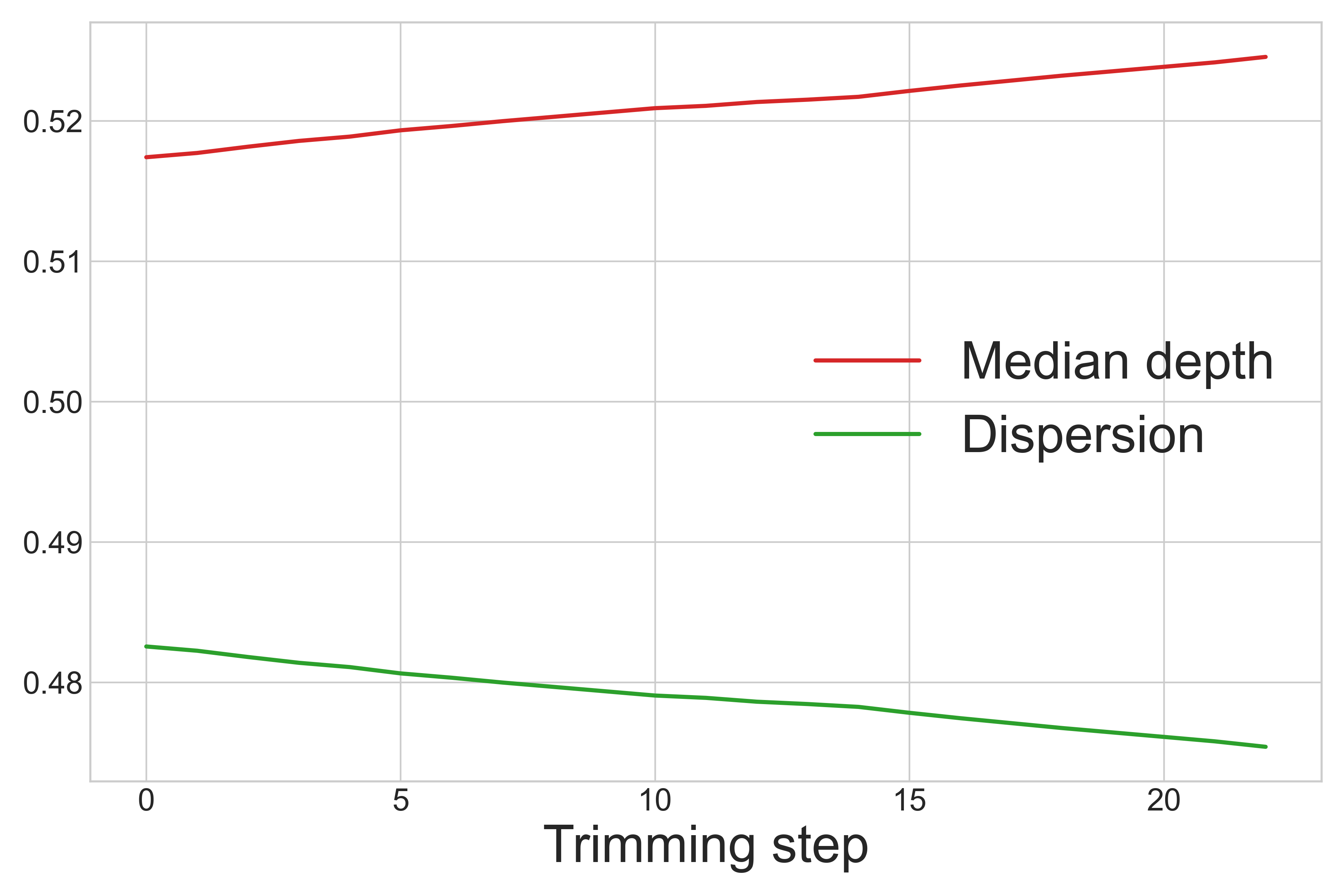}
    \end{tabular}
    \caption{Trimming strategy: evolution of candidate median (deepest ranking) normalized distance to real median and number of cycles through trimming (a); evolution of median depth and sample dispersion through trimming (b).}
    \label{fig:trimming_plot_normal}
    \end{center}
\end{figure}

Fig.~\ref{fig:trimming_plot_normal} illustrates the trimming strategy for a Mallows model generated with $n=8$ items, $\phi_0 = 0.985$, and $N=1000$ samples. We can see that trimming indeed remove cycles from the empirical dataset and thus make the empirical distribution SST, and that during trimming, the deepest rankings (saved as the candidate medians) get closer (in Kendall $\tau$ distance) to the real median used for generating the samples.

The depth function thus provide an alternative and relevant way to compute the median of a dataset: by trimming until getting a SST distribution first, we ensure that the depth function has desirable properties and thus that the median we obtain in practice gets very close to the true median.

\subsection{Visual analysis}
\label{suppl:visual_analysis}

With data depth being a nonparametric tool not exploiting \textit{a priori} information about the distribution, we focus on easy-to-manipulate Mallows model using Kendall $\tau$ distance; we refer to the main text for the formal definition and parameters' notation, see Example~1.

\begin{figure}[!h]
	\begin{center}
		\begin{tabular}{ccccc}
			{\tiny \quad $d_\tau(\sigma_1^*,\sigma_2^*)=0$} & {\tiny \quad $d_\tau(\sigma_1^*,\sigma_2^*)=1$} & {\tiny \quad $d_\tau(\sigma_1^*,\sigma_2^*)=3$} & {\tiny \quad $d_\tau(\sigma_1^*,\sigma_2^*)=5$} & {\tiny \quad $d_\tau(\sigma_1^*,\sigma_2^*)=7$} \\
			\includegraphics[width=0.17\textwidth,trim = 0 0 8mm 9.25mm,clip=true]{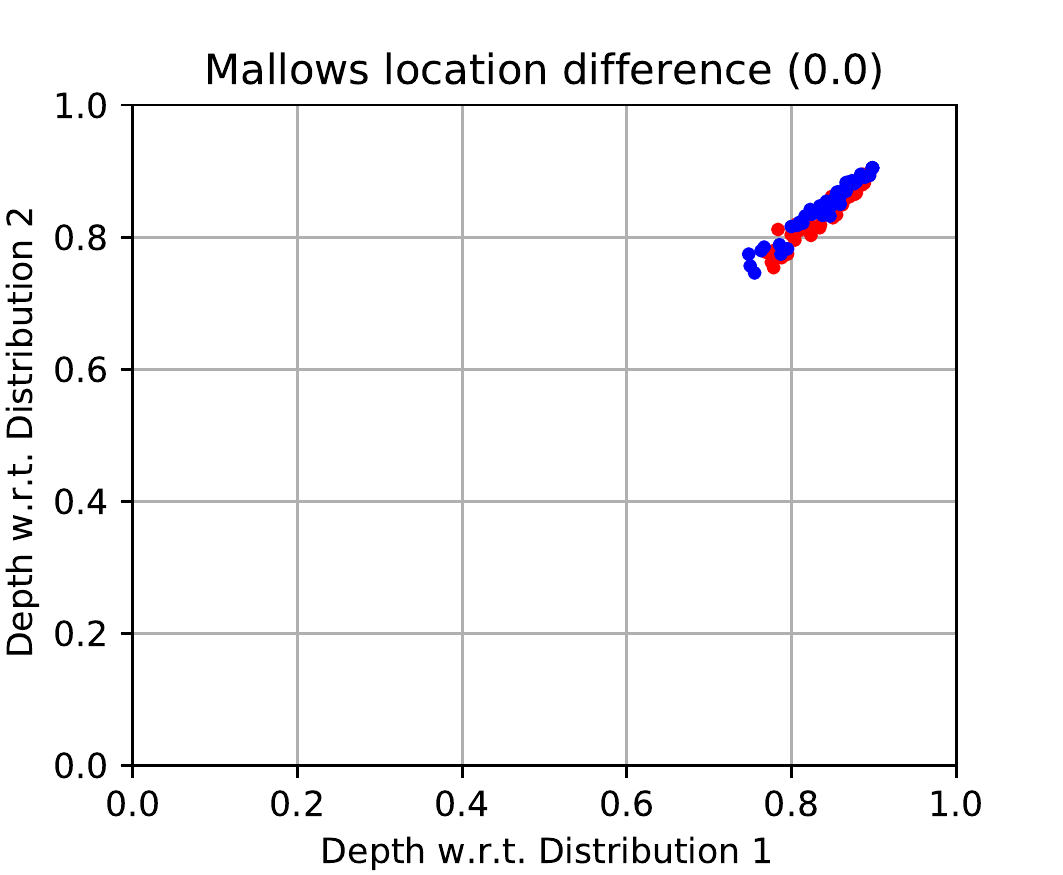} & \includegraphics[width=0.17\textwidth,trim = 0 0 8mm 9.25mm,clip=true]{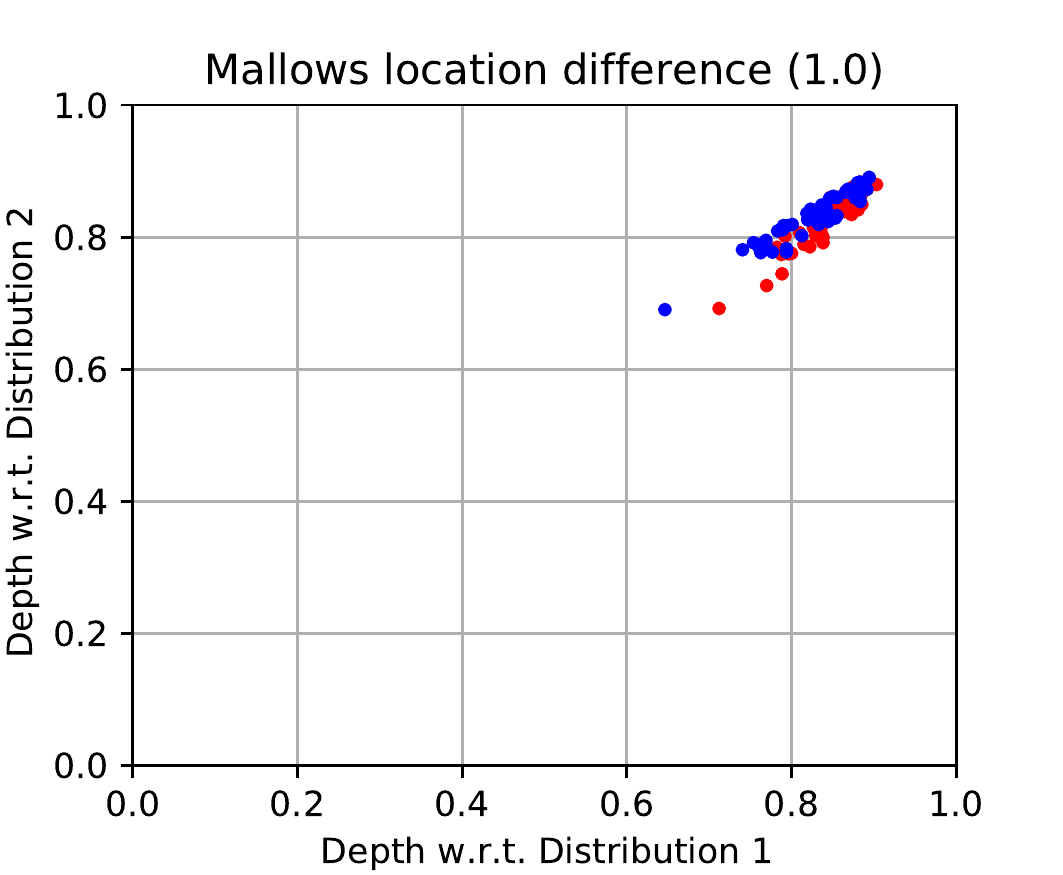} & \includegraphics[width=0.17\textwidth,trim = 0 0 8mm 9.25mm,clip=true]{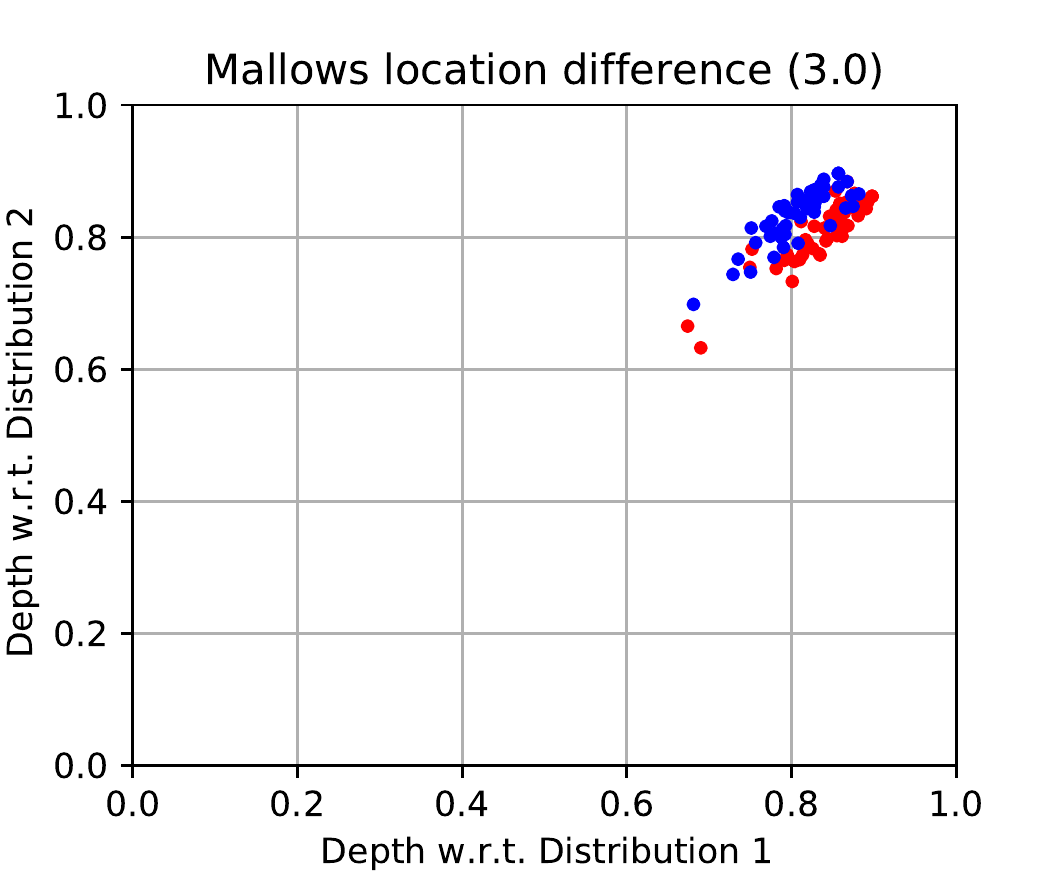} & \includegraphics[width=0.17\textwidth,trim = 0 0 8mm 9.25mm,clip=true]{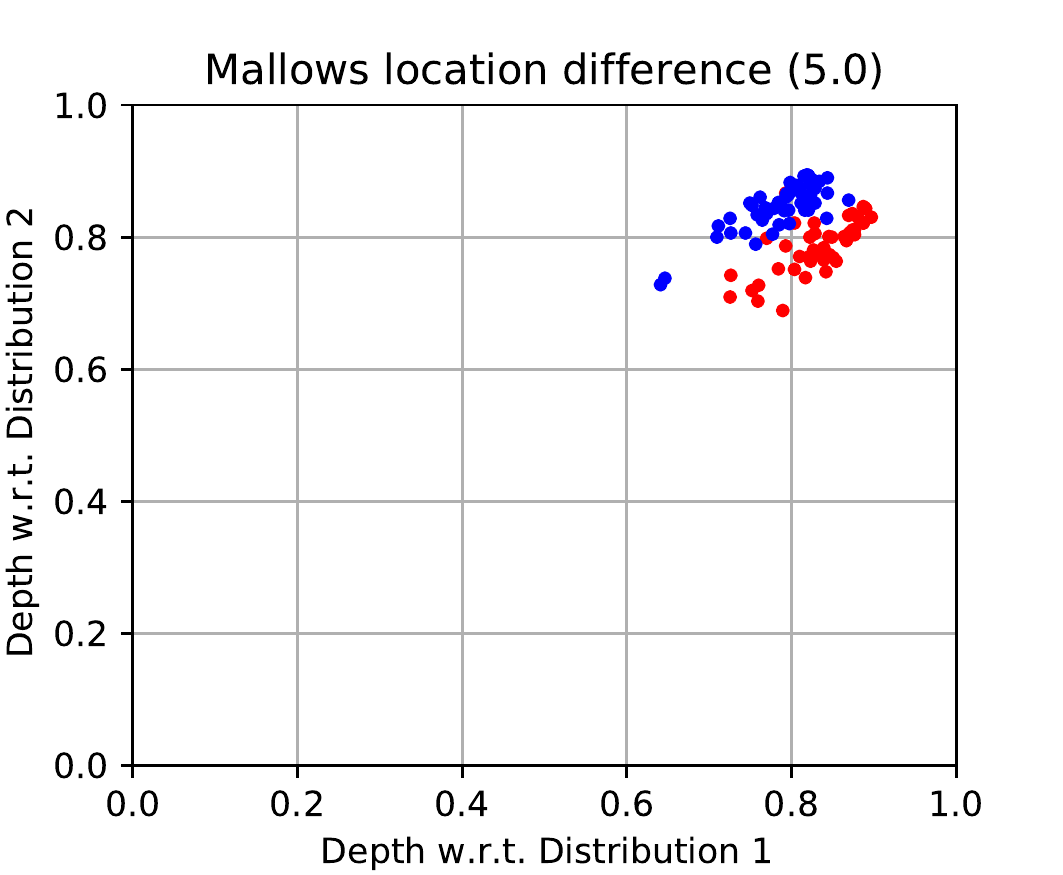} & \includegraphics[width=0.17\textwidth,trim = 0 0 8mm 9.25mm,clip=true]{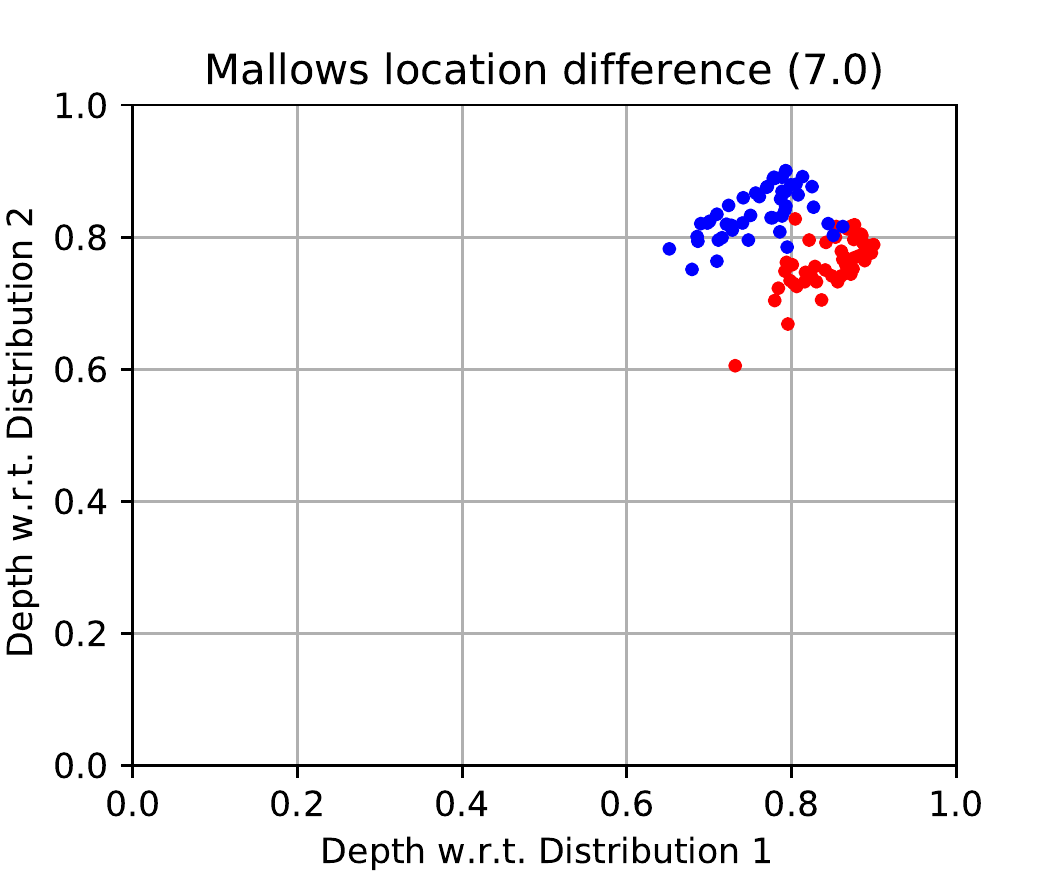} \\
			{\tiny \quad $d_\tau(\sigma_1^*,\sigma_2^*)=9$} & {\tiny \quad $d_\tau(\sigma_1^*,\sigma_2^*)=11$} & {\tiny \quad $d_\tau(\sigma_1^*,\sigma_2^*)=13$} & {\tiny \quad $d_\tau(\sigma_1^*,\sigma_2^*)=15$} & {\tiny \quad $d_\tau(\sigma_1^*,\sigma_2^*)=17$} \\
			\includegraphics[width=0.17\textwidth,trim = 0 0 8mm 9.25mm,clip=true]{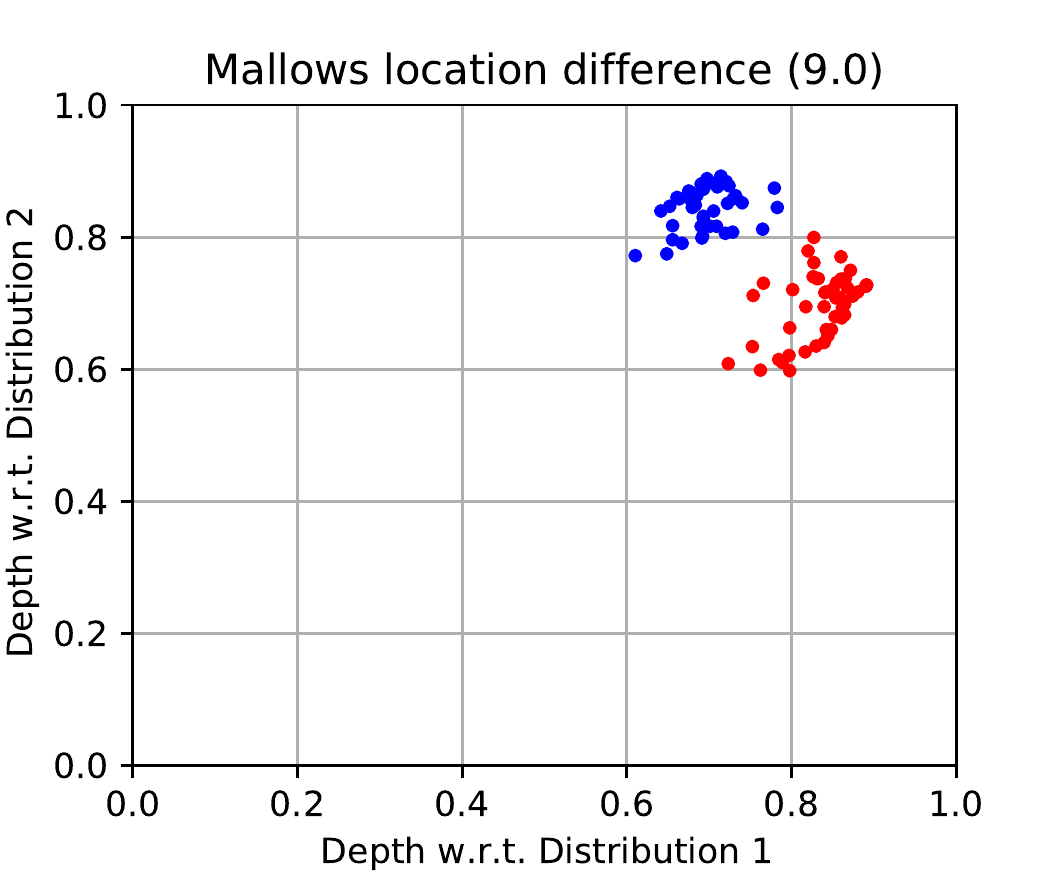} & \includegraphics[width=0.17\textwidth,trim = 0 0 8mm 9.25mm,clip=true]{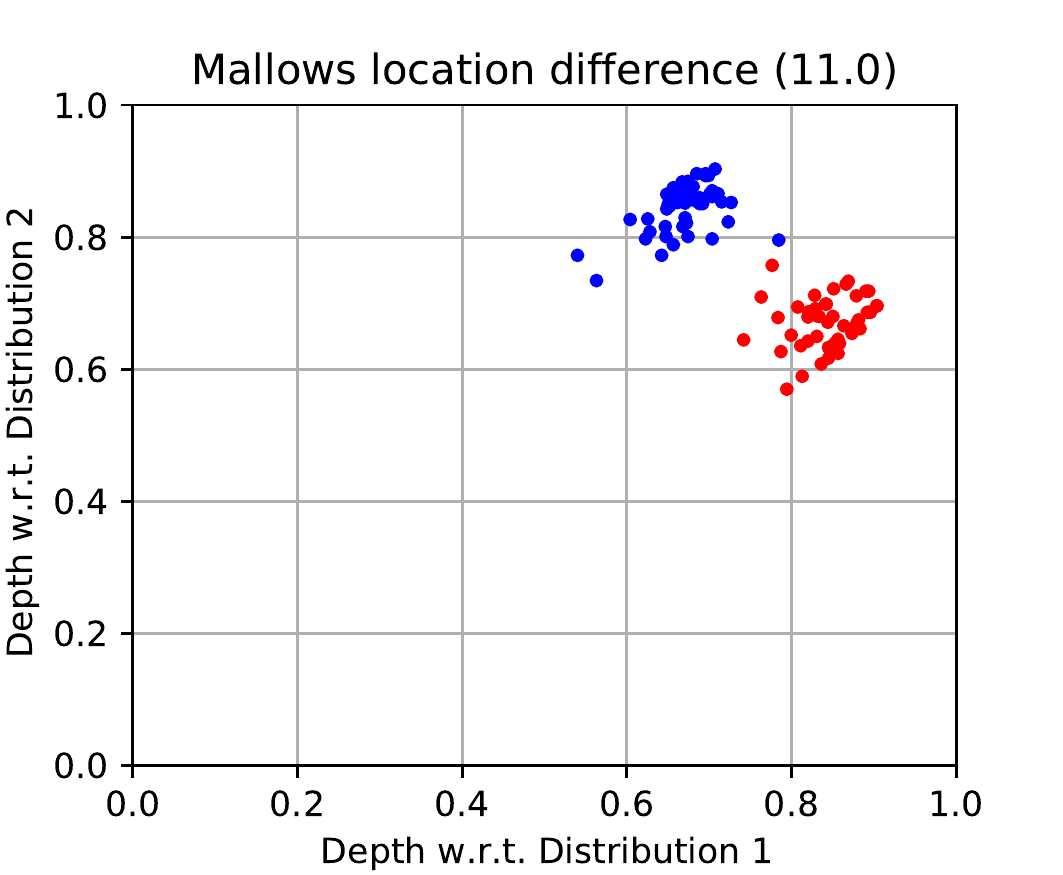} & \includegraphics[width=0.17\textwidth,trim = 0 0 8mm 9.25mm,clip=true]{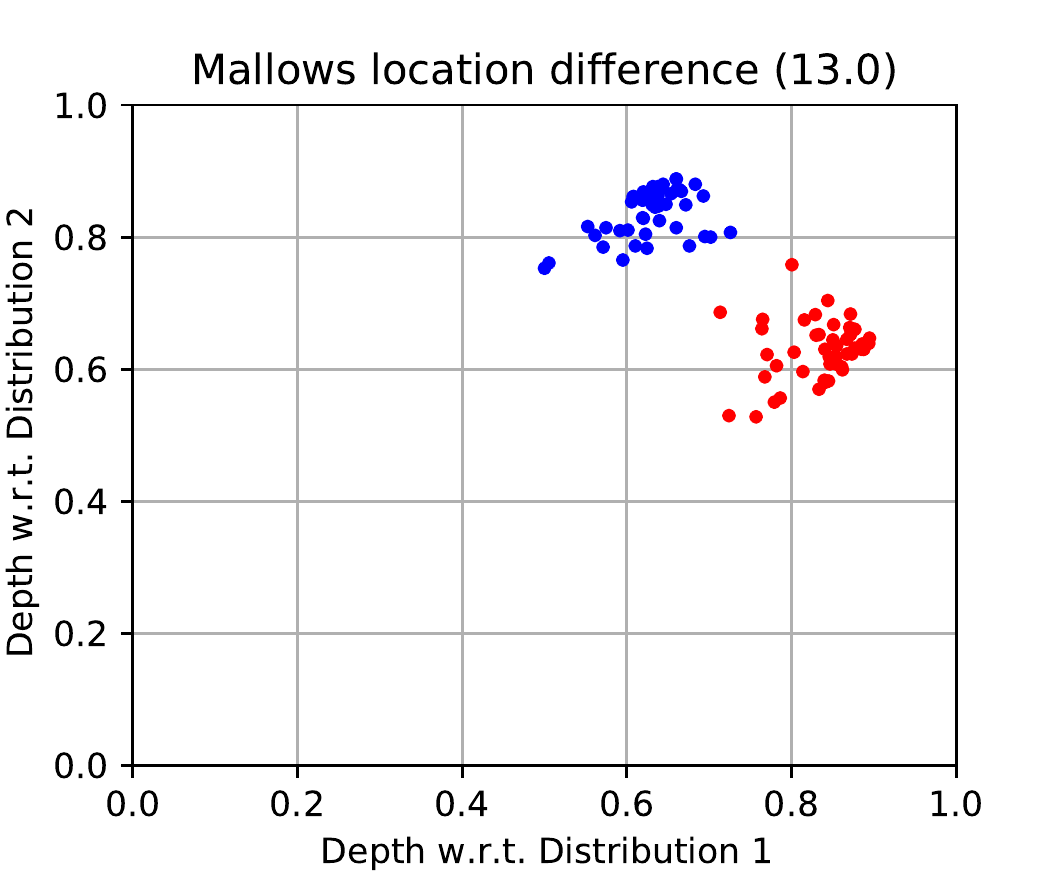} & \includegraphics[width=0.17\textwidth,trim = 0 0 8mm 9.25mm,clip=true]{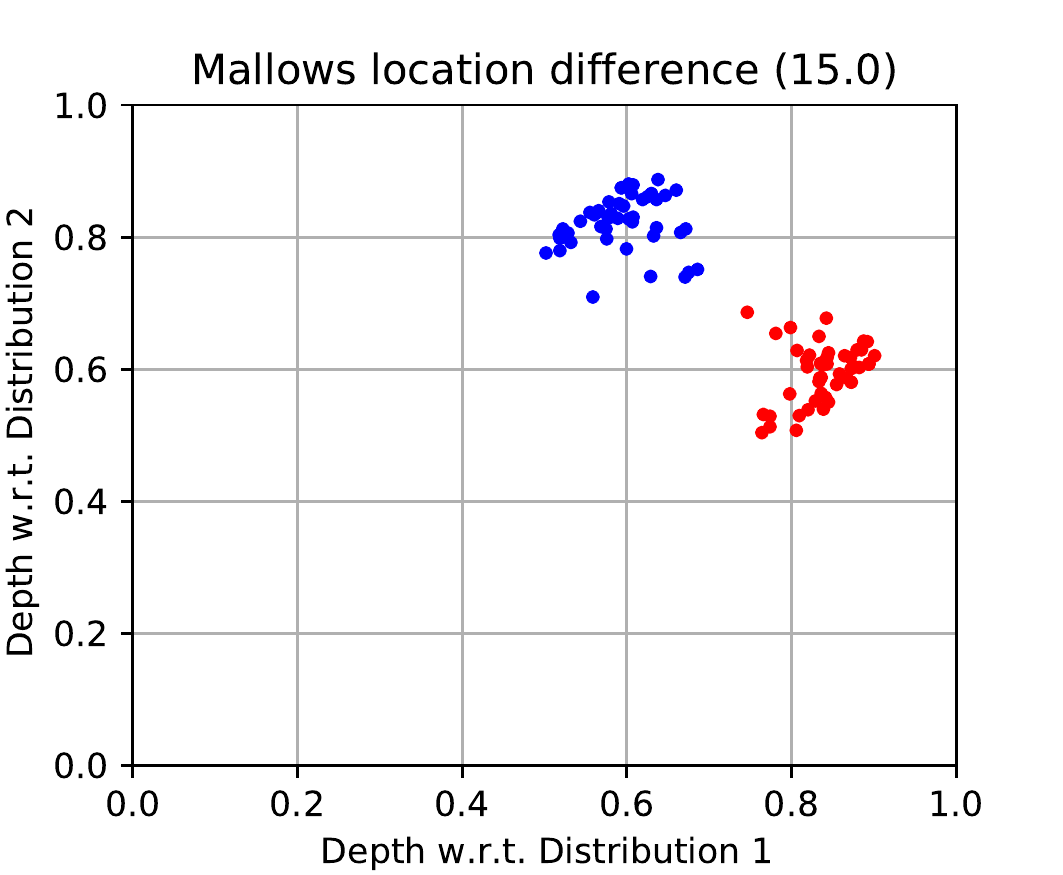} & \includegraphics[width=0.17\textwidth,trim = 0 0 8mm 9.25mm,clip=true]{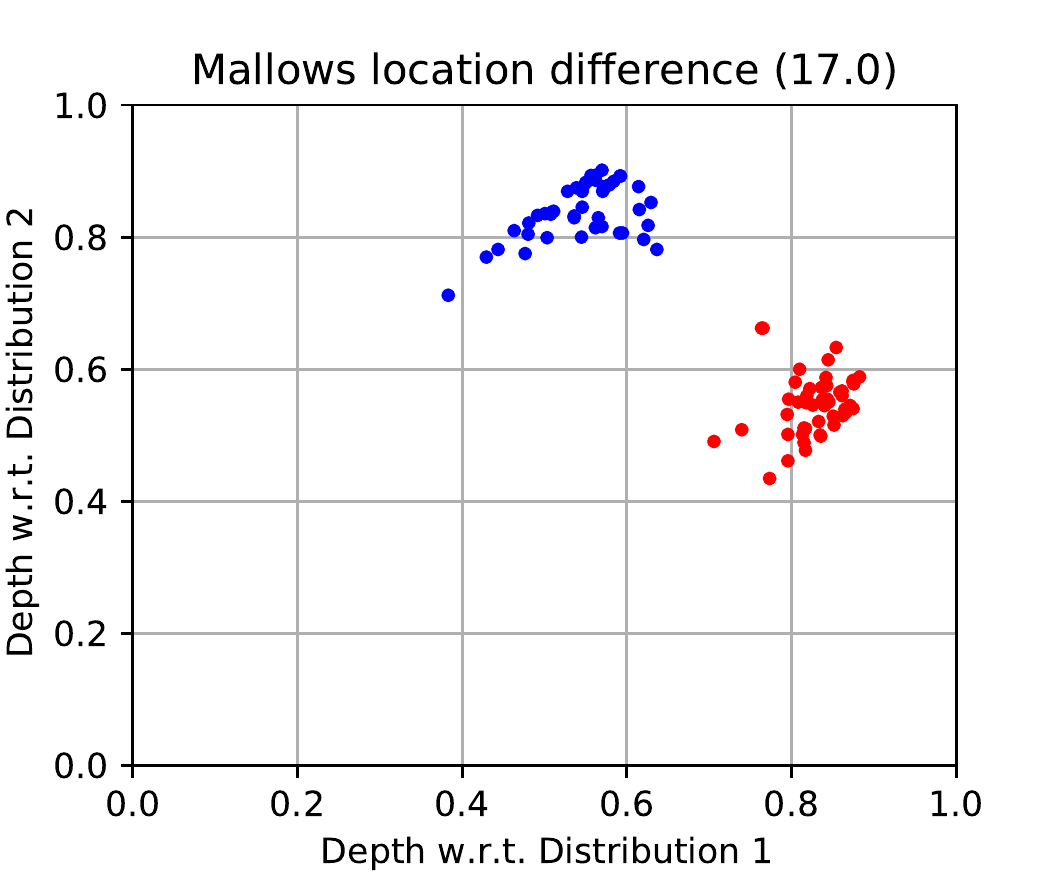} \\
		\end{tabular}
	\end{center}
	\caption{$DD$-plots for pairs of distributions stemming from different instances of the \emph{location-shift model}. The two distributions contain $50$ observations each, drawn from two Mallows models using Kendall $\tau$ distance with parameter $\phi_1=\phi_2=\mathbf{e}^{-1}$. Difference between locations is indicated in each individual plot.}\label{fig:ddplotsMallows1}
\end{figure}

\begin{figure}[!h]
	\begin{center}
		\includegraphics[width=0.65\textwidth,trim = 0 0 0 10mm,clip=true]{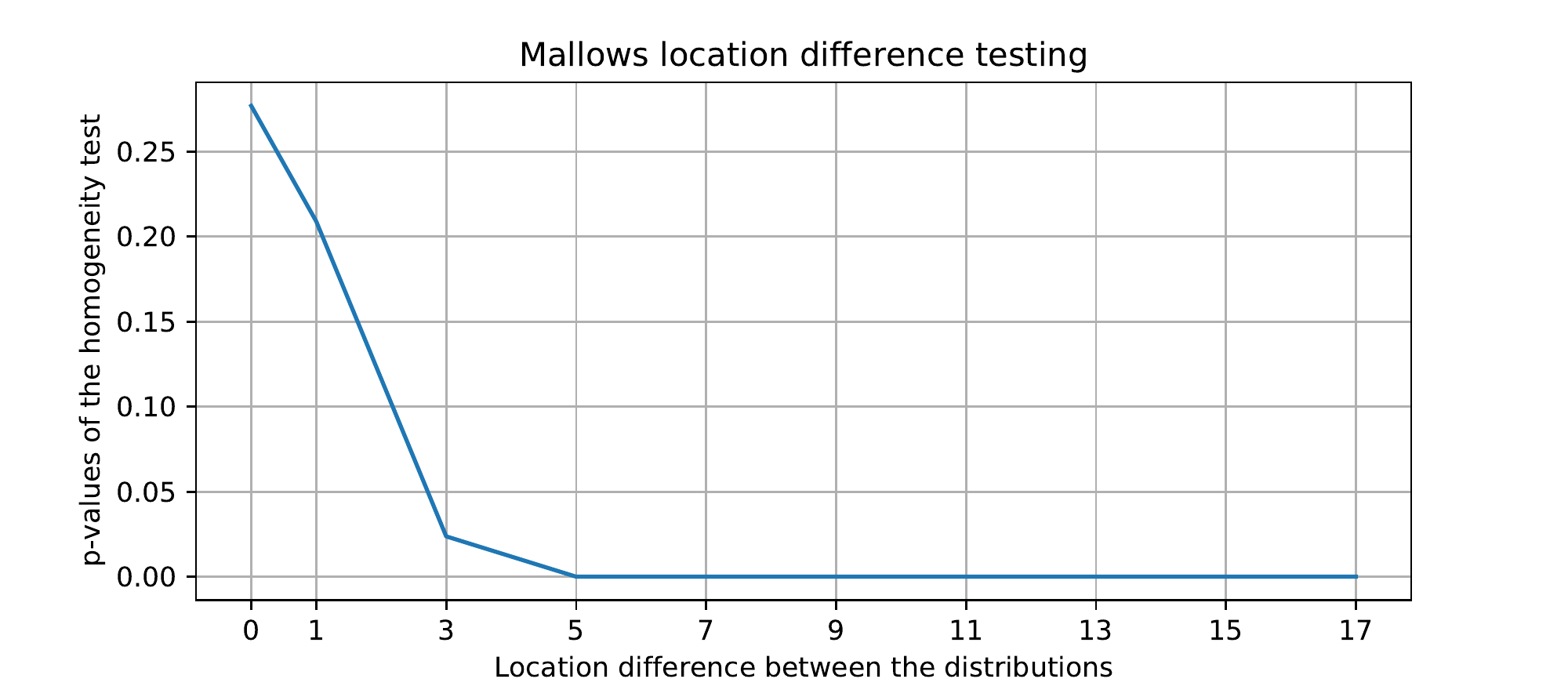}	
	\end{center}
	\caption{$p$-values (averaged over $100$ random repetitions) for the test of homogeneity for a pair of Mallows distributions stemming from the \emph{location-shift model} with location difference $d_\tau(\sigma^*_1,\sigma^*_2)\in\{0,...,17\}$.}\label{fig:pvalsMallows1}
\end{figure}

First, we consider a \emph{location-shift model}: a sequence of pairs of distributions with parameter pairs $(\sigma^*_1, \phi_1)$ and $(\sigma^*_2, \phi_2)$. Setting $\phi_1=\phi_2=\mathbf{e}^{-1}$, we vary $\sigma^*_2$ so that $d_\tau(\sigma^*_1,\sigma^*_2)\in\{0,1,3,5,7,9,11,15,17\}$. Figure~\ref{fig:ddplotsMallows1}, which contains visualization for each pair of parameters for $50$ observations from each distribution, illustrates gradual capturing by the suggested visualization of the increasing location shift between the two laws.

Bearing in mind the same idea, using the same distributional settings, we now provide a formal statistical inference by \emph{homogeneity testing}. More precisely, for each pair of distributions, taking one of them for a reference, we perform the testing procedure $100$ times and indicate average $p$-values in Figure~\ref{fig:pvalsMallows1} (where we stick to the same test setting as in the main text, drawing $500$ observations for the reference distribution and using the Wilcoxon rank-sum statistic). As expected, when there is no parameter difference, the null hypothesis of the distribution's equality cannot be surely rejected. When the difference in the location increases, it is captured very quickly by the testing procedure rejecting on the level $\le 0.05$ when $d_\tau(\sigma^*_1,\sigma^*_2)=3$ only, and with even higher reliability for larger differences.

\begin{figure}[!h]
	\begin{center}
		\begin{tabular}{ccccc}
			{\tiny $\log\phi_1-\log\phi_2=0$} & {\tiny $\log\phi_1-\log\phi_2=0.1$} & {\tiny $\log\phi_1-\log\phi_2=0.2$} & {\tiny $\log\phi_1-\log\phi_2=0.3$} & {\tiny $\log\phi_1-\log\phi_2=0.4$} \\
			\includegraphics[width=0.17\textwidth,trim = 0 0 8mm 9.25mm,clip=true]{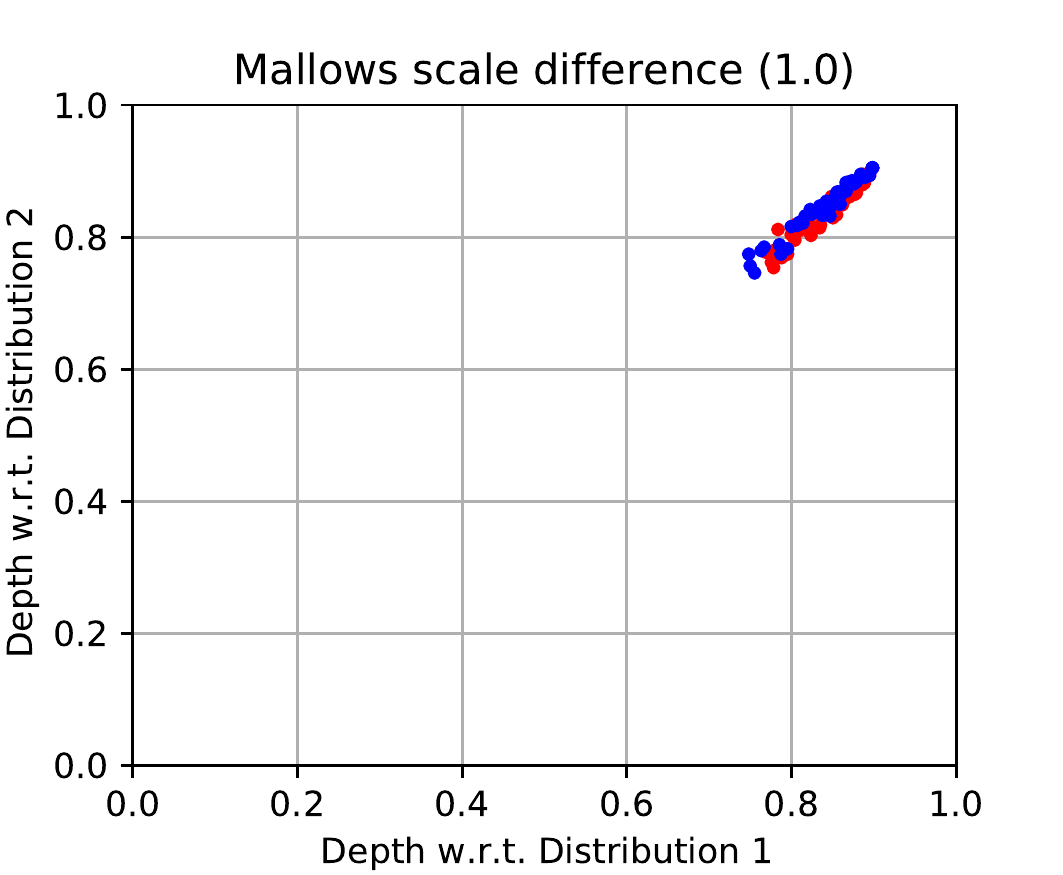} & \includegraphics[width=0.17\textwidth,trim = 0 0 8mm 9.25mm,clip=true]{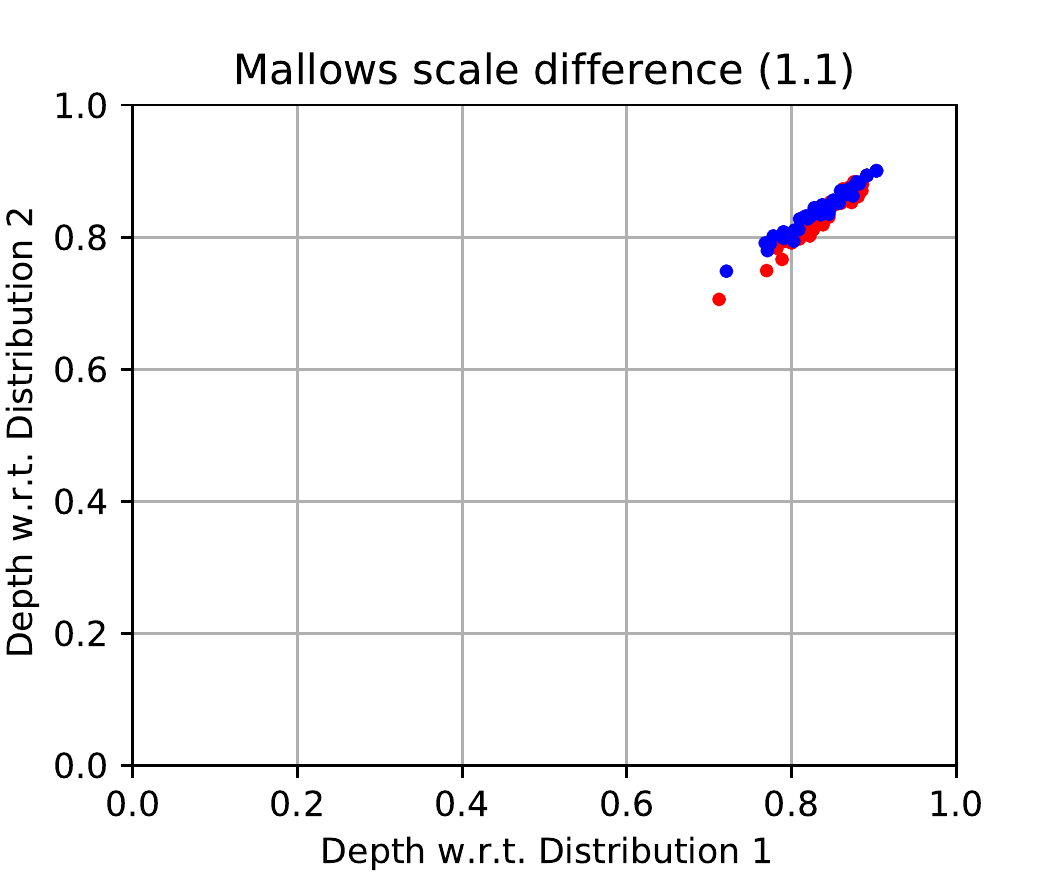} & \includegraphics[width=0.17\textwidth,trim = 0 0 8mm 9.25mm,clip=true]{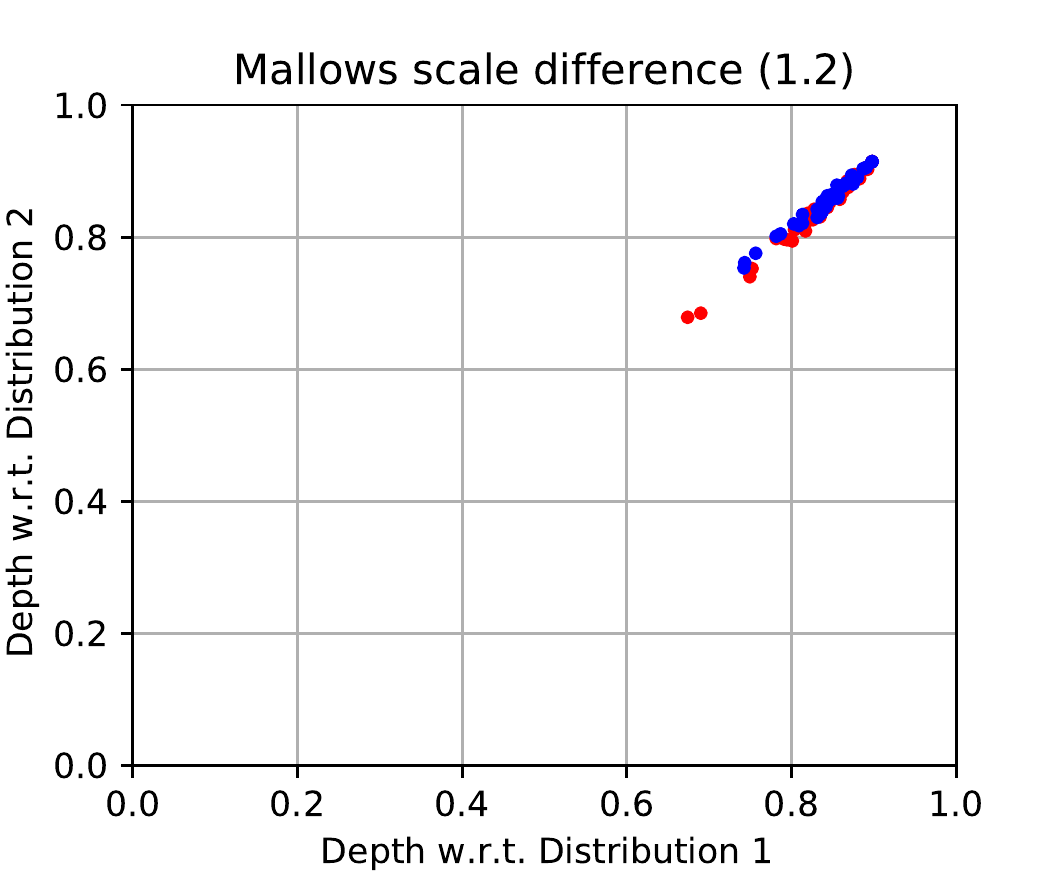} & \includegraphics[width=0.17\textwidth,trim = 0 0 8mm 9.25mm,clip=true]{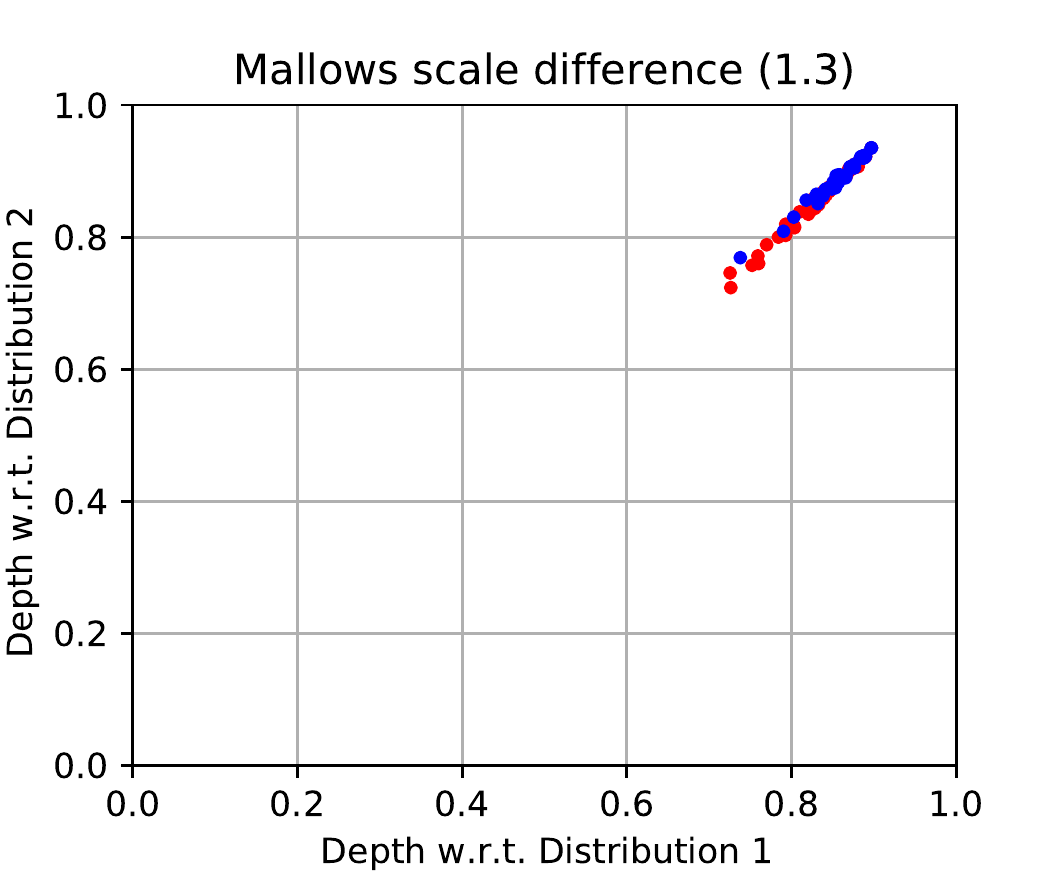} & \includegraphics[width=0.17\textwidth,trim = 0 0 8mm 9.25mm,clip=true]{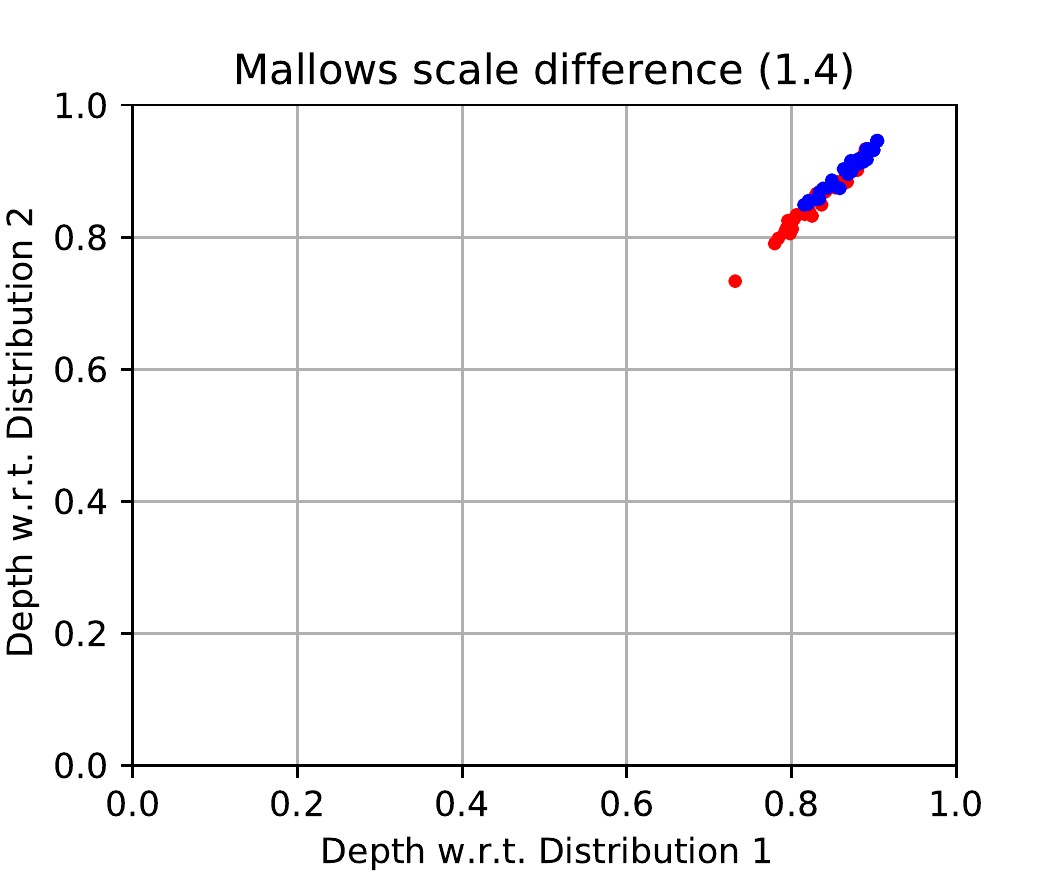} \\
			{\tiny $\log\phi_1-\log\phi_2=0.5$} & {\tiny $\log\phi_1-\log\phi_2=0.6$} & {\tiny $\log\phi_1-\log\phi_2=0.7$} & {\tiny $\log\phi_1-\log\phi_2=0.8$} & {\tiny $\log\phi_1-\log\phi_2=0.9$} \\
			\includegraphics[width=0.17\textwidth,trim = 0 0 8mm 9.25mm,clip=true]{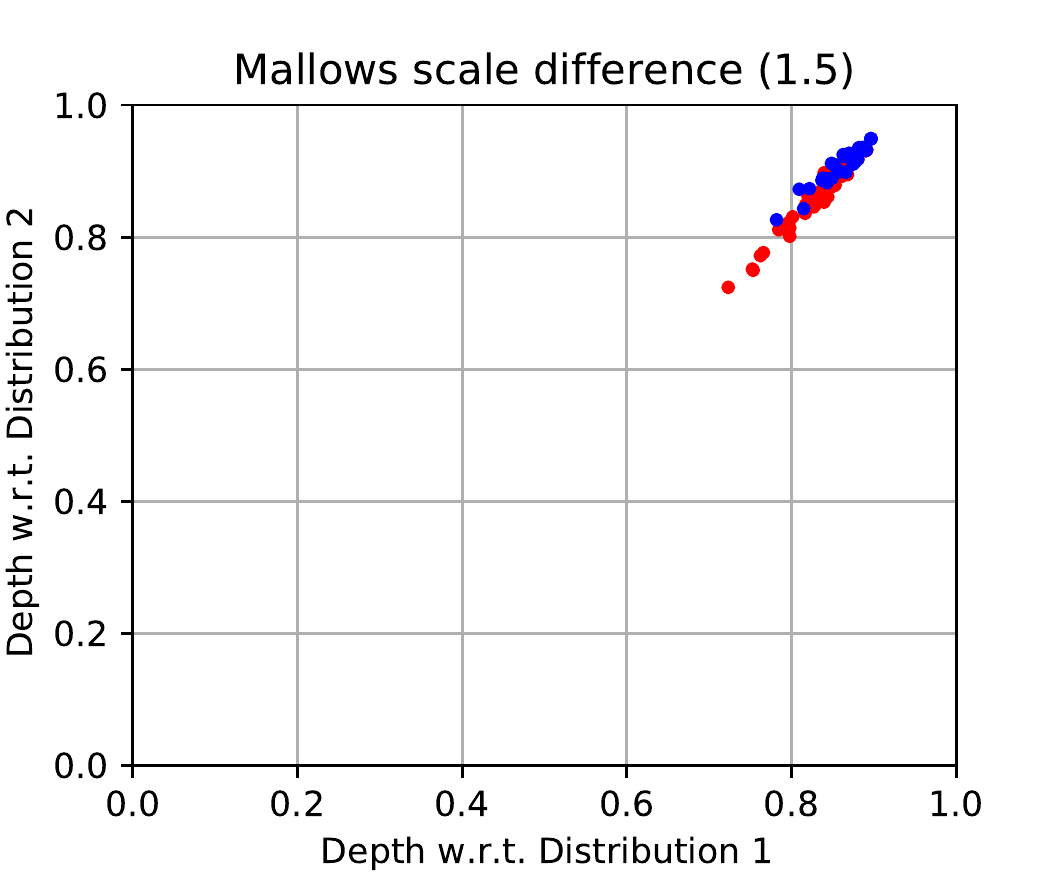} & \includegraphics[width=0.17\textwidth,trim = 0 0 8mm 9.25mm,clip=true]{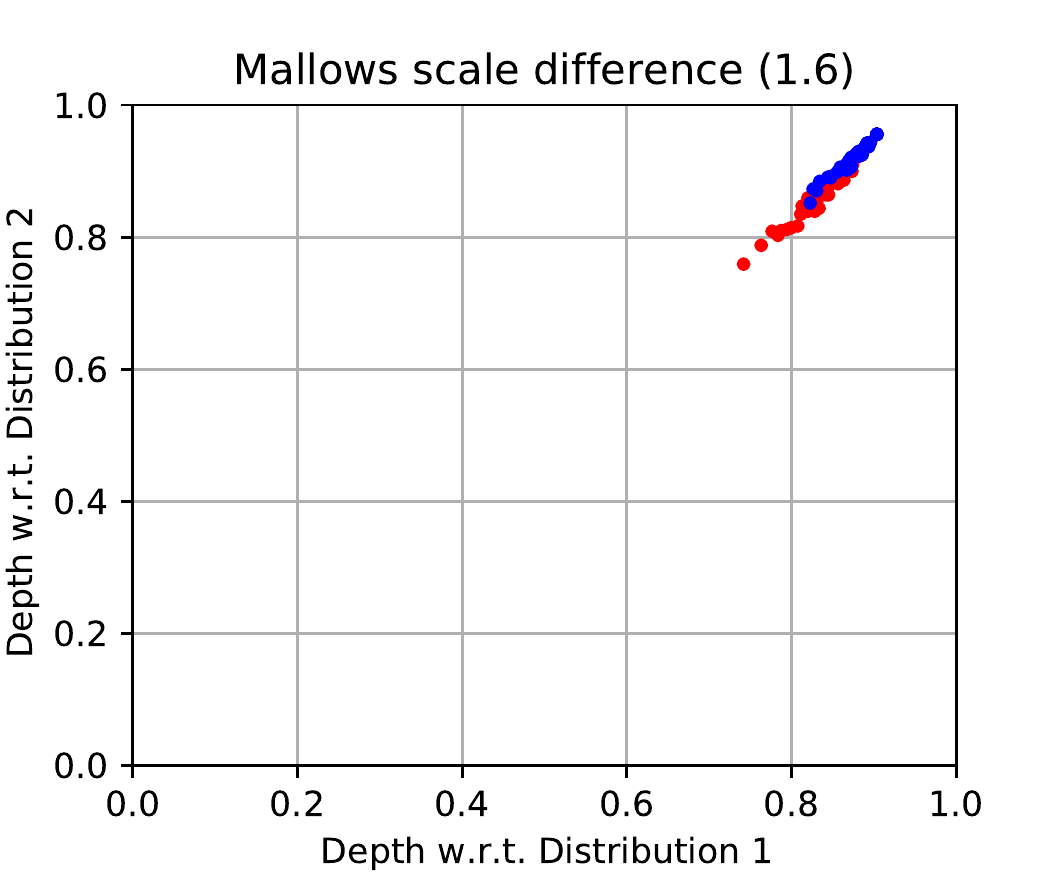} & \includegraphics[width=0.17\textwidth,trim = 0 0 8mm 9.25mm,clip=true]{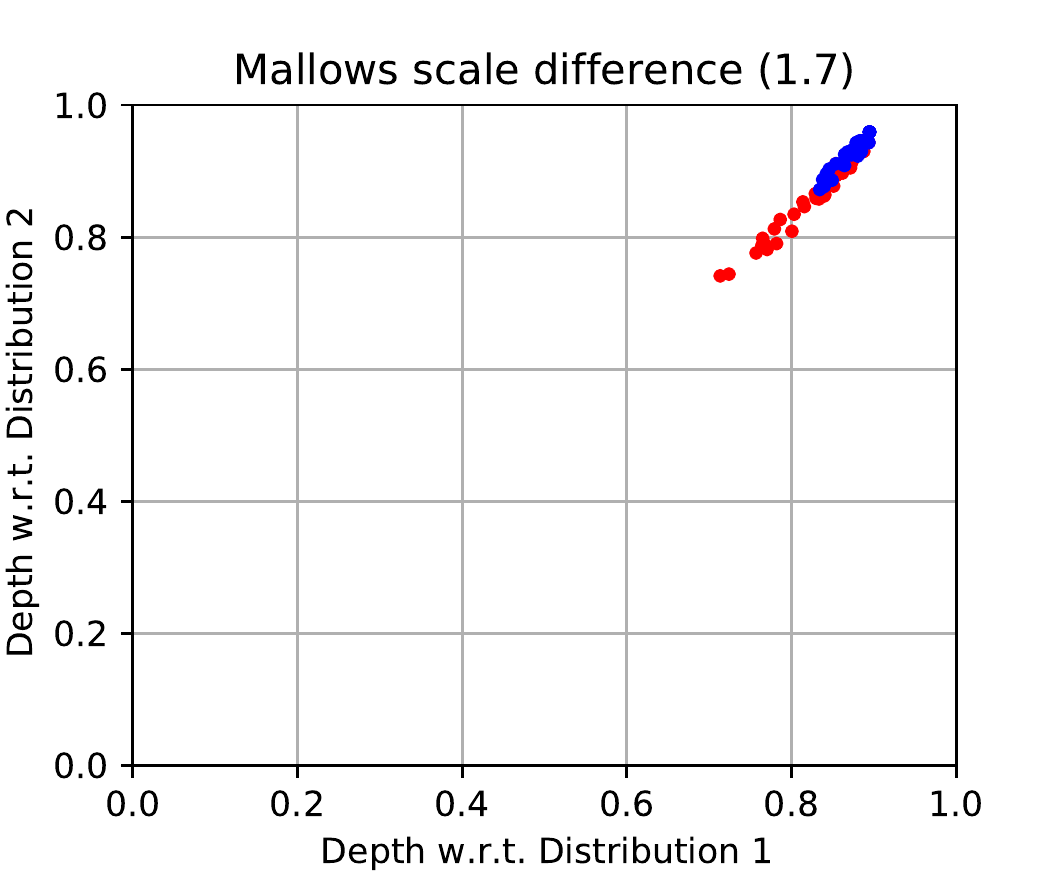} & \includegraphics[width=0.17\textwidth,trim = 0 0 8mm 9.25mm,clip=true]{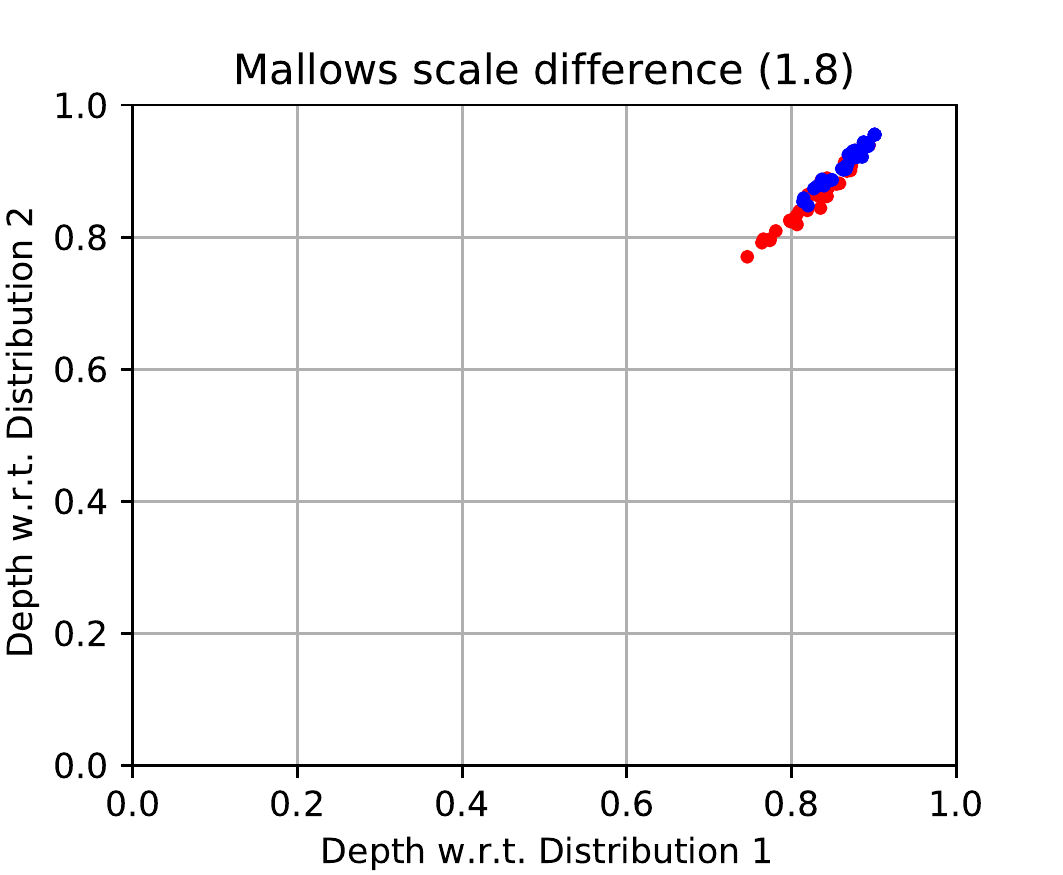} & \includegraphics[width=0.17\textwidth,trim = 0 0 8mm 9.25mm,clip=true]{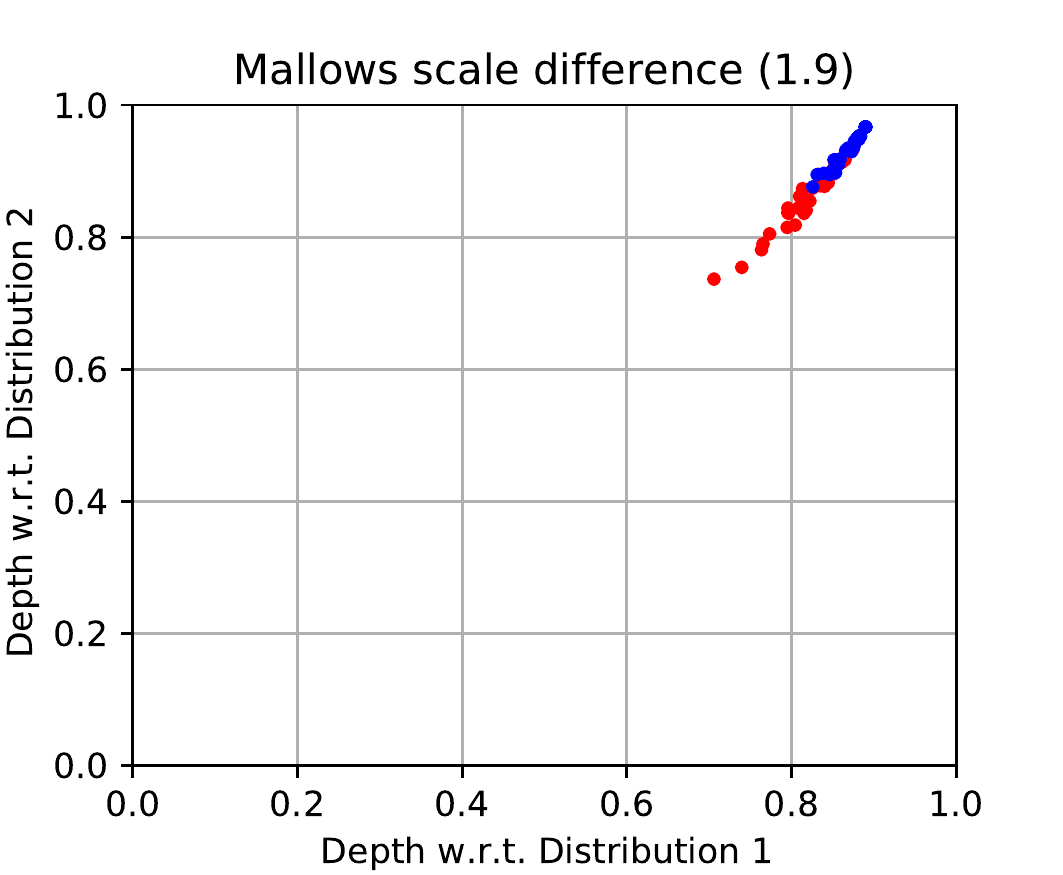} \\
		\end{tabular}
	\end{center}
	\caption{$DD$-plots for pairs of distributions stemming from different instances of the \emph{scale-difference model}. The two distributions contain $50$ observations each, drawn from two Mallows models (using Kendall $\tau$ distance) with the same center and with parameters $\phi_1=\mathbf{e}^{-1}$ and $\phi_2=\mathbf{e}^{\psi}$ where $\psi\in\{-1,-1.1,-1.2,-1.3,-1.4,-1.5,-1.6,-1.7,-1.8,-1.9\}$. Difference between scales' logarithms is indicated in each individual plot.}\label{fig:ddplotsMallows2}
\end{figure}

\begin{figure}[!h]
	\begin{center}
		\includegraphics[width=0.65\textwidth,trim = 0 0 0 10mm,clip=true]{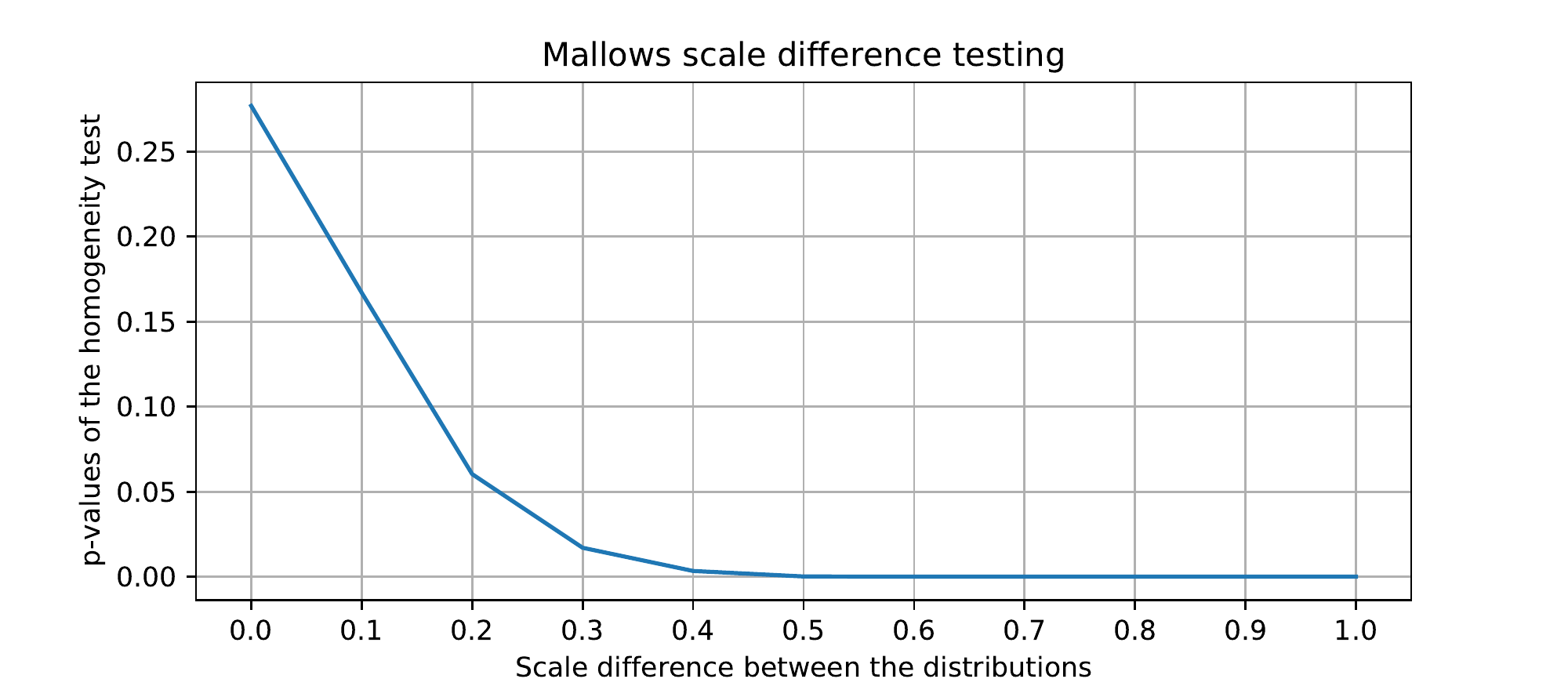}	
	\end{center}
	\caption{$p$-values (averaged over $100$ random repetitions) for the test of homogeneity for a pair of Mallows distributions stemming from the \emph{scale-difference model}, \textit{i.e.} with the same center and difference in scales $\log\phi_1-\log\phi_2\in\{0,...,1\}$.}\label{fig:pvalsMallows2}
\end{figure}

While being already challenging when having no (parametric) assumptions on the distribution, traditional rank-based homogeneity testing procedures usually assume location difference. Thus, we consider an even more disadvantageous setting, the \emph{scale-difference model}: the location of the both distributions is the same, and those differ in dispersion only. As above, the $DD$-plots for $10$ scale difference values (on the equidistant grid with step $0.1$ on the logarithmic scale) are presented in Figure~\ref{fig:ddplotsMallows2}. With increasing scale difference, visual patterns dis-associate (less than in the previous setting though) which intrigues the formal inference.

Finally, we repeat the previously used testing procedure and indicate the average $p$-values in Figure~\ref{fig:pvalsMallows2}. One observes that for difference in scale (measured on the logarithmic scale) equal to $0.3$ or higher, the homogeneity testing procedure distinguishes the distributions with level $\le 0.05$ or less.

\subsection{Application to real data}
\label{sec:real_data}
Let us now explore the applicability of our depth function to different tasks on real data. 

\subsubsection{Student dataset}
\label{sec:student_dataset}

Next, we consider a real data set which consists of students' rankings before ($N_1=169$ students) and after ($N_2=179$ students) the class, with known ground truth (correct answer = $(0,1,2,3)$) where $n=4$ (refer to \url{https://github.com/ekhiru/students-dataset} for details about the dataset). Simple computation indicates that $d_{\tau}$ from the average ranking to the true one is $2$ before the class and $1$ after, thus suggesting that the students improved after studying. We employ the same homogeneity testing methodology as above to derive formal statistical inference. By taking $100$ randomly chosen observations from the 'before class' cohort as the reference, we use $69$ observations from the 'before class' cohort and $79$ (also randomly chosen) from the 'after class' group. The diagnostic $DD$-plot of the two cohorts together with $p$-values over $1000$ random repetitions and the asymptotic density under $H_0$ are indicated in Figure~\ref{fig:testrealdata2}, and illustrate improvement of the students' knowledge after the class.

\begin{figure}[h!]
	\begin{center}
	\begin{tabular}{cc}
		\includegraphics[width=0.33\textwidth,trim=2.5mm 0 10mm 8mm,clip=true]{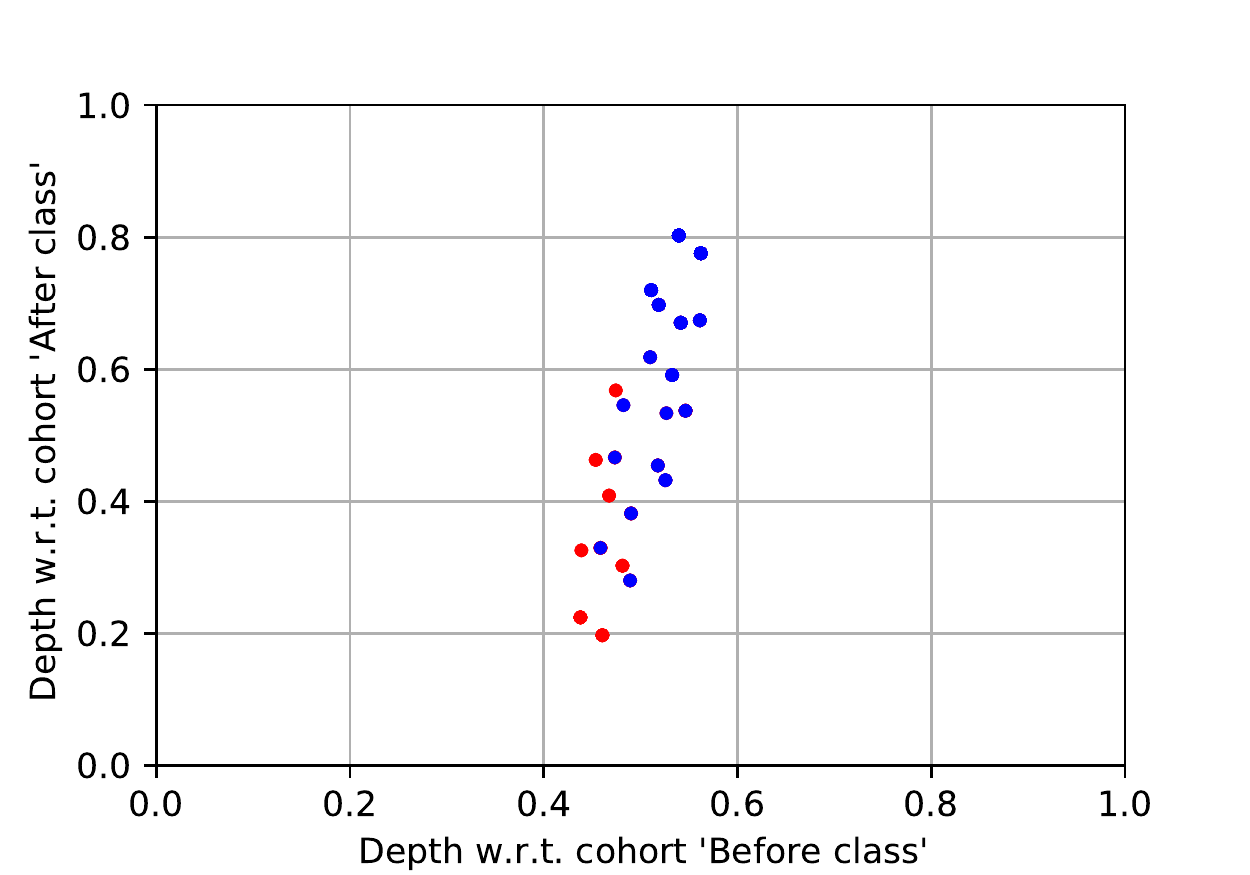} & \includegraphics[width=0.635\textwidth,trim=15mm 0 18mm 8mm,clip=true]{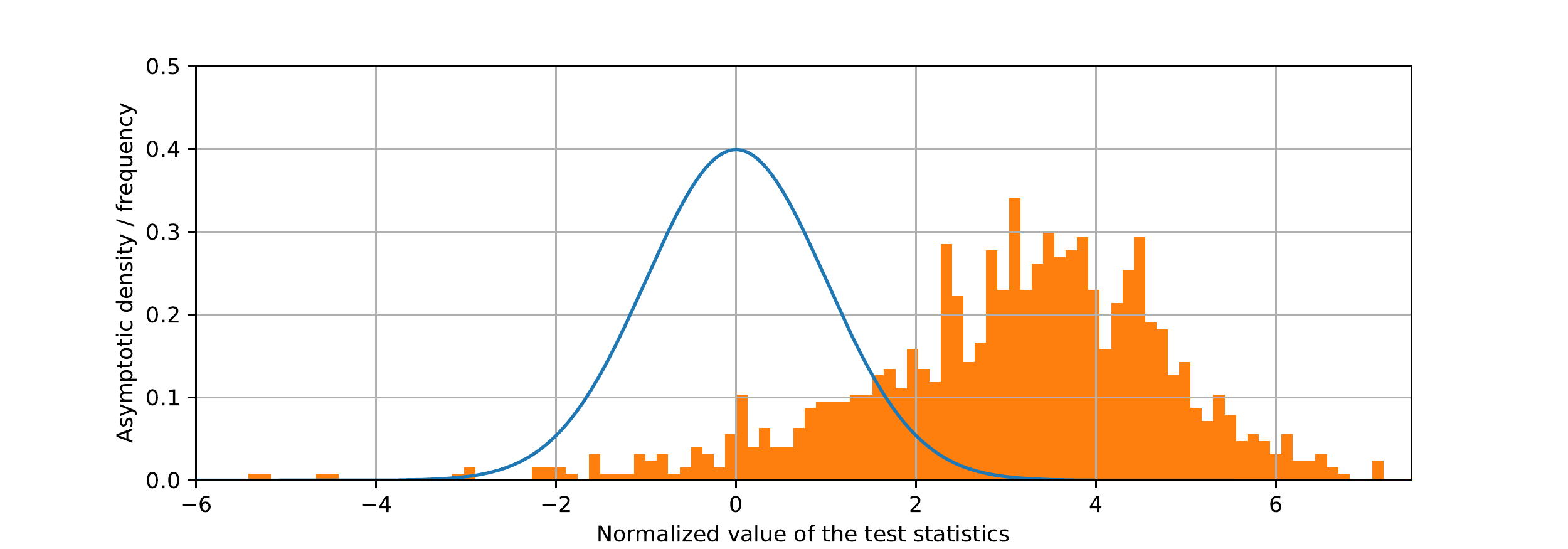}
	\end{tabular}
	\end{center}
	\caption{Left: $DD$-plot for the two cohorts of students, 'before class' (red) and 'after class' (blue), respectively. Right: $p$-values of the homogeneity test over $1000$ random repetitions together with asymptotic density under $H_0$.}
	\label{fig:testrealdata2}
\end{figure}

\subsubsection{Sushi dataset}
\label{suppl:sushi}

\begin{figure}[!h]
	\begin{center}
		\includegraphics[width=0.475\textwidth]{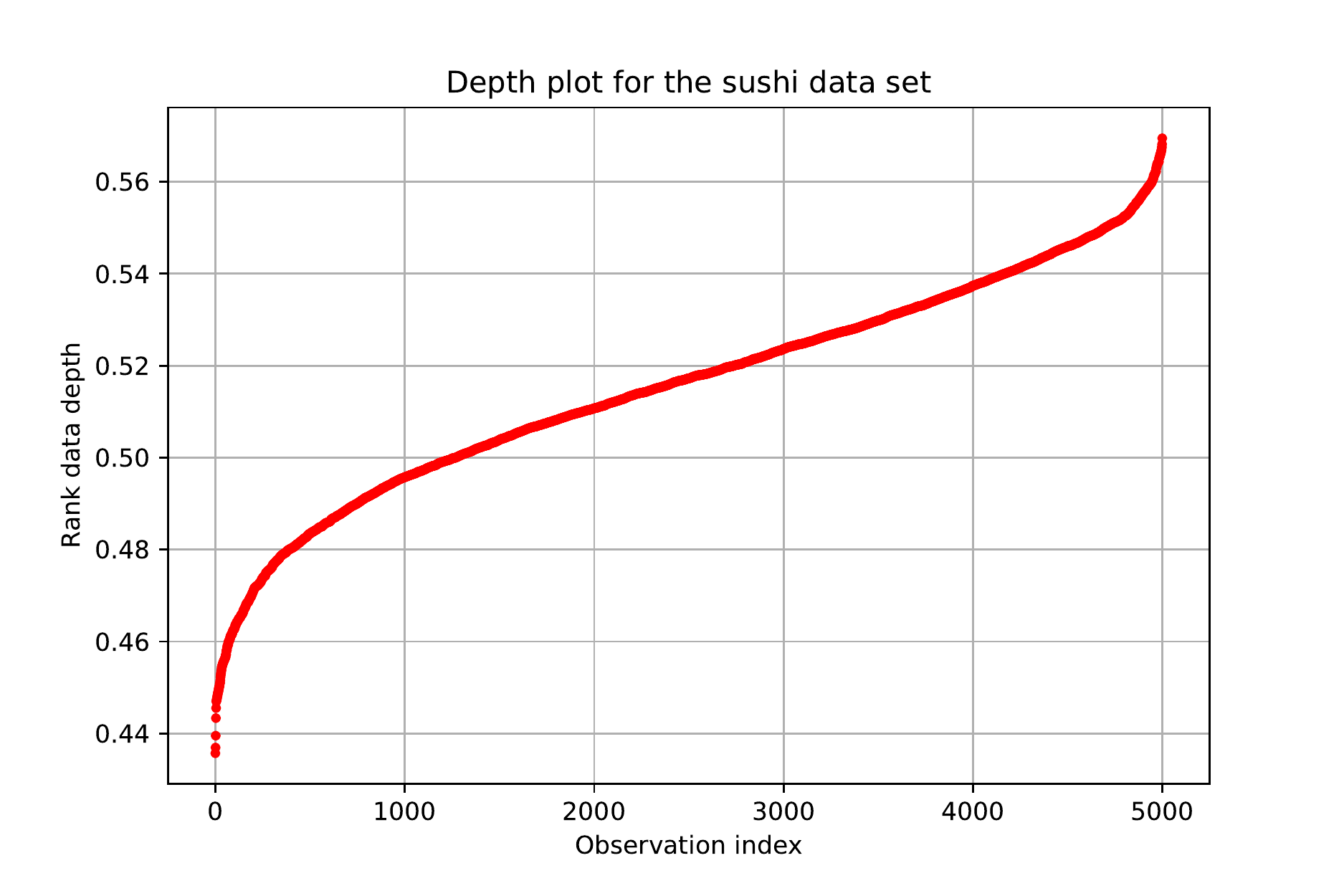} \quad \includegraphics[width=0.475\textwidth]{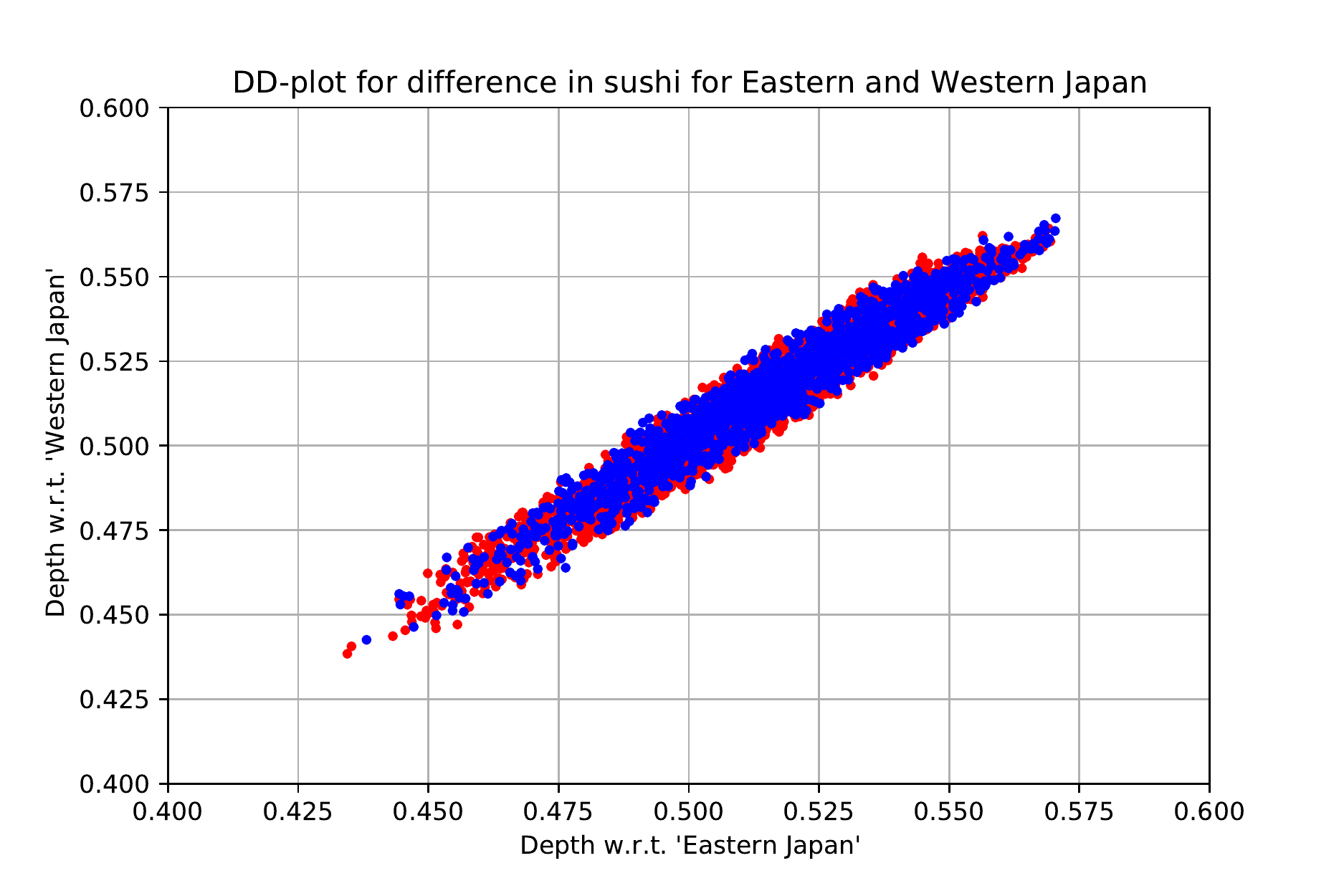}
	\end{center}
	\caption{Exploratory statistics of the Sushi data set by~\cite{Kamishima03} (\url{https://www.kamishima.net/sushi/}). Left: The depth of each observation in the set, in increasing depth order. Right: The comparative $DD$-plot for Eastern (red) and Western (blue) Japan.}\label{fig:sushi}
\end{figure}

As a last application, we analyze the \emph{Sushi} data set, which contains $5\,000$ rankings of $10$ sushi items. We refer the reader to~\cite{Kamishima03} (\url{https://www.kamishima.net/sushi/}) for the detailed description of the data set. By means of the introduced depth notion, we explore this data set from two angles. First, we provide depth based ranking of the entire data set, which can be seen as the ranking equivalent of the cumulative distribution function. The depth of each of the $5\,000$ observations (ordered increasingly) is indicated in Figure~\ref{fig:sushi} (left). Second, in view of the ten considered items, we check the know difference between food preferences in Eastern and Western Japan. The $DD$-plot of these two groups (containing $3\,448$ and $1\,552$ observations each, respectively) is presented in Figure~\ref{fig:sushi} (right). Since the two clouds of points substantially intersect, this rather drives to conclusion that the mentioned above difference is not connected with the choice of the sushi items used in the data set.

\subsubsection{Mechanical Turk Dots dataset}
\label{sec:mecha_turk_dots}

The Mechanical Turk Dots dataset contains 800 full rankings of 4 items. Each item corresponds to random dots presented to a user on Mechanical Turk, who is asked to rank them from those containing the least dots (first) to those containing the most dots (last). Thus, there is a ground truth ranking for this dataset. 40 sets of puzzles were placed on Mechanical Turk and were ranked by 20 users, leading to 800 rankings.

This dataset is SST and the deepest ranking corresponds to the ground truth. We thus contaminate the dataset by swapping a random proportion of 1/4 of the rankings, i.e. by taking the opposite ranking. Figure \ref{fig:mecha_turk_plot2} (a) shows that there is no obvious difference between the swapped and clean rankings, but in figure \ref{fig:mecha_turk_plot2} (b), we see we recovered the ground truth ranking after the trimming strategy.

\begin{figure}[!h]
    \begin{center}
    \begin{tabular}{cc}
        (a) & (b) \\
        \includegraphics[scale=0.45,clip=true,page=1]{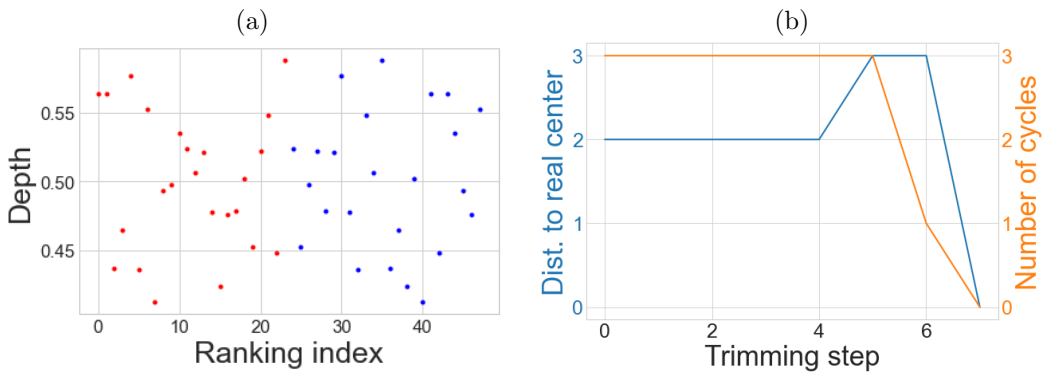} & \includegraphics[scale=0.3,clip=true,page=1]{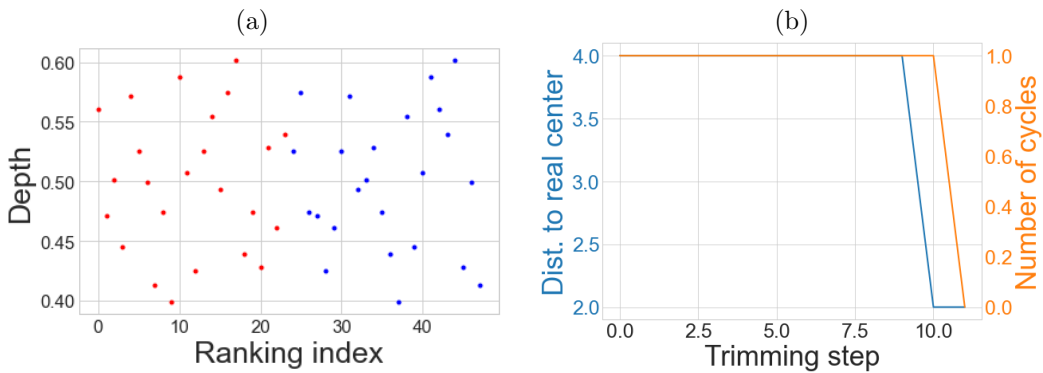}
    \end{tabular}
    \caption{Depth plots before trimming with swapped (red) and clean (blue) points; evolution of candidate median (deepest ranking) distance to real median and number of cycles through trimming (c)}
    \label{fig:mecha_turk_plot2}
    \end{center}
\end{figure}

\subsubsection{Netflix Prize dataset}
\label{sec:netflix_dataset}

We selected one of the Netflix Prize dataset contains 1814 full rankings of 4 movies (Dirty Dancing, Maid in Manhattan, Shrek and Father of the Bride). This dataset is SST and the deepest ranking corresponds to Shrek $\succ$ Father of the Bride $\succ$ Maid in Manhattan $\succ$ Dirty Dancing, considered as the real center of the dataset. We contaminated the dataset by swapping a random proportion of 11\% of the rankings, i.e. by taking the opposite ranking. Figure \ref{fig:netflix_plot} (a) shows that there is no obvious difference between the swapped and clean rankings, but in figure \ref{fig:netflix_plot} (b), we the median computed after trimming is closer to the real center than before.

\begin{figure}[!h]
    \begin{center}
    \begin{tabular}{cc}
        (a) & (b) \\
        \includegraphics[scale=0.45,clip=true,page=1]{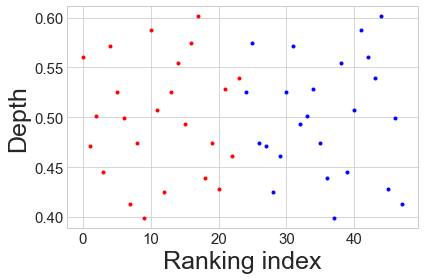} & \includegraphics[scale=0.3,clip=true,page=1]{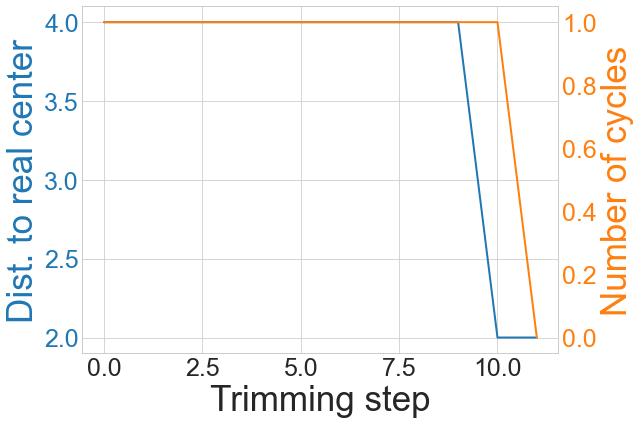}
    \end{tabular}
    \caption{Depth plots before trimming with swapped (red) and clean (blue) points; evolution of candidate median (deepest ranking) distance to real median and number of cycles through trimming (c)}
    \label{fig:netflix_plot}
    \end{center}
\end{figure}

\vfill
\end{document}